\documentclass[10pt]{article}
\usepackage[top=1in, bottom=1in, left=0.9in, right=0.9in]{geometry}

\usepackage{zhiyuz}

\title{\textbf{PDE-Based Optimal Strategy for Unconstrained Online Learning}}
\author{
  Zhiyu Zhang \\
  Boston University\\
  \texttt{zhiyuz@bu.edu}\\
  \and
  Ashok Cutkosky \\
  Boston University\\
  \texttt{ashok@cutkosky.com}\\
  \and
  Ioannis Ch. Paschalidis\\
  Boston University\\
  \texttt{yannisp@bu.edu}\\
}
\date{\vspace{-5ex}}

\begin{document}
\maketitle

\begin{abstract}
Unconstrained Online Linear Optimization (OLO) is a practical problem setting to study the training of machine learning models. Existing works proposed a number of potential-based algorithms, but in general the design of these potential functions relies heavily on guessing. To streamline this workflow, we present a framework that generates new potential functions by solving a Partial Differential Equation (PDE). Specifically, when losses are 1-Lipschitz, our framework produces a novel algorithm with anytime regret bound $C\sqrt{T}+\norms{u}\sqrt{2T}[\sqrt{\log(1+\norms{u}/C)}+2]$, where $C$ is a user-specified constant and $u$ is any comparator unknown and unbounded a priori. Such a bound attains an optimal loss-regret trade-off without the impractical doubling trick. Moreover, a matching lower bound shows that the leading order term, \emph{including} the constant multiplier $\sqrt{2}$, is tight. To our knowledge, the proposed algorithm is the first to achieve such optimalities.
\end{abstract}

\section{Introduction}\label{section:introduction}

Advances in online learning have brought deeper understanding and better algorithms to the training of machine learning models. Among all the problem settings therein, unconstrained online learning has received special attention since the parameter of the model is often unrestricted before seeing any data. Compared to conventional settings with  a bounded domain, the unconstrained setting poses an additional challenge: starting from a bad initialization, how can an algorithm quickly find the optimal parameter that may be \emph{far-away}? With the growing popularity of high-dimensional models, such an issue becomes increasingly important. 

In this paper, we address this issue by studying a theoretical problem called \emph{unconstrained Online Linear Optimization} (OLO). Given an unbounded domain $\R^d$, we need to design an algorithm such that in each round it makes a deterministic prediction $x_t\in\R^d$, observes a loss gradient $g_t\in\R^d$ and suffers a loss $\inner{g_t}{x_t}$, where $g_t$ is adversarial (can arbitrarily depend on $x_1,\ldots,x_t$) and satisfies $\norms{g_t}\leq 1$. The considered performance metric is the regret
\begin{equation*}
\reg_T(u)=\sum_{t=1}^T\inner{g_t}{x_t}-\sum_{t=1}^T\inner{g_t}{u},
\end{equation*}
and the goal is to achieve low regret for all comparator $u\in\R^d$, time horizon $T\in\N_+$ and loss gradients $g_1,\ldots,g_T$. Besides pursuing the optimal rate on $T$, we are also interested in the dependence of $\reg_T(u)$ on $\norms{u}$, as it captures how well the algorithm performs if the optimal fixed prediction (in hindsight) turns out to be far-away from the user's prior guess (in this case, the origin). 

Many algorithms for unconstrained OLO are based on the potential method. Given a potential function $V_t(\cdot)$, the key idea is to accumulate the history into a ``sufficient statistic'' $S_t=-\sum_{i=1}^{t-1}g_i$ and predict the gradient of $V_t(\cdot)$ at $S_t$, i.e., $x_t=\nabla V_t(S_t)$. Through this procedure, designing new algorithms is converted into a more tangible task of finding good potentials. Specifically, with an arbitrary constant $C$, existing works (e.g., \cite{mcmahan2014unconstrained,orabona2016coin,mhammedi2020lipschitz}) adopted the one-dimensional potential
\begin{equation}\label{eq:value_existing}
V_t(S_t)=\frac{C}{\sqrt{t}}\exp\rpar{\frac{S^2_t}{2t}}
\end{equation}
and its variants to achieve the regret bound
\begin{equation}\label{eq:regret_existing}
\reg_T(u)\leq C+\norms{u}O\rpar{\sqrt{T\log\frac{\norms{u}\sqrt{T}}{C}}}.
\end{equation}
Among all the achievable upper bounds with $\reg_T(0)\leq C$, the order of $\norms{u}$ and $T$ in (\ref{eq:regret_existing}) is optimal up to multiplicative constants. In practice, these algorithms have demonstrated promising performance with minimum hyperparameter tuning \cite{orabona2017training,chen2022better}.

Despite these strong results, there is still room for improvement though. Intuitively, requiring a constant $\reg_T(0)$ \emph{all the time} amounts to a strong belief that the initialization of the model is close to the optimal parameter, which somewhat contradicts the use of an unconstrained domain in the first place. Reflected in the regret bound, the RHS of (\ref{eq:regret_existing}) can be more generally viewed as a trade-off between the values of $\reg_T(u)$ at small $\norms{u}$ and large $\norms{u}$: if the \emph{cumulative loss} $\reg_T(0)$ is allowed to increase with $T$, then one may obtain lower regret with respect to far-away comparators. This will be favorable in high-dimensional problems, as good initializations become harder to obtain. 

The question now becomes, what is the \emph{optimal loss-regret trade-off}, and how to efficiently achieve it? As a first attempt, one could assume a known time horizon $T$, set $C=\sqrt{T}$ in (\ref{eq:regret_existing}) and obtain \cite{mcmahan2014unconstrained}
\begin{equation}\label{eq:regret_intro}
\reg_T(u)\leq \sqrt{T}+\norms{u}O\rpar{\sqrt{T\log\norms{u}}}.
\end{equation}
With respect to $T$ alone, $R_T(u)=O(\sqrt{T})$. Since it matches the standard minimax lower bound for constrained OLO, we consider this loss-regret trade-off as \emph{optimal}. The real challenge is an anytime bound - existing arguments rely on a doubling trick\footnote{Running the fixed-$T$ algorithm on time intervals of doubling lengths, i.e., $[2^i:2^{i+1}-1]$.} \cite{shalev2011online}, which not only is notoriously impractical, but also leads to an extra multiplying constant with unclear optimality. Perhaps due to this reason, regret bounds like (\ref{eq:regret_intro}) have received a lot less attention than (\ref{eq:regret_existing}), despite its theoretical advantages. 

The present work aims at a practical and optimal approach towards an anytime bound in the form of (\ref{eq:regret_intro}) - this requires a significant departure from existing techniques. Specifically, we will go back one step and rethink the design of potential functions in unconstrained OLO. The classical workflow is based on heuristic guessing, which is challenging when the suitable potential is not an elementary function (e.g., involving complicated integrals or series). Our goal is to propose a systematic approach for this task, which reduces the amount of guessing and allows us to handle more complicated potentials. Eventually, as a byproduct, our framework produces a new algorithm that efficiently achieves the optimal loss-regret trade-off.

\subsection{Result and contribution}

As motivated above, our contributions are twofold. 
\begin{itemize}
\item We propose a framework that uses solutions of a specific \emph{Partial Differential Equation} (PDE) as potential functions for unconstrained OLO. To this end, we characterize minimax optimal potentials via a backward recursion, and our PDE naturally arises in its continuous-time limit. Solutions of this PDE approximately solve the discrete-time recursion. Therefore, one may search for suitable potentials within such solutions and their variants, which is a more structured procedure than direct guessing. 

\item Using our framework, we design a one-dimensional potential which is not elementary and hard to guess without the help of a PDE. The induced algorithm guarantees
\begin{equation*}
\reg_T(u)\leq C\sqrt{T}+\norm{u}\sqrt{2T}\spar{\sqrt{\log\rpar{1+\frac{\norm{u}}{\sqrt{2}C}}}+2}.
\end{equation*}
Our bound achieves an optimal loss-regret trade-off (\ref{eq:regret_intro}) without the doubling trick. Moreover, by constructing a matching lower bound, we further show that the leading order term, \emph{including} the constant multiplier $\sqrt{2}$, is tight. To our knowledge, the proposed algorithm is the first to achieve such optimalities. The obtained theoretical benefits are validated by experiments. 
\end{itemize}

\subsection{Related work}

\paragraph{Unconstrained OLO} Unconstrained convex optimization has been extensively studied in both the offline and online settings. Typically, strong guarantees can be obtained assuming certain curvature on the loss function. Without curvature, the problem becomes harder but more practical for large scale applications (e.g., training machine learning models), as gradients become the only available feedback.

For unconstrained OLO, if the optimal learning rate in hindsight is known a priori, \emph{Online Gradient Descent} (OGD) \cite{zinkevich2003online} guarantees $O(\norms{u}\sqrt{T})$ regret with respect to the optimal comparator $u$. Without that prior knowledge, the regret bound downgrades to $O(\norms{u}^2\sqrt{T})$. A line of works on \emph{parameter-free} algorithms aim at achieving $\tilde O(\norms{u}\sqrt{T})$ regret in the latter setting. Specifically, McMahan and Steeter \cite{mcmahan2012no} proposed the first parameter-free algorithm with $O(\norms{u}\sqrt{T}\log (\norms{u}T))$ regret, which was later improved to $O(\norms{u}\sqrt{T\log (\norms{u}T)})$ by a potential-based algorithm \cite{mcmahan2014unconstrained}; this is the optimal rate \cite{mcmahan2012no,orabona2013dimension,orabona2019modern} given the constraint $\reg_T(0)\leq \textrm{constant}$. More recently, the analysis was streamlined in \cite{orabona2016coin,cutkosky2018black} through a \emph{coin-betting game}, and in \cite{foster2018online} through the \emph{Burkholder method}. The obtained algorithms find applications in differential privacy \cite{jun2019parameter,van2019user},
combining optimizers \cite{cutkosky2019combining,cutkosky2020parameter,zhang2022adversarial} and training neural networks \cite{orabona2017training}.

Among all these results, a shared limitation is the focus on $\reg_T(0)\leq \textrm{constant}$. Other forms of loss-regret trade-offs are less explored, both theoretically and practically. Moreover, the optimality of leading constants has not been considered. 

\paragraph{Differential equations for online learning} Recently, applying differential equations in online learning has received growing interests. The first idea was proposed by Kapralov and Panigrahy \cite{kapralov2011prediction}, where a potential function for \emph{Learning with Expert Advice} (LEA) \cite{littlestone1994weighted} was designed by solving an \emph{Ordinary Differential Equation} (ODE). As a key benefit, the obtained regret bound achieves a trade-off with respect to different individual experts. The proposed techniques were later applied to the discounted setting \cite{andoni2013differential} and the movement-constrained setting \cite{daniely2019competitive}. Interestingly, our prior work \cite{zhang2022adversarial} used the coin-betting approach to achieve a similar goal as \cite{daniely2019competitive},
suggesting intriguing connections between differential equations and parameter-free online learning.

An improved approach uses PDEs (rather than ODEs) to generate time-dependent potential functions. Still
considering the LEA problem, such works aim at the optimal regret bound nonasymptotic in the
number of experts. Zhu \cite{zhu2014two} first derived a PDE to characterize the continuous-time limit of LEA, whose arguments were streamlined by Drenska and Kohn \cite{drenska2020prediction}. Exact solutions were obtained in special cases \cite{bayraktar2020finite,bayraktar2020asymptotic,drenska2020prediction}, and more generally, algorithms based on approximate solutions were designed in \cite{rokhlin2017pde,kobzar2020a_new,kobzar2020b_new}. Follow-up works considered history-dependent experts \cite{drenska2020pde,drenska2022online} and malicious experts \cite{bayraktar2020malicious,bayraktar2021prediction}. Furthermore, Harvey et al. \cite{harvey2020optimal} extended this idea to the anytime setting with two experts, using a different, stochastic derivation of the continuous-time PDE. 

Our use of PDE in unconstrained OLO is inspired by these works on LEA. Notably, we emphasize two differences. 
\begin{itemize}
\item Existing works considered settings that enforce a unique solution to the PDE, by requiring a fixed time horizon (e.g., \cite{zhu2014two, drenska2020prediction,kobzar2020a_new}) or imposing boundary conditions \cite{harvey2020optimal}. In contrast, we directly consider a class of solutions which are generally not comparable to each other.
\item In LEA, the goal of the PDE approach is to achieve the optimal uniform regret (with respect to all experts). In contrast, we use a PDE to achieve performance trade-offs in adaptive online learning. Trade-offs among experts have been studied using ODEs (e.g., \cite{kapralov2011prediction}). However, we focus on the anytime setting, and the trade-off in unconstrained OLO is with respect to all comparators in $\R^d$, which is more challenging. 
\end{itemize}

\subsection{Notation}

Let $\norms{\cdot}$ be the Euclidean norm, and let $\ball^d$ be the unit $d$-dimensional Euclidean norm ball. For a twice differentiable function $V(t,S)$ where $t$ represents time and $S$ represents a spatial variable, let $\nabla_tV$, $\nabla_{tt}V$, $\nabla_SV$ and $\nabla_{SS}V$ be the first and second order partial derivatives. $\lambda_\max(\cdot)$ is the largest eigenvalue of a real symmetric matrix. For a function $f$, let $f^*$ be its Fenchel conjugate. For two integers $a\leq b$, $[a:b]$ is the set of all integers $c$ such that $a\leq c\leq b$; the brackets are removed when on the subscript, denoting a finite sequence with indices in $[a:b]$. Finally, $\log$ denotes natural logarithm when the base is omitted. 

\section{OLO, betting and limiting PDE}

Our approach critically relies on a continuous-time view of the discrete-time unconstrained OLO problem. It consists of three steps, detailed in the following three subsections. First, we convert OLO to a \emph{coin-betting} problem - the latter is easier from a technical perspective, due to the absence of comparators.

\subsection{Unconstrained coin-betting and duality}

\emph{Unconstrained coin-betting} is a two-person zero-sum game, with $\calX=\R^d$ and $\calC=\ball^d$ being the action space of the player and the adversary respectively. The player's policy $\bm{p}$ contains an initial bet $x_1\in\calX$ and a collection of functions $\{p_2,p_3,\ldots\}$, with $p_t:\calC^{t-1}\rightarrow \calX$. Similarly, the adversary's policy $\bm{a}$ is defined as a collection of functions $\{a_1,a_2,\ldots\}$, with $a_t:\calX^{t}\rightarrow \calC$. Analogous to our setting of unconstrained OLO, randomized betting strategies are not considered.

Fixing policies $\bm{p}$ and $\bm{a}$ on both sides, the game runs as follows. In the $t$-th round, the player makes a bet $x_t=p_t(c_{1:t-1})$ based on past coin outcomes. Then, the adversary decides a new coin $c_t=a_t(x_{1:t})$, reveals it to the player, and the player gains $\inner{c_t}{x_t}$ amount of money (effectively, the player loses money if $\inner{c_t}{x_t}$ is negative). The performance metric for the player is the total gained wealth 
\begin{equation*}
\wel_T=\sum_{t=1}^T\inner{c_t}{x_t},
\end{equation*}
where $T$ is not pre-specified. In other words, the player aims to ensure an anytime wealth lower bound against all possible adversaries. 

Research on adversarial betting has a long history, dating back at least to \cite{cover1966behavior} - Cover studied the setting with a fixed and known time horizon, where all achievable lower bounds can be characterized via dynamic programming. Our anytime setting is more involved, but due to a classical dual relation \cite{mcmahan2014unconstrained}, solving it is \emph{equivalent} to solving the unconstrained OLO problem we ultimately care about: one can construct a unique OLO algorithm (Algorithm~\ref{algorithm:conversion}) from any coin-betting algorithm $\A$, and characterize its performance through Lemma~\ref{lemma:duality}. Consequently, the rest of the paper will focus on solving the betting problem in a principled way.

\begin{algorithm*}[ht]
\caption{From coin-betting to OLO.\label{algorithm:conversion}}
\begin{algorithmic}[1]
\REQUIRE An algorithm $\A$ for unconstrained coin-betting. 
\FOR{$t=1,2,\ldots$}
\STATE Query $\A$ for its $t$-th bet $x_t$ and predict it \emph{exactly} in OLO.
\STATE Observe loss gradient $g_t$ and suffer $\inner{g_t}{x_t}$.
\STATE Let $c_t=-g_t$ and send it to $\A$ as the $t$-th coin outcome. 
\ENDFOR
\end{algorithmic}
\end{algorithm*}

\begin{lemma}[Theorem 9.6 of \citep{orabona2019modern}]\label{lemma:duality}
Let $\Psi$ be any proper, closed and convex function. For all $T\in\N_+$, the following two statements are equivalent: 
\begin{enumerate}
\item The unconstrained coin-betting algorithm $\A$ guarantees $\wel_T\geq \Psi\rpar{\sum_{t=1}^Tc_t}$ against any adversary.
\item The unconstrained OLO algorithm constructed from $\A$ guarantees $\reg_T(u)\leq \Psi^*(u)$ for all $u\in\R^d$, against any adversarial loss sequence. ($\Psi^*$ is the Fenchel conjugate of $\Psi$.)
\end{enumerate}
\end{lemma}

Before proceeding, we note that the above unconstrained coin-betting game strictly generalizes the well-known existing one for unconstrained OLO analysis \cite{mcmahan2014unconstrained,orabona2016coin}. The latter assigns an initial wealth $C$ to the player, and the player's betting amount $\abs{x_t}$ should be less than the total wealth it possesses at the beginning of the $t$-th round. A budget constraint of this form faithfully models many real-world investment problems, but since our ultimate goal is online learning rather than any particular financial application, such a constraint is not necessary for our purpose. In fact, relaxing it gives us greater flexibility to achieve general forms of regret trade-offs beyond (\ref{eq:regret_existing}). Intuitively speaking, the player in our setting can make decisions solely based on the perceived risk-gain trade-off, without being constrained by its budget. 

\subsection{Minimax coin-betting}

For the second step of our derivation, we will characterize the unconstrained coin-betting game from a minimax perspective. Rather than the \emph{value} of the game, we consider a refined quantity called \emph{value function}. 

\begin{definition}[Value function]\label{def:value}
$V:\N\times\R^d\rightarrow\R$ is a value function of the unconstrained coin-betting game if
\begin{enumerate}
\item $V(0,0)=0$.
\item For all $t\in\N$, $V(t,\cdot)$ is continuous on $\R^d$.
\item For all $t\in\N$ and $S\in t\cdot \calC$, \footnote{Even though $\calX$ is not compact, the minimization on the RHS of (\ref{eq:bellman}) is well-posed since $\max_{c\in\calC}\spar{V(t+1,S+c)-\inner{c}{x}}$ as a function of $x$ is coercive.} 
\begin{equation}
V(t,S)=\min_{x\in\calX}\max_{c\in\calC}\spar{V(t+1,S+c)-\inner{c}{x}}.\label{eq:bellman}
\end{equation}
\end{enumerate}
\end{definition}

The recursive relation in Definition~\ref{def:value} is reminiscent of the \emph{conditional value function} previously studied in online learning \cite{rakhlin2012relax,mcmahan2013minimax,drenska2020prediction} and minimax dynamic programming \cite{bertsekas2012dynamic}. The key difference is that we care about anytime performance, therefore a terminal condition to initiate the backward recursion (\ref{eq:bellman}) is missing. Rather than the \emph{value-to-go}, we model the \emph{value-so-far}. This largely complicates the analysis, as the solution of (\ref{eq:bellman}) is not unique (e.g., $V(t,S)=\textrm{constant}\cdot S$). In general, similar to the concept of \emph{Pareto optimality}, different value functions are not comparable as they represent different trade-offs on the \emph{shape} of the wealth lower bound (ultimately, the associated regret upper bound due to Lemma~\ref{lemma:duality}). 

On the bright side, \emph{any} value function can lead to a pair of player-adversary strategies with tight wealth lower and upper bounds. Given a good value function (or more generally, its approximation), a good betting algorithm can be naturally induced. The proof is deferred to Appendix~\ref{subsection:value_policy_proof}.
\begin{restatable}{lemma}{minimax}\label{lemma:minimax}
Given any value function $V$ satisfying Definition~\ref{def:value}, 
\begin{enumerate}
\item We can construct a player policy $\bm{p}^*$ such that for all $\bm{a}$ and $T\in\N_+$, 
\begin{equation*}
\wel_T\geq V\rpar{T,\sum_{t=1}^Tc_t}. 
\end{equation*}
In addition, for all $t$, the player's bet $p^*_t(c_{1:t-1})$ depends on the past coins only through their sum $\sum_{i=1}^{t-1}c_i$. 
\item We can construct an adversary policy $\bm{a}^*$ such that for all $\bm{p}$ and $T\in\N_+$, 
\begin{equation*}
\wel_T\leq V\rpar{T,\sum_{t=1}^Tc_t}. 
\end{equation*}
\end{enumerate}
\end{restatable}

To proceed, let us further define the \emph{unit time} as the time interval between consecutive rounds in the coin betting game, and assign it to 1. In this way, the game can be analyzed on the real time axis $t\in\R_+$. 

\subsection{The scaled game and limiting PDE}

Intuitively, solving the backward recursion (\ref{eq:bellman}) is difficult due to its discrete formulation. If we adopt a finer discretization on the time axis, then the recursion becomes ``smoother'' which is easier to describe using continuous-time analysis. To this end, let us introduce a \emph{scaled coin-betting game}. 

\begin{definition}[Scaled game]\label{def:scaled_game}
Given $\eps>0$, the $\eps$-scaled game is the unconstrained coin-betting game with unit time $\eps^2$ and adversary action space $\eps\cdot\calC$. That is, actions are taken every $\eps^2$ original unit time, and the adversary chooses the coin outcomes in a scaled set $\eps\cdot\calC$ instead of $\calC$.\footnote{We choose such scaling factors due to results in the coin-betting setting with budget constraints \cite{mcmahan2013minimax}. Detailed discussions are presented in Appendix~\ref{subsection:scaling}.}
\end{definition}

Similar to Definition~\ref{def:value}, we can define \emph{$\eps$-scaled value functions} $V_\eps$ on the scaled game. Moreover, we extend its domain and assume it is twice-differentiable on $\R_{>0}\times\R^d$. The backward recursion on $V_\eps$ becomes
\begin{equation*}
V_\eps(t,S)=\min_{x\in\calX}\max_{c\in\calC}\spar{V_\eps(t+\eps^2,S+\eps c)-\inner{\eps c}{x}}.
\end{equation*}
Similar to \cite{zhu2014two,drenska2020prediction}, we take a Taylor approximation on the RHS,
\begin{equation*}
V_\eps(t+\eps^2,S+\eps c)=V_\eps(t,S)+\eps^2\nabla_tV_\eps(t,S)+\eps \inner{c}{\nabla_SV_\eps(t,S)}+\frac{\eps^2}{2}\inner{\nabla_{SS}V_\eps(t,S)\cdot c}{c}+o(\eps^2),
\end{equation*}
which leads to
\begin{equation*}
0=\min_{x\in\calX}\max_{c\in\calC}\big[\inner{c}{\nabla_SV_\eps(t,S)-x}+\eps\nabla_tV_\eps(t,S)+\frac{\eps}{2}\inner{\nabla_{SS}V_\eps(t,S)\cdot c}{c}+o(\eps)\big].
\end{equation*}
As $\eps$ approaches $0$, the dominant term on the RHS is $\min_{x\in\calX}\max_{c\in\calC}\inner{c}{\nabla_SV_\eps(t,S)-x}$, therefore the outer minimizing argument should be $x=\nabla_SV_\eps(t,S)$. Along this argument, taking $\eps\rightarrow 0$ and plugging in $\calC=\ball^d$ (i.e., the unit $d$-dimensional Euclidean norm ball), we obtain a second order nonlinear PDE for a \emph{limiting value function}. 
\begin{definition}[Limiting value function]
A function $\bar V:\R_{>0}\times\R^d\rightarrow \R$ is a limiting value function of the unconstrained coin-betting game if
\begin{equation}\label{eq:PDE}
\nabla_t\bar V=-\frac{1}{2}\max\{\lambda_\max(\nabla_{SS}\bar V), 0\}.
\end{equation}
\end{definition}
The PDE (\ref{eq:PDE}) can be regarded as a continuous-time approximation of the backward recursion (\ref{eq:bellman}), and solving it, while still challenging, is more tractable than handling the discrete-time recursion itself. Given solutions of this PDE, one may invoke a perturbed analysis of Lemma~\ref{lemma:minimax} and obtain corresponding wealth lower bounds. 

\section{One-dimensional analysis}\label{section:1d}

To demonstrate the power of the PDE framework, let us focus on the one-dimensional convex case where the nonlinear PDE (\ref{eq:PDE}) becomes linear. Despite this restriction, our approach can still handle the general $d$-dimensional unconstrained OLO problem due to a standard extension technique \cite{cutkosky2018black} reviewed in Appendix~\ref{section:appendix_olo}. 

For now, let us assume $d=1$. To further comply with the duality lemma (Lemma~\ref{lemma:duality}), we will consider $\bar V$ that are convex with respect to the second argument. Then, the PDE (\ref{eq:PDE}) reduces to the one-dimensional \emph{backward heat equation}
\begin{equation}\label{eq:PDE_1d}
\nabla_t\bar V=-\frac{1}{2}\nabla_{SS}\bar V.
\end{equation}
Such a linear PDE has received considerable interest from the maths community \cite{miranker1961well,payne1975improperly}, since its initial value problem has an intriguing \emph{ill-posed} issue. Interestingly, an insightful work by Harvey et al. \cite{harvey2020optimal} showed that the backward heat equation gives rise to an optimal two-expert LEA algorithm - the proposed techniques will be useful in our analysis as well. Our key observations are twofold:
\begin{itemize}
\item The PDE framework recovers both the OGD potential and the existing parameter-free potential (\ref{eq:value_existing}), thus appears to be a very general approach for unconstrained OLO.
\item The optimal potential that Harvey et al. adopted for two-expert LEA is also strong for \emph{adaptive online learning}, resulting in an optimal unconstrained OLO algorithm \emph{in high-dimensions}.
\end{itemize}

\subsection{PDE-based policy class}

Motivated by the classical parameter-free potential (\ref{eq:value_existing}), let us consider the ansatz
\begin{equation}\label{eq:ansatz}
\bar V(t,S)=t^\alpha g\rpar{c\cdot t^\beta S},
\end{equation}
where $\alpha$, $\beta$ and $c$ are constants, and $g:\R\rightarrow\R$ is a one-dimensional function to be determined. For simplicity we omit shifting on $S$, $t$ and the function value. In other words, once we find appropriate $(\alpha,\beta,c)$ and $g$, we immediately obtain a more general solution
\begin{equation*}
\bar V(t,S)=C_0+(t+\tau)^\alpha g\rpar{c\cdot t^\beta (S+S_0)},
\end{equation*}
with shifting constants $C_0$, $\tau$ and $S_0$. Moreover, any linear combination of two solutions is also a solution, allowing the user to interpolate their induced behavior. 

Plugging in (\ref{eq:ansatz}) and letting $z=c\cdot t^\beta S$, the PDE (\ref{eq:PDE_1d}) reduces to a second order linear ODE for the function $g$:
\begin{equation*}
c^2t^{2\beta+1}g''(z)+2\beta zg'(z)+2\alpha g(z)=0.
\end{equation*}
Letting $\beta=-1/2$ and $c=1/\sqrt{2}$, it becomes the standard \emph{Hermite} type
\begin{equation}\label{eq:hermite}
g''(z)-2zg'(z)+4\alpha g(z)=0, 
\end{equation}
whose general solutions can be expressed in power series \citep[Chapter~7]{arfken2013mathematical}. By varying the parameter $\alpha$, we obtain a rich class of limiting value functions $\bar V$. 

To construct coin-betting policies, our key idea is to use $\bar V$ as a surrogate for the actual value function $V$ (Definition~\ref{def:value}) and apply the same argument as in Lemma~\ref{lemma:minimax}. Specifically, the adversary should pick the coin outcome that maximizes the RHS of the backward recursion (\ref{eq:bellman}), which is
\begin{equation}
c_{t}\in\argmax_{c\in\calC}\spar{\bar V\rpar{t,\sum_{i=1}^{t-1}c_i+c}-\inner{c}{x_{t}}}.\label{eq:adversary_coin}
\end{equation}
Since $\bar V$ is convex and $\calC=[-1,1]$, the adversary can simply focus on the boundary coins $\{-1,1\}$, leading to the adversary policy presented in Algorithm~\ref{algorithm:adversary}.

\begin{algorithm*}[ht]
\caption{PDE-based adversary policy.\label{algorithm:adversary}}
\begin{algorithmic}[1]
\REQUIRE A limiting value function $\bar V$ for 1d unconstrained coin-betting. 
\FOR{$t=1,2,\ldots$}
\STATE Receive the player's bet $x_t$ and choose the coin outcome as
\begin{equation}\label{eq:adversary_strategy}
c_{t}\in\argmax_{c\in\{-1,1\}}\spar{\bar V\rpar{t,\sum_{i=1}^{t-1}c_i+c}-\inner{c}{x_{t}}}.
\end{equation}
\ENDFOR
\end{algorithmic}
\end{algorithm*}

As for the player, the optimal bet is the one that minimizes the objective function in (\ref{eq:adversary_coin}), which is equivalent to the discrete derivative shown in Algorithm~\ref{algorithm:player}. Intuitively, the discrete derivative serves as an approximation of the standard derivative in classical potential methods. Therefore, Algorithm~\ref{algorithm:player} essentially has a potential-based structure, with the potential function $\bar V$ generated from a PDE. Alternatively, Algorithm~\ref{algorithm:player} can be interpreted as a discrete approximation of \emph{Follow the Regularized Leader} (FTRL) \cite{abernethy2008competing} whose regularizer is the Fenchel conjugate of $\bar V(t,\cdot)$. The equivalence of potential functions and regularizers has been discussed in \citep[Section~7.3]{orabona2019modern}. 

\begin{algorithm*}[ht]
\caption{PDE-based player policy.\label{algorithm:player}}
\begin{algorithmic}[1]
\REQUIRE A limiting value function $\bar V$ for 1d unconstrained coin-betting. 
\FOR{$t=1,2,\ldots$}
\STATE Choose the bet
\begin{equation}\label{eq:player_strategy}
x_{t}=\frac{1}{2}\spar{\bar V\rpar{t,\sum_{i=1}^{t-1}c_i+1}-\bar V\rpar{t,\sum_{i=1}^{t-1}c_i-1}}.
\end{equation}
\STATE Observe the coin outcome $c_t$ and store it. 
\ENDFOR
\end{algorithmic}
\end{algorithm*}

\subsection{Example}

Before any technical analysis, let us demonstrate the generality of this framework through a few examples. We show how classical algorithms can be derived from this framework, and more importantly, we present a potential function which permits an optimal loss-regret trade-off. For any $\alpha$, let $\bar V_\alpha$ be a limiting value function obtained from (\ref{eq:hermite}). Let $C>0$ be any positive scaling constant. 

\paragraph{Warm up: $\alpha=1$.} The Hermite ODE (\ref{eq:hermite}) has a solution $g(z)=C(2z^2-1)$, resulting in $\bar V_1(t,S)=C(S^2-t)$. Accordingly, Algorithm~\ref{algorithm:player} bets $x_t=2C\sum_{i=1}^{t-1}c_i=x_{t-1}+2Cc_{t-1}$, which is equivalent to \emph{Online Gradient Descent} (OGD) with learning rate $2C$. Notably, $\bar V_1$ also satisfies Definition~\ref{def:value}; that is, $\bar V_1$ is not only a limiting value function, but \emph{also a value function} for the discrete-time game. Therefore, both Algorithm~\ref{algorithm:adversary} and Algorithm~\ref{algorithm:player} can be \emph{directly} analyzed through Lemma~\ref{lemma:minimax}, as shown in Appendix~\ref{subsection:special}.

\paragraph{Recovering existing potentials: $\alpha=-1/2$.} The Hermite ODE can be solved by $g(z)=C\exp(z^2)$, resulting in $\bar V_{-1/2}(t,S)=C\cdot t^{-1/2}\exp[S^2/(2t)]$. Such a potential recovers the existing popular choice (\ref{eq:value_existing}), and its time shifted version $C\cdot (t+\tau)^{-1/2}\exp[S^2/(2(t+\tau))]$ naturally recovers the \emph{shifted potential} \cite{orabona2016coin} with minimum effort. Different from the previous example, $\bar V_{-1/2}$ does not satisfy Definition~\ref{def:value}. Therefore, we should characterize its approximation error on the backward recursion (\ref{eq:bellman}) in order to quantify the performance of the induced player policy. This procedure will be demonstrated in the next subsection. 

\paragraph{A new potential: $\alpha=1/2$.} The two linearly independent solutions of the Hermite ODE are both useful. First, $g(z)=\sqrt{2}Cz$ and $\bar V(t,S)=CS$. Such a potential leads to betting a fixed amount in coin-betting and shifting the coordinate system in unconstrained OLO. The idea is simple, and it will be applied in our experiments. For now, let us focus on the other solution which is more interesting. 
\begin{equation*}
g(z)=C\spar{2z\cdot \int_{0}^z\exp(x^2)dx-\exp(z^2)},
\end{equation*}
and the corresponding potential is
\begin{equation}
\bar V_{1/2}(t,S)=C\sqrt{t}\spar{2\int_0^{\frac{S}{\sqrt{2t}}}\rpar{\int_0^u\exp(x^2)dx}du-1}.\label{eq:potential_novel}
\end{equation}

Notably, as shown in Appendix~\ref{subsection:derivative}, $\nabla_{SS}\bar V_{1/2}(t,S)=\bar V_{-1/2}(t,S)$, suggesting possible deeper connections to the existing parameter-free algorithms. Harvey et al. \cite{harvey2020optimal} constructed a two-expert LEA algorithm from $\bar V_{1/2}$, which achieves the optimal uniform regret. As for unconstrained OLO, we will show that using $\bar V_{1/2}$ in Algorithm~\ref{algorithm:player} leads to superior performance compared to $\bar V_{-1/2}$, both in theory and in practice. Without the help of a PDE, such a potential has not been discovered in adaptive online learning before (to the best of our knowledge); this emphasizes the value of the PDE-based framework.

\subsection{Analysis of Algorithm~\ref{algorithm:player}}

Now we provide rigorous performance guarantees for the PDE-based player policy (Algorithm~\ref{algorithm:player}). To begin with, define discrete derivatives of a limiting value function $\bar V$ as
\begin{equation*}
\bar \nabla_t \bar V(t,S)=\bar V(t,S)-\bar V(t-1,S),
\end{equation*}
\begin{equation*}
\bar \nabla_{SS} \bar V(t,S)=\bar V(t,S+1)+\bar V(t,S-1)-2\bar V(t,S).
\end{equation*}
When doing this we extend the domain of $\bar V(t,S)$ to $t=0$, and assign $\bar V(0,0)=0$ without loss of generality. 

The key component of this analysis is the \emph{Discrete It\^{o} formula} \cite{klenke2013probability,harvey2020optimal}. We modify it for the coin-betting problem, and the proof is provided in Appendix~\ref{subsection:ito}. 

\begin{restatable}[Lemma D.3 and D.4 of \cite{harvey2020optimal}, adapted]{lemma}{ito}\label{lemma:ito}
Consider applying Algorithm~\ref{algorithm:player} against any adversary coin-betting policy $\bm{a}$. For all $t\in\N$, 
\begin{equation}\label{eq:ito}
\bar V\rpar{t+1,\sum_{i=1}^{t+1}c_i}-\bar V\rpar{t,\sum_{i=1}^{t}c_i}\leq c_{t+1}x_{t+1}+\underbrace{\spar{\bar\nabla_t\bar V\rpar{t+1,\sum_{i=1}^tc_i}+\frac{1}{2}\bar\nabla_{SS}\bar V\rpar{t+1,\sum_{i=1}^tc_i}}}_{\Diamond}.
\end{equation}
Moreover, equality is achieved when $c_{t+1}\in\{-1,1\}$.
\end{restatable}

Summing (\ref{eq:ito}) over $t\in[0:T-1]$, the LHS becomes a telescopic sum which returns $\bar V(T,\sum_{i=1}^T c_i)$, and the RHS contains $\wel_T=\sum_{t=1}^T c_tx_t$ which we aim to bound - the remaining task is to quantify the sum $\Diamond$ in the bracket. Comparing $\Diamond$ to the backward heat equation (\ref{eq:PDE_1d}), one can see that $\Diamond$ represents the ``discrete approximation error'' on the PDE. Bounding this error is case-dependent: we will only consider $\bar V_{1/2}$ in the following, and the analysis of $\bar V_{-1/2}$ is deferred to Appendix~\ref{subsection:detail_neg}. 

\begin{restatable}{lemma}{casethree}\label{lemma:upper_lower_three}
For all $t\in\N_+$ and $S\in[1-t,t-1]$, $\bar V_{1/2}$ with any parameter $C>0$ satisfies
\begin{equation*}
0\geq\bar\nabla_t\bar V_{1/2}(t,S)+\bar\nabla_{SS}\bar V_{1/2}(t,S)/2\geq \begin{cases}
-C,&t=1,\\
-\frac{C}{8}(t-1)^{-3/2}\exp\rpar{\frac{S^2}{2(t-1)}}\rpar{\frac{S^2}{t-1}+1},&t>1.
\end{cases}
\end{equation*}
\end{restatable}

Combining the above, we immediately obtain a wealth lower bound (Theorem~\ref{thm:lower_3}) for the player policy constructed from $\bar V_{1/2}$. Its proof is due to a telescopic sum therefore omitted. 

\begin{theorem}\label{thm:lower_3}
For all $T\in\N_+$, Algorithm~\ref{algorithm:player} constructed from $\bar V_{1/2}$ guarantees a wealth lower bound
\begin{equation*}
\wel_T\geq \bar V_{1/2}\rpar{T,\sum_{t=1}^Tc_t},
\end{equation*}
against any adversary policy $\bm{a}$. 
\end{theorem}

Furthermore, by applying the analysis on the opposite direction, the following theorem shows that Algorithm~\ref{algorithm:adversary} is a strong adversary policy to confront Algorithm~\ref{algorithm:player}. That is, the pair of player-adversary policies induced by $\bar V_{1/2}$ has a ``dual property''. Note that when the player applies Algorithm~\ref{algorithm:player}, both $c_t=-1$ and $c_t=1$ satisfy (\ref{eq:adversary_strategy}), therefore Algorithm~\ref{algorithm:adversary} can freely choose from these two boundary coins. The proof is deferred to Appendix~\ref{subsection:onehalf}.

\begin{restatable}{theorem}{upper}\label{thm:upper_3}
For all $T\in\N_+$ and $S\in[-T,T]$, we can construct $c_1\in\calC$ and $c_2,\ldots,c_T\in\{-1,1\}$ such that
\begin{enumerate}
\item $\sum_{t=1}^Tc_t=S$;
\item If the player applies Algorithm~\ref{algorithm:player} constructed from $\bar V_{1/2}$ (with parameter $C$) and the adversary plays the aforementioned coin sequence $c_{1:T}$, then
\begin{equation*}
\wel_T\leq \bar V_{1/2}\rpar{T,S}+\frac{3C}{8}\exp\rpar{\frac{S^2}{2T}}\rpar{\frac{S^2}{T}+1}+2C.
\end{equation*}
\end{enumerate}
\end{restatable}

Comparing Theorem~\ref{thm:lower_3} to Theorem~\ref{thm:upper_3}, the lower and upper bound are separated by at most a constant when $\sum_{t=1}^Tc_t=O(\sqrt{T})$. It means that \emph{everywhere} on the set $\{(t,S)|S=O(\sqrt{t})\}$, the value of $\bar V_{1/2}(t,S)$ provides a tight performance guarantee for Algorithm~\ref{algorithm:player}.

\subsection{Optimality of Algorithm~\ref{algorithm:player}}

The previous wealth upper bound shows that Theorem~\ref{thm:lower_3} faithfully characterizes the performance of Algorithm~\ref{algorithm:player}, but does not address the optimality of this betting policy. To this end, we now present a wealth upper bound that holds for all betting policies. The proof is deferred to Appendix~\ref{subsection:detail_lower_proof}.

\begin{restatable}{theorem}{optimality}\label{thm:optimality}
For all $\lambda\geq \exp[(\sqrt{2}+1)/2]$, $T\geq8\pi\lambda^2\log\lambda$, and any player policy $\bm{p}$ that guarantees $\wel_T\geq -C\sqrt{T}$ (e.g., Algorithm~\ref{algorithm:player} constructed from $\bar V_{1/2}$), there exists an adversary policy $\bm{a}$ such that the following statement holds. In the coin-betting game induced by the policy pair $(\bm{p},\bm{a})$, 
\begin{enumerate}
\item $|\sum_{t=1}^Tc_t|\geq\sqrt{2T\log\lambda}$;
\item $\wel_T\leq 2\sqrt{2\pi}\lambda\sqrt{\log\lambda}\cdot C\sqrt{T}$.
\end{enumerate}
\end{restatable}

The proof of Theorem~\ref{thm:optimality} is based on a stochastic adversary argument similar to \cite{mcmahan2012no,orabona2013dimension}. However, using an improved lower bound for the tail probability of random walks, our wealth upper bound is tight up to a poly-logarithmic factor. To see this, let us compare it to the wealth lower bound from Theorem~\ref{thm:lower_3}: if $|\sum_{t=1}^Tc_t|\geq \sqrt{2T\log\lambda}$, then Algorithm~\ref{algorithm:player} guarantees (the last inequality due to Lemma~\ref{lemma:lower_easy})
\begin{equation*}
\wel_T\geq \bar V_{1/2}\rpar{t,\sqrt{2T\log\lambda}}\geq C\sqrt{T}\spar{\frac{\lambda}{2\log\lambda}-\frac{3}{2}}.
\end{equation*}
For comparison, previous analysis \cite{orabona2013dimension} only guarantees the suboptimal rate $\wel_T\leq \tilde O(\lambda^{\log 4}\sqrt{T})$. Later we will see that matching the $\tilde O(\lambda\sqrt{T})$ factor in the wealth bounds leads to matching the leading term (\emph{including} the multiplicative constant) in the regret of OLO. 

\section{Optimal unconstrained OLO}\label{section:olo}

This section presents our main results on unconstrained OLO. Notice that using the conversion from coin-betting to OLO (Algorithm~\ref{algorithm:conversion}), our PDE-based betting strategy (Algorithm~\ref{algorithm:player}) can be directly converted into a one-dimensional unconstrained OLO algorithm with a potential structure. For clarity, its pseudo-code is restated as Algorithm~\ref{algorithm:combined_1d} in Appendix~\ref{section:appendix_olo}. To further extend it to $\R^d$, a standard polar decomposition technique \cite{cutkosky2018black} will be adopted. Combining everything, our final product is a general unconstrained OLO algorithm (Algorithm~\ref{algorithm:combined}) constructed from any solution of the one-dimensional PDE (\ref{eq:PDE_1d}). 

\begin{algorithm*}[ht]
\caption{PDE-based unconstrained OLO algorithm.\label{algorithm:combined}}
\begin{algorithmic}[1]
\REQUIRE A one-dimensional limiting value function $\bar V$ which satisfies (\ref{eq:PDE_1d}). 
\STATE Define $\A_B$ as the standard Online Gradient Descent (OGD) on $\ball^d$ with learning rate $\eta_t=1/\sqrt{t}$, initialized at the origin. 
\STATE Initialize a parameter (``sufficient statistic'') $S_1=0$. 
\FOR{$t=1,2,\ldots$}
\STATE Let $y_{t}=\spar{\bar V\rpar{t,S_t+1}-\bar V\rpar{t,S_t-1}}/2$.
\STATE Query $\A_B$ for its $t$-th prediction and assign it to $z_t$. 
\STATE Predict $x_t=y_tz_t\in\R^d$.
\STATE Observe $g_t\in\R^d$ generated by an adversary ($g_t$ can depend on $x_1,\ldots,x_t$).
\STATE Return $g_t$ as the $t$-th loss gradient to $\A_B$, and let $S_{t+1}=S_t-\inner{g_t}{z_t}$.
\ENDFOR
\end{algorithmic}
\end{algorithm*}

Let us consider Algorithm~\ref{algorithm:combined} constructed from $\bar V_{1/2}$ (\ref{eq:potential_novel}), with proofs deferred to Appendix~\ref{subsection:OLO_detail}. Recall that $C$ is any positive scaling constant. Converting the coin-betting lower bound (Theorem~\ref{thm:lower_3}) to OLO, we have

\begin{restatable}{theorem}{oloupper}\label{thm:olo_upper}
For all $T\in\N_+$ and $u\in\R^d$, against any adversary, Algorithm~\ref{algorithm:combined} constructed from $\bar V_{1/2}$ guarantees
\begin{equation*}
\reg_T(u)\leq C\sqrt{T}+\norm{u}\sqrt{2T}\spar{\sqrt{\log\rpar{1+\frac{\norm{u}}{\sqrt{2}C}}}+2}.
\end{equation*}
\end{restatable}

Theorem~\ref{thm:olo_upper} offers two advantages over existing results. 
\begin{enumerate}
\item It is a naturally anytime bound with the optimal\footnote{In the sense that the asymptotic rate on $T$ alone is optimal. That is, compared to the optimally tuned gradient descent algorithm with regret $O(\norm{u}T)$, the \emph{price of being parameter-free} is only an extra $\sqrt{\log\norm{u}}$ factor.} loss-regret trade-off, i.e., $\reg_T(u)=O\rpar{\norm{u}\sqrt{T\log\norm{u}}}$, shaving a $\sqrt{\log T}$ factor from most existing bounds like (\ref{eq:regret_existing}). Actually, as discussed in the introduction, prior works \emph{can} achieve this optimal trade-off, but they rely on the impractical doubling trick \cite{shalev2011online} for an anytime bound. In contrast, our algorithm has a more efficient potential structure, thus making the optimal loss-regret trade-off practical for real-world applications. 
\item In addition, Theorem~\ref{thm:olo_upper} also attains the optimal leading term, \emph{including} the multiplying constant $\sqrt{2}$. To our knowledge, this is the first parameter-free bound with the leading constant optimality. The precise statement is the following, derived from the wealth upper bound (Theorem~\ref{thm:optimality}). For clarity, we write $\reg_T^{\A, adv}(u)$ as the regret induced by an algorithm $\A$ and an adversary $adv$. 
\end{enumerate}

\begin{restatable}{theorem}{interpret}\label{thm:interpret}
Define $\A_{1/2}$ as Algorithm~\ref{algorithm:combined} constructed from $\bar V_{1/2}$, then Theorem~\ref{thm:olo_upper} leads to
\begin{equation*}
\limsup_{U\rightarrow\infty}\limsup_{T\rightarrow\infty}\sup_{\norms{u}=U, adv}\frac{\reg_T^{\A_{1/2},adv}(u)}{\norms{u}\sqrt{T\log\norms{u}}}\leq \sqrt{2}. 
\end{equation*}
Conversely, for all $C$ and any unconstrained OLO algorithm $\A$ (e.g., $\A_{1/2}$) that guarantees $\reg_T^{\A,adv}(0)\leq C\sqrt{T}$ for all $adv$ and $T$, we have
\begin{equation*}
\liminf_{U\rightarrow\infty}\liminf_{T\rightarrow\infty}\sup_{\norms{u}=U, adv}\frac{\reg_T^{\A,adv}(u)}{\norms{u}\sqrt{T\log\norms{u}}}\geq \sqrt{2}. 
\end{equation*}
\end{restatable}

Finally, parallel results based on $\bar V_{-1/2}$ are left to Appendix~\ref{subsection:OLO_detail_alt}. We also convert the player-dependent coin-betting upper bound (Theorem~\ref{thm:upper_3}) to OLO, presented in Appendix~\ref{subsection:OLO_dual}. It estimates the performance of our one-dimensional OLO algorithm (Algorithm~\ref{algorithm:combined_1d}) up to a small error term that does not grow with time. 

\section{Experiment}\label{section:empirical}

Our theoretical results are supported by experiments.\footnote{Code is available at \url{https://github.com/zhiyuzz/ICML2022-PDE-Potential}.} In this section, we test our one-dimensional unconstrained OLO algorithm (Algorithm~\ref{algorithm:combined_1d}) on a synthetic \emph{Online Convex Optimization} (OCO) task, based on the standard reduction from OCO to OLO. Its simplicity allows us to directly compute the regret, thus clearly demonstrate the benefit of $\bar V_{1/2}$ over the existing potential $\bar V_{-1/2}$. Additional experiments are deferred to Appendix~\ref{subsection:exp_omitted_regression}, including ($i$) a one-dimensional OLO task with stochastic loss; and ($ii$) a high-dimensional regression task with real-world data. 

Let us consider a simple one-dimensional OCO problem with time invariant loss function $|x_t-u^*|$, where $u^*\in\R$ is a constant hidden from the online learning algorithm. Reduced into OLO \citep[Section 2.3]{orabona2019modern}, the adversary picks the loss gradient $g_t=1$ if $x_t\geq u^*$, while $g_t=-1$ otherwise. The most natural comparator is the hidden constant $u^*$, and the induced regret of OLO can be nicely interpreted as the \emph{cumulative loss} of OCO. That is, $\reg_T(u^*)=\sum_{t=1}^Tg_t(x_t-u^*)=\sum_{t=1}^T|x_t-u^*|$. We will test three algorithms: ($i$) Algorithm~\ref{algorithm:combined_1d} constructed from $\bar V_{1/2}$ (our main contribution); ($ii$) Algorithm~\ref{algorithm:combined_1d} constructed from $\bar V_{-1/2}$; and ($iii$) the classical \emph{Krichevsky-Trofimov} (KT) algorithm \cite{orabona2016coin} which is an optimistic version of ($ii$) with similar guarantees. Each algorithm requires one hyperparameter: we set $C=1$ for the first two, and set the ``initial wealth'' as $\sqrt{e}$ for KT. Such choices make a fair comparison, as discussed in Appendix~\ref{subsection:hyper}. 

Since $\reg_T(u^*)$ depends on both $u^*$ and $T$, there are multiple ways to visualize our results. In Figure~\ref{fig:1a}, we fix $u^*=10$ and plot $\reg_T(u^*)$ as a function of $T$ (lower is better), with more settings of $u^*$ tested in Appendix~\ref{subsection:exp_omitted_1d}. For comparison, we also plot the regret upper bound based on $\bar V_{1/2}$ (Corollary~\ref{thm:olo_upper_dep}). Consistent with our theory, ($i$) the upper bound (red dashed) closely captures the actual performance of our algorithm (blue); ($ii$) the two baselines (orange and green) exhibit similar performance, and our algorithm improves both when $u^*=10$.

In Figure~\ref{fig:1b}, we fix $T=500$ and plot the difference between the regret of KT and our algorithm (i.e., $\reg_T(u^*)|_{KT}-\reg_T(u^*)|_{ours}$ as a function of $u^*$, higher means our algorithm improves the KT baseline by a larger margin). The obtained curve demonstrates the benefit of our special loss-regret trade-off: while sacrificing the regret at small $|u^*|$, our algorithm significantly improves the baseline when $u^*$ is far-away. Notably, the magnitude of $|u^*|$ represents the quality of initialization: with an oracle guess $\tilde u$, one can shift the origin to $\tilde u$, and the effective distance to $u^*$ becomes $|\tilde u-u^*|$. Figure~\ref{fig:1b} shows that in order to beat our algorithm, the baseline has to guess $u^*$ beforehand with error at most 1, which is obviously very hard. Therefore, our algorithm prevails in most situations.

To strengthen the intuition, let us fix $u^*=100$ and take a closer look at the progression of predictions $x_t$ (Figure~\ref{fig:1c}). Similar to both baselines, our algorithm approaches $u^*$ with exponentially growing speed at the beginning, which is a key benefit of parameter-free algorithms over gradient descent \citep[Section 5]{orabona2017training}. However, after overshooting, the prediction of our algorithm exhibits a much smaller ``dip''. This aligns with the intuition, as our algorithm allows higher $\reg_T(0)$. In other words, compared to the baselines, our algorithm has a weaker belief that the initialization is correct; instead, it believes more in the incoming information. Such a property leads to advantages when the initialization is indeed far from the optimum. 

\begin{figure*}[t]
     \centering
     \begin{subfigure}[b]{0.32\textwidth}
         \centering
         \includegraphics[width=\textwidth]{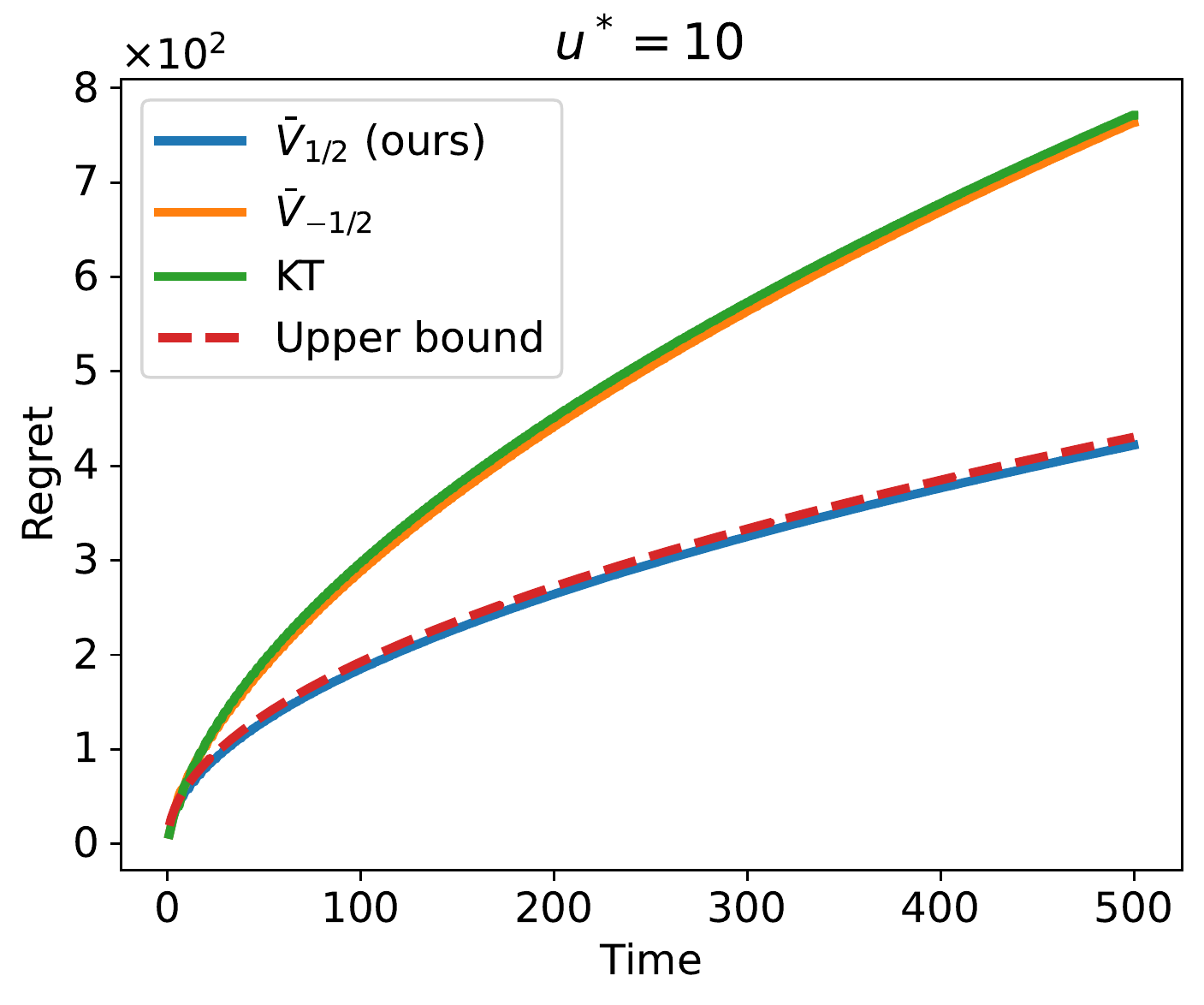}
         \caption{Fix $u^*$ and vary $T$.}\label{fig:1a}
     \end{subfigure}
     \hfill
     \begin{subfigure}[b]{0.33\textwidth}
         \centering
         \includegraphics[width=\textwidth]{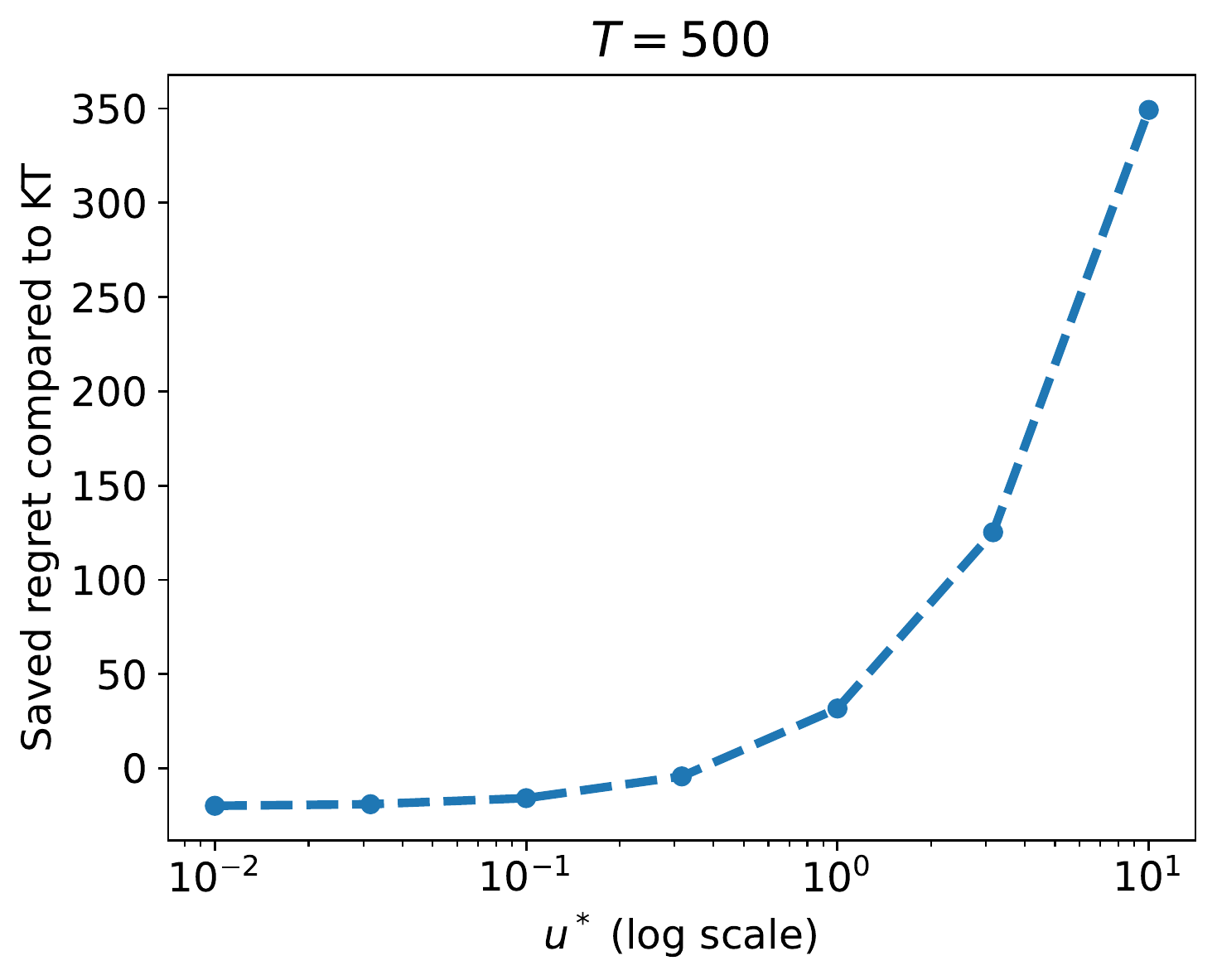}
         \caption{Fix $T$ and vary $u^*$.}\label{fig:1b}
     \end{subfigure}
          \hfill
     \begin{subfigure}[b]{0.33\textwidth}
         \centering
         \includegraphics[width=\textwidth]{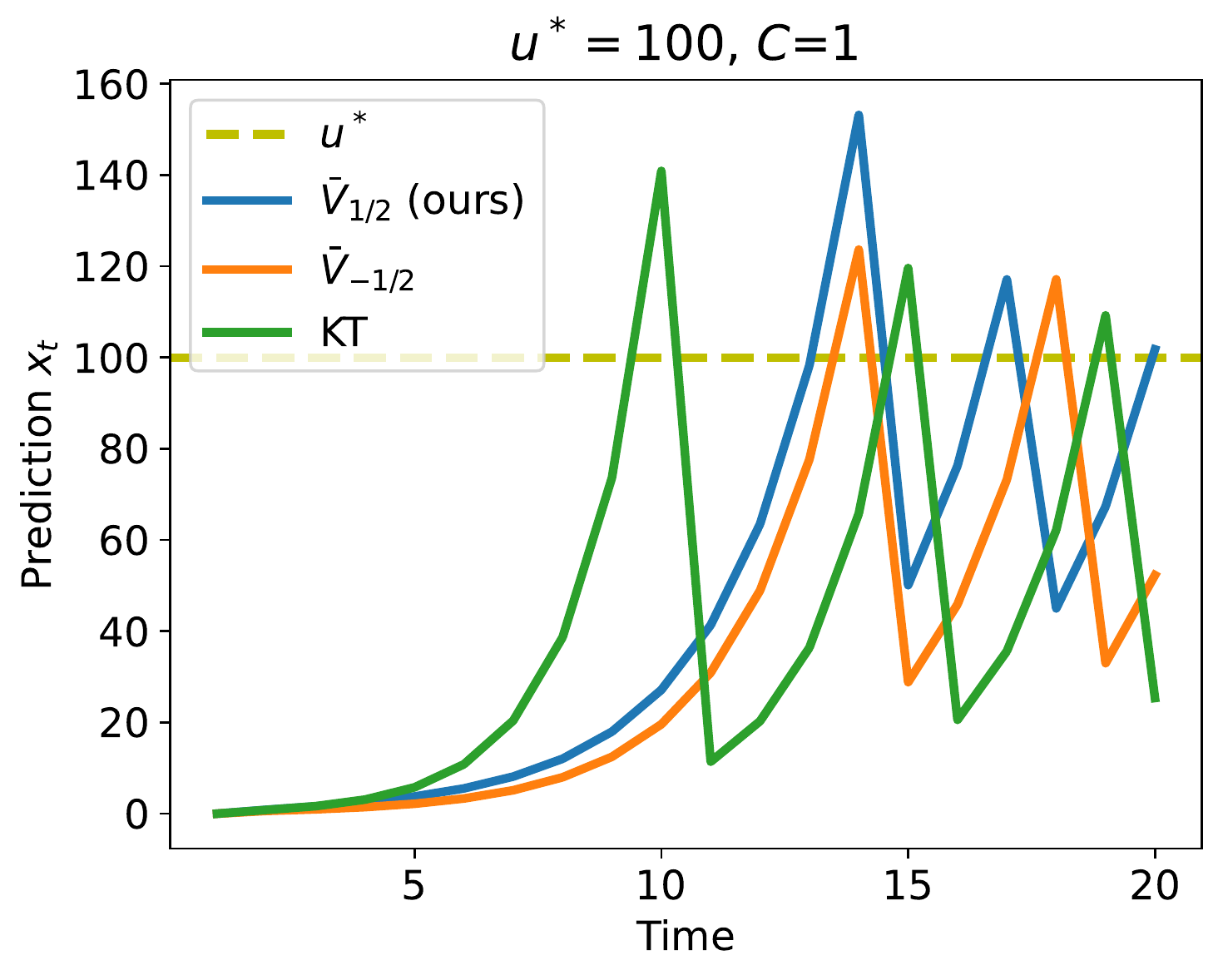}
         \caption{The progression of $x_t$.}\label{fig:1c}
     \end{subfigure}
        \caption{One-dimensional synthetic task with loss $|x_t-u^*|$. Specifically, Subfigure (b) fixes $T=500$ and plots $\reg_T(u^*)$ of KT minus $\reg_T(u^*)$ of our algorithm ($\bar V_{1/2}$) as a function of $u^*$.}
        \label{fig:onedimensional}
\end{figure*}

\section{Conclusion}

We propose a framework that generates unconstrained OLO potentials by solving a PDE. It reduces the amount of guessing in the current workflow, thus simplifying the discovery and analysis of more complicated potentials. To demonstrate its power, we use this framework to design a concrete algorithm - it achieves the optimal loss-regret trade-off without any impractical doubling trick, and moreover, attains the optimal leading constant. Such properties lead to practical advantages when a good initialization is not available. 

Overall, we feel the continuous-time perspective adopted in this paper, based on a series of recent works \cite{drenska2020prediction,harvey2020optimal}, could be a powerful tool for adaptive online learning in general. Several interesting directions are left open:
\begin{itemize}
\item In order to further improve practicality, can we use the new potential to achieve gradient-adaptive \cite{cutkosky2018black,cutkosky2019artificial,mhammedi2020lipschitz} bounds? 
\item Can the PDE framework achieve adaptivity or trade-offs in a broader range of online learning problems? For example, with bandit feedback, switching cost, delays, etc. 
\item Is there a more principled way to handle the obtained PDEs, without enough boundary conditions? Can we \emph{automate} the discovery and verification of new potentials? 
\end{itemize}

\section*{Acknowledgement}

We thank Francesco Orabona for valuable discussion and the anonymous reviewers for their constructive feedback. This research was partially supported by the NSF under grants IIS-1914792, DMS-1664644, and CNS-1645681, by the ONR under grants N00014-19-1-2571 and N00014-21-1-2844, by the DOE under grants DE-AR-0001282 and DE-EE0009696, by the NIH under grants R01 GM135930 and UL54 TR004130, and by Boston University. 

\bibliography{PDE_OLO}

\newpage
\section*{Appendix}
\appendix

\paragraph{Organization} Appendix~\ref{section:appendix_derivation} presents details on the derivation of our PDE (\ref{eq:PDE}). Appendix~\ref{section:appendix_betting} solves the one-dimensional PDE and verifies the quality of the induced coin-betting policies. Appendix~\ref{section:appendix_olo} converts our theoretical results on coin-betting to unconstrained OLO. Appendix~\ref{section:appendix_experiments} presents details on our experiments. 

\section{Detail on the derivation of PDE}\label{section:appendix_derivation} 
In this section we present two aspects of our PDE derivation omitted in the main paper. First, we prove Lemma~\ref{lemma:minimax} which shows that any value function can naturally induce a pair of ``dual'' player-adversary policies. Next, we discuss our choice of scaling factors for the scaled game (Definition~\ref{def:scaled_game}).

\subsection{Proof of Lemma~\ref{lemma:minimax}}\label{subsection:value_policy_proof}

\minimax*

\begin{proof}[Proof of Lemma~\ref{lemma:minimax}]
We only prove the first part by induction. The proof of the second part is similar, therefore omitted. Let us restate the backward recursion (\ref{eq:bellman}),
\begin{equation*}
V(t,S)=\min_{x\in\calX}\max_{c\in\calC}\spar{V(t+1,S+c)-\inner{c}{x}}.
\end{equation*}

Starting from $t=0$ and $S=0$, let $x_1$ be the outer minimizing argument. Then, for all adversary policy $a_1$ such that $c_1=a_1(x_1)$, we have $V(1,c_1)=V(1,c_1)-V(0,0)\leq \inner{c_1}{x_1}$. 

Now consider the following induction hypothesis: there exists $T\in\N_+$, initial bet $x_1$ and functions $p^*_2,\ldots,p^*_T$ such that for all $\bm{a}$, 
\begin{equation*}
\sum_{t=1}^T\inner{c_t}{x_t}\geq V\rpar{T,\sum_{t=1}^Tc_t}.
\end{equation*}
Plugging $(t,S)=(T,\sum_{t=1}^Tc_t)$ into the backward recursion, 
\begin{equation*}
V\rpar{T,\sum_{t=1}^Tc_t}=\min_{x_{T+1}\in\calX}\max_{c_{T+1}\in\calC}\spar{V\rpar{T+1,\sum_{t=1}^{T+1}c_t}-\inner{c_{T+1}}{x_{T+1}}}.
\end{equation*}
Given the value function $V$, there exists $x_{T+1}$ only depending on $T$ and $\sum_{t=1}^Tc_t$ such that for all $c_{T+1}$, 
\begin{equation*}
V\rpar{T,\sum_{t=1}^Tc_t}\geq V\rpar{T+1,\sum_{t=1}^{T+1}c_t}-\inner{c_{T+1}}{x_{T+1}}.
\end{equation*}
Define the policy $p^*_{T+1}$ in this way, we have
\begin{equation*}
\sum_{t=1}^{T+1}\inner{c_t}{x_t}\geq V\rpar{T+1,\sum_{t=1}^{T+1}c_t}.\qedhere
\end{equation*}
\end{proof}

\subsection{The choice of scaling factors}\label{subsection:scaling}

We now discuss our choice of scaling factors for the scaled game (Definition~\ref{def:scaled_game}). To begin with, let us review the wealth lower bounds for the existing coin-betting setting \cite{mcmahan2013minimax,orabona2016coin} with budget constraints. For simplicity, assume $d=1$. Inspired by the celebrated Kelly bettor \cite{kelly1956new}, McMahan and Abernethy \cite{mcmahan2013minimax} made an interesting observation: if starting from an initial wealth $C$ and knowing the bias of future coins ($\sum_{t=1}^Tc_t/T$), the player could bet a fixed fraction $\beta=\sum_{t=1}^Tc_t/T$ of his wealth in each round and guarantees a wealth lower bound
\begin{equation}\label{eq:wealth_existing}
\wel_T\geq C\exp\rpar{\frac{(\sum_{t=1}^Tc_t)^2}{2T}}.
\end{equation}
Of course, this strategy is not implementable in reality. However, using a time-dependent betting fraction $\beta_t$ that reflects the bias \emph{observed online}, the player can actually implement a strategy \cite{orabona2016coin} with
\begin{equation*}
\wel_T\geq \frac{C}{\sqrt{T}}\exp\rpar{\frac{(\sum_{t=1}^Tc_t)^2}{2T}},
\end{equation*}
which matches (\ref{eq:wealth_existing}) in the important exponential factor. Under the presence of budget constraints, such an exponential factor is optimal. Extended to our unconstrained coin-betting setting (which is a strict generalization of the existing one), this exponential factor characterizes the best result when the player can only tolerate a fixed amount of total loss. This is intuitively similar to the concept of \emph{Pareto optimality}: the optimal policy for the player depends on how risk-tolerant it is. 

Back to the design of scaling factors for the $\eps$-scaled game, our guideline is simple: the baseline strategy \cite{mcmahan2013minimax} discussed above should guarantee the \emph{same} wealth bound in the scaled game and the original game. In this way, the PDE derived in the scaling limit could recover the specific Pareto optimal result (\ref{eq:wealth_existing}). Concretely, let $f(\eps)$ and $g(\eps)$ be the scaling factors on the unit time and the coin space respectively. Without loss of generality, let $g(\eps)=\eps$; we now justify our choice $f(\eps)=\eps^2$.

Consider an extreme adversary whose decisions are always 1. In the original game, the baseline strategy \cite{mcmahan2013minimax} guarantees $\wel_T\geq C\exp(T/2)$ due to (\ref{eq:wealth_existing}). In the scaled game, the adversary decisions are scaled by $\eps$, and the total number of decision rounds is $[f(\eps)]^{-1}$ times as many. Therefore, the baseline strategy \cite{mcmahan2013minimax} guarantees
\begin{equation*}
\wel_T\geq C\exp\rpar{\frac{(\eps T\cdot [f(\eps)]^{-1})^2}{2T\cdot [f(\eps)]^{-1}}}=C\exp\rpar{\frac{\eps^2 T\cdot [f(\eps)]^{-1}}{2}}. 
\end{equation*}
If $f(\eps)=\eps^2$, then the wealth bounds for the scaled game and the original game are equal.

\section{Detail on the PDE-based betting policy}\label{section:appendix_betting}

In this section we present detailed analysis of the one-dimensional coin-betting game (Section~\ref{section:1d}). By solving the backward heat equation (\ref{eq:PDE_1d}), we obtain three specific limiting value functions (i.e., potential functions) $\bar V_{1}$, $\bar V_{-1/2}$ and $\bar V_{1/2}$. The performance of their induced coin-betting policies (Algorithm~\ref{algorithm:player}) will be characterized next.

In particular, $\bar V_1$ is a special case where the \emph{exact} minimax relation (Lemma~\ref{lemma:minimax}) can be directly applied. For the general case (e.g., $\bar V_{-1/2}$ and $\bar V_{1/2}$), the minimax relation only approximately holds, therefore we need to use a perturbed analysis introduced in Appendix~\ref{subsection:ito} to \ref{subsection:detail_neg}. Finally, Appendix~\ref{subsection:detail_lower_proof} shows the optimality of the coin-betting policies. 

\subsection{Special case: Policy induced by \texorpdfstring{$\bar V_{1}$}{}}\label{subsection:special}

It is easy to verify that $\bar V_1(t,S)=C(S^2-t)$ satisfies Definition~\ref{def:value}. Therefore, the performance guarantee of the associated player and adversary policies (Algorithm~\ref{algorithm:adversary} and \ref{algorithm:player}) can be stated as a corollary of Lemma~\ref{lemma:minimax}. 

\begin{corollary}\label{corollary:special}
For all $T\in\N_+$,
\begin{enumerate}
\item Against any adversary policy $\bm{a}$, Algorithm~\ref{algorithm:player} constructed from $\bar V_{1}$ guarantees
\begin{equation*}
\wel_T\geq C\spar{\rpar{\sum_{t=1}^Tc_t}^2-T}.
\end{equation*}
\item Against any player policy $\bm{p}$, Algorithm~\ref{algorithm:adversary} constructed from $\bar V_{1}$ guarantees
\begin{equation*}
\wel_T\leq C\spar{\rpar{\sum_{t=1}^Tc_t}^2-T}.
\end{equation*}
\end{enumerate}
\end{corollary}

\subsection{General case: Discrete It\^{o} formula}\label{subsection:ito}

In general, the solution of the backward heat equation (\ref{eq:PDE_1d}) is only an approximation of a value function (for the discrete-time coin-betting game), therefore Lemma~\ref{lemma:duality} cannot be directly applied. Instead, we pursue a perturbed analysis using the \emph{Discrete It\^{o} formula}. Harvey et al. used this technique in the two-expert LEA problem. Here we modify it for our coin-betting problem. 

\ito*

\begin{proof}[Proof of Lemma~\ref{lemma:ito}]
Starting from the LHS of (\ref{eq:ito}), 
\begin{align*}
LHS&=\bar V\rpar{t+1,\sum_{i=1}^{t+1}c_i}-\frac{1}{2}\spar{\bar V\rpar{t+1,\sum_{i=1}^tc_i+1}+\bar V\rpar{t+1,\sum_{i=1}^tc_i-1}}\\
&\hspace{3cm}+\frac{1}{2}\spar{\bar V\rpar{t+1,\sum_{i=1}^tc_i+1}+\bar V\rpar{t+1,\sum_{i=1}^tc_i-1}}-\bar V\rpar{t,\sum_{i=1}^{t}c_i}\\
&=\bar V\rpar{t+1,\sum_{i=1}^{t+1}c_i}-\frac{1}{2}\spar{\bar V\rpar{t+1,\sum_{i=1}^tc_i+1}+\bar V\rpar{t+1,\sum_{i=1}^tc_i-1}}\\
&\hspace{3cm}+\bar\nabla_t\bar V\rpar{t+1,\sum_{i=1}^tc_i}+\frac{1}{2}\bar\nabla_{SS}\bar V\rpar{t+1,\sum_{i=1}^tc_i}.
\end{align*}
The remaining task is to show 
\begin{equation*}
\bar V\rpar{t+1,\sum_{i=1}^{t+1}c_i}-\frac{1}{2}\spar{\bar V\rpar{t+1,\sum_{i=1}^tc_i+1}+\bar V\rpar{t+1,\sum_{i=1}^tc_i-1}}\leq c_{t+1}x_{t+1}.
\end{equation*}
Plugging in the player's bet $x_{t+1}$ (\ref{eq:player_strategy}), it suffices to show that
\begin{equation*}
\bar V\rpar{t+1,\sum_{i=1}^{t+1}c_i}\leq \frac{1+c_{t+1}}{2}\bar V\rpar{t+1,\sum_{i=1}^{t}c_i+1}+\frac{1-c_{t+1}}{2}\bar V\rpar{t+1,\sum_{i=1}^{t}c_i-1},
\end{equation*}
which follows from the convexity of $\bar V$. Equality is achieved when $c_{t+1}\in\{-1,1\}$. 
\end{proof}

\subsection{Preliminary: Properties of \texorpdfstring{$\bar V_{1/2}$}{} and \texorpdfstring{$\bar V_{-1/2}$}{}}\label{subsection:derivative}

In the Discrete It\^{o} formula, quantifying the perturbation term $\Diamond$ is case-dependent. Before doing so, we present some facts on $\bar V_{1/2}$ and $\bar V_{-1/2}$ which will be useful later. Let us first consider $\bar V_{1/2}$. For clarity, we copy the definition here. 
\begin{equation*}
\bar V_{1/2}(t,S)=C\sqrt{t}\spar{2\int_0^{S/\sqrt{2t}}\rpar{\int_0^u\exp(x^2)dx}du-1}.
\end{equation*}

Some calculation yields its derivatives. Let $\nabla_t$ and $\nabla_{tt}$ be the first and second order derivative with respect to $t$. Let $\nabla_{S}$, $\nabla_{SS}$, $\nabla_{SSS}$, $\nabla_{SSSS}$ be the first to fourth order derivative with respect to $S$. 
\begin{multicols}{2}
\begin{equation*}
\nabla_S\bar V_{1/2}(t,S)=\sqrt{2}C\int_0^{S/\sqrt{2t}}\exp(x^2)dx.
\end{equation*}
\begin{equation*}
\nabla_{SS}\bar V_{1/2}(t,S)=\frac{C}{\sqrt{t}}\exp\rpar{\frac{S^2}{2t}}.
\end{equation*}
\begin{equation*}
\nabla_{SSS}\bar V_{1/2}(t,S)=\frac{CS}{t^{3/2}}\exp\rpar{\frac{S^2}{2t}}.
\end{equation*}

\begin{equation*}
\nabla_{SSSS}\bar V_{1/2}(t,S)=\frac{C}{t^{3/2}}\exp\rpar{\frac{S^2}{2t}}\rpar{\frac{S^2}{t}+1}.
\end{equation*}
\begin{equation*}
\nabla_{t}\bar V_{1/2}(t,S)=-\frac{C}{2\sqrt{t}}\exp\rpar{\frac{S^2}{2t}}.
\end{equation*}
\begin{equation*}
\nabla_{tt}\bar V_{1/2}(t,S)=\frac{C}{4t^{3/2}}\exp\rpar{\frac{S^2}{2t}}\rpar{\frac{S^2}{t}+1}.
\end{equation*}
\end{multicols}

Similarly, for $\bar V_{-1/2}$ we have the following. 
\begin{multicols}{2}
\begin{equation*}
\bar V_{-1/2}(t,S)=\frac{C}{\sqrt{t}}\exp\rpar{\frac{S^2}{2t}}.
\end{equation*}
\begin{equation*}
\nabla_{S}\bar V_{-1/2}(t,S)=\frac{CS}{t^{3/2}}\exp\rpar{\frac{S^2}{2t}}.
\end{equation*}
\begin{equation*}
\nabla_{SS}\bar V_{-1/2}(t,S)=\frac{C}{t^{3/2}}\exp\rpar{\frac{S^2}{2t}}\rpar{\frac{S^2}{t}+1}.
\end{equation*}
\begin{equation*}
\nabla_{SSS}\bar V_{-1/2}(t,S)=\frac{C}{t^{3/2}}\exp\rpar{\frac{S^2}{2t}}\rpar{\frac{S^3}{t^2}+\frac{3S}{t}}.
\end{equation*}

\begin{equation*}
\nabla_{SSSS}\bar V_{-1/2}(t,S)=\frac{C}{t^{5/2}}\exp\rpar{\frac{S^2}{2t}}\rpar{\frac{S^4}{t^2}+\frac{6S^2}{t}+3}.
\end{equation*}
\begin{equation*}
\nabla_{t}\bar V_{-1/2}(t,S)=-\frac{C}{2t^{3/2}}\exp\rpar{\frac{S^2}{2t}}\rpar{\frac{S^2}{t}+1}.
\end{equation*}
\begin{equation*}
\nabla_{tt}\bar V_{-1/2}(t,S)=\frac{C}{4t^{5/2}}\exp\rpar{\frac{S^2}{2t}}\rpar{\frac{S^4}{t^2}+\frac{6S^2}{t}+3}.
\end{equation*}
\end{multicols}

Let us compare the betting behavior induced by $\bar V_{1/2}$ and $\bar V_{-1/2}$. The bets in both cases are roughly their derivatives. 

\begin{lemma}\label{lemma:comparison}
For all $t\in\N_+$ and $\abs{S}\leq t-1$, $|\nabla_S\bar V_{1/2}(t,S)|\geq |\nabla_S\bar V_{-1/2}(t,S)|$. 
\end{lemma}

\begin{proof}[Proof of Lemma~\ref{lemma:comparison}]
Due to symmetry, it suffices to consider $0\leq S\leq t-1$. Notice that when $S=0$, $\nabla_S\bar V_{1/2}(t,0)= \nabla_S\bar V_{-1/2}(t,0)$. With $S\leq t-1$, 
\begin{equation*}
\frac{\nabla_{SS}\bar V_{-1/2}(t,S)}{\nabla_{SS}\bar V_{1/2}(t,S)}=\frac{1}{t}\rpar{\frac{S^2}{t}+1}\leq 1-t^{-1}+t^{-2}\leq 1.\qedhere
\end{equation*}
\end{proof}

Moreover, let us specifically compare the derivatives when $|S|=O(\sqrt{t})$: $|\nabla_S\bar V_{1/2}(t,S)|=O(1)$ while $|\nabla_S\bar V_{-1/2}(t,S)|=O(t^{-1})$. Therefore, the betting behavior induced by $\bar V_{1/2}$ cannot be achieved by simply scaling $\bar V_{-1/2}$ (i.e., using a different $C$). 

Finally back to $\bar V_{1/2}$, the integral definition may not be easy to interpret. We can further lower bound it as the following. 

\begin{lemma}\label{lemma:lower_easy}
For all $(t,S)$ such that $\sqrt{2t}\leq|S|\leq t$, 
\begin{equation*}
\bar V_{1/2}(t,S)\geq C\sqrt{t}\cdot\spar{\frac{t}{S^2}\exp\rpar{\frac{S^2}{2t}}-\frac{3}{2}}.
\end{equation*}
\end{lemma}

\begin{proof}[Proof of Lemma~\ref{lemma:lower_easy}]
Based on the definition of $\bar V_{1/2}$, it suffices to show that for all $z\geq 1$, 
\begin{equation*}
f(z)\defeq 2\int_0^{z}\rpar{\int_0^u\exp(x^2)dx}du-\frac{1}{2z^2}\exp\rpar{z^2}\geq -1/2.
\end{equation*}
Notice that $f(1)\geq -1/2$. Taking the derivatives, 
\begin{equation*}
f'(z)=2\int^z_0\exp(x^2)dx+\exp(z^2)\spar{z^{-3}-z^{-1}}.
\end{equation*}
$f'(1)\geq 0$, and
\begin{equation*}
f''(z)=3\exp\rpar{z^2}\rpar{z^{-2}-z^{-4}}\geq 0.\qedhere
\end{equation*}
\end{proof}

\subsection{Policy induced by \texorpdfstring{$\bar V_{1/2}$}{}}\label{subsection:onehalf}

In this subsection we characterize the player policy constructed from $\bar V_{1/2}$. The first step is to quantify the perturbation error ($\Diamond$ in (\ref{eq:ito})). After that, the wealth lower bound (Theorem~\ref{thm:lower_3}) follows from a telescopic sum on the Discrete It\^{o} formula. 

\casethree*

\begin{proof}[Proof of Lemma~\ref{lemma:upper_lower_three}]
Plugging in the definition of discrete derivative, 
\begin{equation}\label{eq:perturbed_PDE}
\bar\nabla_t\bar V_{1/2}(t,S)+\bar\nabla_{SS}\bar V_{1/2}(t,S)/2=\frac{1}{2}\bar V_{1/2}(t,S+1)+\frac{1}{2}\bar V_{1/2}(t,S-1)-\bar V_{1/2}(t-1,S).
\end{equation}

\paragraph{Step 1: upper bound.} For clarity, define a function $f:\R\rightarrow\R$ as
\begin{equation*}
f(z)=2z\int_0^z\exp(x^2)dx-\exp(z^2). 
\end{equation*}
Then, using the definition of $\bar V_{1/2}$, it suffices to show that
\begin{equation*}
f\rpar{-\frac{1}{\sqrt{2}}}+f\rpar{\frac{1}{\sqrt{2}}}\leq 0,
\end{equation*}
and for all $t>1$, 
\begin{equation*}
f\rpar{\frac{S-1}{\sqrt{2t}}}+f\rpar{\frac{S+1}{\sqrt{2t}}}\leq 2\sqrt{1-\frac{1}{t}}f\rpar{\frac{S}{\sqrt{2(t-1)}}}.
\end{equation*}

The first inequality can be easily verified by computing the values of $f(1/\sqrt{2})$ and $f(-1/\sqrt{2})$. As for the second inequality, we use an existing result \cite[Lemma C.4]{harvey2020optimal}: for all $x\in\R$ and $z\in[0,1)$, 
\begin{equation*}
f\rpar{\frac{x-z}{\sqrt{2}}}+f\rpar{\frac{x+z}{\sqrt{2}}}\leq 2\sqrt{1-z^2}f\rpar{\frac{x}{\sqrt{2(1-z^2)}}}.
\end{equation*}
Taking $x=S/\sqrt{t}$ and $z=1/\sqrt{t}$ completes the proof. 

\paragraph{Step 2: lower bound.} From Taylor's theorem, 
\begin{equation*}
\bar V_{1/2}(t,S+1)=\bar V_{1/2}(t,S)+\nabla_S\bar V_{1/2}(t,S)+\frac{1}{2}\nabla_{SS}\bar V_{1/2}(t,S)+\frac{1}{6}\nabla_{SSS}\bar V_{1/2}(t,S)+\frac{1}{24}\nabla_{SSSS}\bar V_{1/2}(t,S+a),
\end{equation*}
\begin{equation*}
\bar V_{1/2}(t,S-1)=\bar V_{1/2}(t,S)-\nabla_S\bar V_{1/2}(t,S)+\frac{1}{2}\nabla_{SS}\bar V_{1/2}(t,S)-\frac{1}{6}\nabla_{SSS}\bar V_{1/2}(t,S)+\frac{1}{24}\nabla_{SSSS}\bar V_{1/2}(t,S-b),
\end{equation*}
\begin{equation*}
\bar V_{1/2}(t-1,S)=\bar V_{1/2}(t,S)-\nabla_t\bar V_{1/2}(t,S)+\frac{1}{2}\nabla_{tt}\bar V_{1/2}(t-c,S),
\end{equation*}
where $a,b,c\in[0,1]$. Plugging these into (\ref{eq:perturbed_PDE}) and using the condition $\nabla_t\bar V_{1/2}=-\nabla_{SS}\bar V_{1/2}/2$ (since $\bar V_{1/2}$ is a solution of the backward heat equation), we have
\begin{equation*}
\bar\nabla_t\bar V_{1/2}(t,S)+\bar\nabla_{SS}\bar V_{1/2}(t,S)/2=\frac{1}{48}\nabla_{SSSS}\bar V_{1/2}(t,S+a)+\frac{1}{48}\nabla_{SSSS}\bar V_{1/2}(t,S-b)-\frac{1}{2}\nabla_{tt}\bar V_{1/2}(t-c,S).
\end{equation*}
From Appendix~\ref{subsection:derivative}, $\nabla_{SSSS}\bar V_{1/2}(t,S)\geq 0$ for all $(t,S)$, and
\begin{equation*}
\nabla_{tt}\bar V_{1/2}(t,S)=\frac{C}{4}t^{-3/2}\exp\rpar{\frac{S^2}{2t}}\rpar{\frac{S^2}{t}+1}.
\end{equation*}
Therefore, 
\begin{align*}
\bar\nabla_t\bar V_{1/2}(t,S)+\bar\nabla_{SS}\bar V_{1/2}(t,S)/2&\geq -\frac{C}{8}\max_{c\in[0,1]}(t-c)^{-3/2}\exp\rpar{\frac{S^2}{2(t-c)}}\rpar{\frac{S^2}{t-c}+1}\\
&=-\frac{C}{8}(t-1)^{-3/2}\exp\rpar{\frac{S^2}{2(t-1)}}\rpar{\frac{S^2}{t-1}+1}.\qedhere
\end{align*}
\end{proof}

Next we prove Theorem~\ref{thm:upper_3}. It shows that the wealth lower bound (Theorem~\ref{thm:lower_3}) faithfully characterizes the performance of the player policy constructed from $\bar V_{1/2}$. 

\upper*

\begin{proof}[Proof of Theorem~\ref{thm:upper_3}]
We first construct the coin sequence. For all $S\in[-T,T]$, there exists an integer $\tilde S$ such that $|\tilde S|\leq T$, $(|\tilde S|+1) \bmod 2=T\bmod 2$ and $|S-\tilde S|\leq 1$. We define the coins using three phases. 
\begin{enumerate}
\item $c_1=S-\tilde S$;
\item For all $1<t\leq T-|\tilde S|$, let $c_t=\sign(c_1)\cdot(-1)^{t-1}$;
\item If $\tilde S\neq 0$, then for all $t$ such that $T-|\tilde S|<t\leq T$, let $c_t=\tilde S/|\tilde S|$.
\end{enumerate}

Based on this coin sequence, there are three immediate observations:
\begin{enumerate}
\item The sum of coins from the second phase is 0, and the sum of coins from the third phase is $\tilde S$; therefore, $\sum_{t=1}^Tc_t=S$. 
\item If $\tau\leq T-|\tilde S|$ then $|\sum_{t=1}^\tau c_t|\leq 1$.
\item If $T-|\tilde S|<\tau\leq T$ then $|\sum_{t=1}^\tau c_t|=|S|-T+\tau$.
\end{enumerate}

Next, we derive the wealth upper bound induced by such a coin sequence and the player policy (Algorithm~\ref{algorithm:player}). Starting from the first round, $x_1=0$, therefore $\wel_1=0$. $\wel_1= \bar V_{1/2}(1,c_1)-\bar V_{1/2}(1,c_1)\leq \bar V_{1/2}(1,c_1)-\bar V_{1/2}(1,0)=\bar V_{1/2}(1,c_1)+C$. Considering the rest of the rounds, there are two cases: ($i$) $|S|\leq \sqrt{T}$; ($ii$) $|S|>\sqrt{T}$.

\paragraph{Case ($i$)} In this case we first show that for all integer $\tau$ in $[1:T]$, $|\sum_{t=1}^\tau c_t|\leq \sqrt{\tau}$. Due to the second observation above, this condition holds for all $\tau\leq T-|\tilde S|$, and we only need to focus on $T-|\tilde S|<\tau\leq T$ (the third phase) where $|\sum_{t=1}^\tau c_t|=|S|-T+\tau\leq\sqrt{T}-T+\tau$; since $T-\sqrt{T}\geq \tau-\sqrt{\tau}$, we further have $|\sum_{t=1}^\tau c_t|\leq\sqrt{\tau}$. Based on this result, telescoping Lemma~\ref{lemma:ito} (notice that equality is achieved) and using Lemma~\ref{lemma:upper_lower_three}, we have
\begin{align*}
\wel_T&\leq \bar V_{1/2}\rpar{T,\sum_{t=1}^Tc_t}+C+\frac{C}{8}\sum_{t=1}^{T-1}t^{-3/2}\exp\rpar{\frac{(\sum_{i=1}^tc_i)^2}{2t}}\rpar{\frac{(\sum_{i=1}^tc_i)^2}{t}+1}\\
&\leq \bar V_{1/2}\rpar{T,\sum_{t=1}^Tc_t}+C+\frac{\sqrt{e}C}{4}\sum_{t=1}^{T-1}t^{-3/2}\leq \bar V\rpar{T,\sum_{t=1}^Tc_t}+\rpar{\frac{3\sqrt{e}}{4}+1}C.
\end{align*}

\paragraph{Case ($ii$)} In this case we show that for all integer $\tau$ in $[1:T]$, $|\sum_{t=1}^\tau c_t|/\sqrt{\tau}\leq |\sum_{t=1}^T c_t|/\sqrt{T}$. Similar to Case ($i$), we consider $\tau\leq T-|\tilde S|$ and $T-|\tilde S|<\tau\leq T$ separately. When $\tau\leq T-|\tilde S|$, we have $|\sum_{t=1}^\tau c_t|/\sqrt{\tau}\leq 1\leq|S|/\sqrt{T}=|\sum_{t=1}^T c_t|/\sqrt{T}$. On the other hand, when $T-|\tilde S|<\tau\leq T$ it suffices to show that
\begin{equation*}
\frac{|S|-T+\tau}{\sqrt{\tau}}\leq \frac{|S|}{\sqrt{T}}.
\end{equation*}
The LHS monotonically increases with respect to $\tau$, and when $\tau=T$ the inequality holds with equality. In summary, the required condition $|\sum_{t=1}^\tau c_t|/\sqrt{\tau}\leq |\sum_{t=1}^T c_t|/\sqrt{T}$ holds for all $\tau\in[1:T]$. 

Based on this result, telescoping Lemma~\ref{lemma:ito} and using Lemma~\ref{lemma:upper_lower_three}, we have
\begin{align*}
\wel_T&\leq \bar V_{1/2}\rpar{T,\sum_{t=1}^Tc_t}+C+\frac{C}{8}\exp\rpar{\frac{(\sum_{i=1}^Tc_i)^2}{2T}}\rpar{\frac{(\sum_{i=1}^Tc_i)^2}{T}+1}\sum_{t=1}^{T-1}t^{-3/2}\\
&\leq \bar V_{1/2}\rpar{T,\sum_{t=1}^Tc_t}+C+\frac{3C}{8}\exp\rpar{\frac{(\sum_{i=1}^Tc_i)^2}{2T}}\rpar{\frac{(\sum_{i=1}^Tc_i)^2}{T}+1}.
\end{align*}

Combining Case ($i$) and Case ($ii$) completes the proof. 
\end{proof}

Theorem~\ref{thm:upper_3} has a special form: it fixes both the player policy (Algorithm~\ref{algorithm:player}) and the adversary policy (Algorithm~\ref{algorithm:adversary}), and then bounds the wealth induced by both of them. Results of this form are seldom studied in conventional online learning settings. The reason is that, the performance metric for those settings is usually the uniform regret (a real number), therefore the gap between policy-independent upper and lower bounds is relatively easy to describe. In contrast, we care about the \emph{trade-offs} on our performance metric, so our upper and lower bounds are both expressed as functions; the characterization of their gap is much richer. We present our player-policy-independent wealth upper bound as Theorem~\ref{thm:optimality}. It is related, but incomparable to Theorem~\ref{thm:upper_3} stated above. 

\subsection{Policy induced by \texorpdfstring{$\bar V_{-1/2}$}{}}\label{subsection:detail_neg}

Analogous to the previous subsection, we now characterize the performance of the player policy induced by $\bar V_{-1/2}$. The first step is to quantify the perturbation error $\Diamond$. 

\begin{lemma}\label{lemma:upper_lower_two}
For all $t\in\N_+$ and $S\in[1-t,t-1]$, $\bar V_{-1/2}$ with any parameter $C>0$ satisfies the following conditions. 
\begin{enumerate}
\item If $t=1$, then
\begin{equation*}
\bar\nabla_t\bar V_{-1/2}(t,S)+\bar\nabla_{SS}\bar V_{-1/2}(t,S)/2=C\sqrt{e}.
\end{equation*}
\item If $t>1$, then
\begin{equation*}
-\frac{C}{8}(t-1)^{-5/2}\exp\rpar{\frac{S^2}{2(t-1)}}\rpar{\frac{S^4}{(t-1)^2}+\frac{6S^2}{t-1}+3}\leq \bar\nabla_t\bar V_{-1/2}(t,S)+\bar\nabla_{SS}\bar V_{-1/2}(t,S)/2\leq 0.
\end{equation*}
\end{enumerate}
\end{lemma}

\begin{proof}[Proof of Lemma~\ref{lemma:upper_lower_two}]
The case of $t=1$ can be easily verified. We will prove the second case next. Plugging in the definition, we have
\begin{align}
\bar\nabla_t\bar V_{-1/2}(t,S)+\bar\nabla_{SS}\bar V_{-1/2}(t,S)/2&=\frac{1}{2}\bar V_{-1/2}(t,S+1)+\frac{1}{2}\bar V_{-1/2}(t,S-1)-\bar V_{-1/2}(t-1,S)\nonumber\\
&=\frac{C}{2\sqrt{t}}\exp\rpar{\frac{(S-1)^2}{2t}}+\frac{C}{2\sqrt{t}}\exp\rpar{\frac{(S+1)^2}{2t}}-\frac{C}{\sqrt{t-1}}\exp\rpar{\frac{S^2}{2(t-1)}}.\label{eq:perturbation_PDE_alt}
\end{align}

First, let us consider the upper bound. Since $\exp(-t^{-1})\geq 1-t^{-1}$, we have
\begin{equation*}
\exp\rpar{\frac{1}{2t}}\leq \sqrt{\frac{t}{t-1}}. 
\end{equation*}
Therefore, 
\begin{align*}
\exp\rpar{\frac{(S-1)^2}{2t}}+\exp\rpar{\frac{(S+1)^2}{2t}}&=\exp\rpar{\frac{S^2+1}{2t}}\spar{\exp\rpar{-\frac{S}{t}}+\exp\rpar{\frac{S}{t}}}\\
&\leq \sqrt{\frac{t}{t-1}}\exp\rpar{\frac{S^2}{2t}}\spar{\exp\rpar{-\frac{S}{t}}+\exp\rpar{\frac{S}{t}}}\\
&\leq 2\sqrt{\frac{t}{t-1}}\exp\rpar{\frac{S^2}{2t}}\exp\rpar{\frac{S^2}{2t^2}},
\end{align*}
where the last inequality is due to the classical result $\cosh(x)\leq \exp(x^2/2)$. Back to (\ref{eq:perturbation_PDE_alt}), 
\begin{equation*}
\bar\nabla_t\bar V_{-1/2}(t,S)+\bar\nabla_{SS}\bar V_{-1/2}(t,S)/2\leq C\sqrt{\frac{1}{t-1}}\spar{\exp\rpar{\frac{S^2}{2t}}\exp\rpar{\frac{S^2}{2t^2}}-\exp\rpar{\frac{S^2}{2(t-1)}}},
\end{equation*}
and it is straightforward to verify that $\rhs\leq 0$. 

Next, we consider the lower bound. Similar to the proof of Lemma~\ref{lemma:upper_lower_three}, using the derivatives from Appendix~\ref{subsection:derivative}, 
\begin{align*}
\bar\nabla_t\bar V_{-1/2}(t,S)+\bar\nabla_{SS}\bar V_{-1/2}(t,S)/2&\geq -\frac{1}{2}\nabla_{tt}\bar V_{-1/2}(t-1,S)\\
&=-\frac{C}{8}(t-1)^{-5/2}\exp\rpar{\frac{S^2}{2(t-1)}}\rpar{\frac{S^4}{(t-1)^2}+\frac{6S^2}{t-1}+3}.\qedhere
\end{align*}
\end{proof}

Similar to the wealth lower bound induced by $\bar V_{1/2}$ (Theorem~\ref{thm:lower_3}), we can plug the above lemma into the Discrete It\^{o} formula (Lemma~\ref{lemma:ito}) and obtain the following theorem via a telescopic sum. The proof is omitted. Essentially, a wealth lower bound of this form recovers the result from \cite{orabona2016coin}. However, our analysis is based on a general framework without budget constraints, therefore does not involve any \emph{betting fractions}. 

\begin{theorem}
For all $T\in\N_+$, Algorithm~\ref{algorithm:player} constructed from $\bar V_{-1/2}$ guarantees a wealth lower bound
\begin{equation*}
\wel_T\geq \bar V_{-1/2}\rpar{T,\sum_{t=1}^Tc_t}-C\sqrt{e},
\end{equation*}
against any adversary policy $\bm{a}$. 
\end{theorem}

In addition, analogous to Theorem~\ref{thm:upper_3}, we can also state a wealth upper bound based on $\bar V_{-1/2}$. The proof uses a similar strategy, therefore is omitted. 

\begin{theorem}
For all $T\in\N_+$ and $S\in[-T,T]$, we can construct $c_1\in\calC$ and $c_2,\ldots,c_T\in\{-1,1\}$ such that
\begin{enumerate}
\item $\sum_{t=1}^Tc_t=S$;
\item If the player applies Algorithm~\ref{algorithm:player} constructed from $\bar V_{-1/2}$ (with parameter $C$) and the adversary plays the aforementioned coin sequence $c_{1:T}$, then
\begin{equation*}
\wel_T\leq \bar V_{-1/2}\rpar{T,S}+\frac{5C}{24}\exp\rpar{\frac{S^2}{2T}}\rpar{\frac{S^4}{T^2}+\frac{6S^2}{T}+3}+2C.
\end{equation*}
\end{enumerate}
\end{theorem}

\subsection{The optimality of betting policies}\label{subsection:detail_lower_proof}

Finally, we prove the player-policy-independent wealth upper bounds (Theorem~\ref{thm:optimality} and its analogue based on $\bar V_{-1/2}$). The first step is to prove a sharp lower bound for the tail probability of one-dimensional symmetric random walk, based on a normal approximation. 

\begin{lemma}\label{lemma:tail}
For all $T\in\N_+$, let $z_1,\ldots,z_T$ be i.i.d. Rademacher random variables. Then for any $k>0$, 
\begin{equation*}
\P\spar{\abs{\sum_{t=1}^Tz_t}\geq k}\geq \sqrt{\frac{2}{\pi}}\frac{k\sqrt{T}}{k^2+T}\exp\rpar{-\frac{k^2}{2T}}-\frac{1}{\sqrt{T}}.
\end{equation*}
\end{lemma}

\begin{proof}[Proof of Lemma~\ref{lemma:tail}]
Due to Central Limit Theorem, the random variable $(\sum_{t=1}^Tz_t)/\sqrt{T}$ converges in distribution to standard normal $N(0,1)$. Concretely, the nonasymptotic convergence rate can be characterized via the Berry-Esseen Theorem \cite{korolev2012improvement}: Let $F_T(x)$ be the CDF of $(\sum_{t=1}^Tz_t)/\sqrt{T}$ and $\Phi(x)$ be the standard normal CDF, then, 
\begin{equation*}
\sup_{x\in\R}\abs{F_T(x)-\Phi(x)}\leq \frac{1}{2\sqrt{T}}.
\end{equation*}
For the tail probability of standard normal distribution, there is a standard lower bound \cite{gordon1941values} through the Mills ratio, which can be verified via a derivative argument: For all $x>0$, 
\begin{equation*}
1-\Phi(x)\geq \frac{1}{\sqrt{2\pi}}\frac{1}{x+x^{-1}}\exp\rpar{-\frac{x^2}{2}}.
\end{equation*}
Therefore, 
\begin{equation*}
\P\spar{\abs{\sum_{t=1}^Tz_t}\geq k}=2\cdot\spar{1-F_T(k/\sqrt{T})}\geq 2\cdot\spar{1-\Phi(k/\sqrt{T})-\frac{1}{2\sqrt{T}}}\geq \sqrt{\frac{2}{\pi}}\frac{k\sqrt{T}}{k^2+T}\exp\rpar{-\frac{k^2}{2T}}-\frac{1}{\sqrt{T}}. \qedhere
\end{equation*}
\end{proof}

Compared to similar tail lower bounds from existing works on unconstrained OLO \cite{mcmahan2012no,orabona2013dimension}, Lemma~\ref{lemma:tail} has the tight exponent (1/2) in the exponential function. This allows us to justify the optimality of our PDE-based coin-betting policy (Algorithm~\ref{algorithm:player} constructed from $\bar V_{1/2}$), and eventually the converted unconstrained OLO algorithm. 

\optimality*

\begin{proof}[Proof of Theorem~\ref{thm:optimality}]
Let us first generalize the unconstrained coin-betting game to allow random adversary on the coin space $\{-1,1\}$. That is, based on past player bets $x_1,\ldots,x_t$, the adversary decides a distribution on $\{-1,1\}$ and samples $c_t$ from this distribution. 

Now, consider the setting where the player applies any policy $\bm{p}$ that guarantees $\wel_T\geq -C\sqrt{T}$, and the adversary picks coin outcomes according to a Rademacher distribution: regardless of $x_1,\ldots,x_t$, the coin $c_t$ equals $-1$ and $1$ with probability $1/2$ respectively. Then for all $T\in\N_+$, let $k=\sqrt{2T\log\lambda}$.
\begin{align*}
0&=\E\spar{\sum_{t=1}^Tc_tx_t}\\
&=\E\spar{\sum_{t=1}^Tc_tx_t\bigg |\abs{\sum_{t=1}^Tc_t}\geq k}\P\spar{\abs{\sum_{t=1}^Tc_t}\geq k}+\E\spar{\sum_{t=1}^Tc_tx_t\bigg |\abs{\sum_{t=1}^Tc_t}<k}\P\spar{\abs{\sum_{t=1}^Tc_t}< k}\\
&\geq \E\spar{\sum_{t=1}^Tc_tx_t\bigg |\abs{\sum_{t=1}^Tc_t}\geq k}\P\spar{\abs{\sum_{t=1}^Tc_t}\geq k}-C\sqrt{T}.
\end{align*}

Applying Lemma~\ref{lemma:tail}, using $\lambda\geq \exp[(\sqrt{2}+1)/2]$ and $T\geq8\pi\lambda^2\log\lambda$,
\begin{align*}
\P\spar{\abs{\sum_{t=1}^Tc_t}\geq k}&\geq \sqrt{\frac{2}{\pi}}\frac{\sqrt{2\log\lambda}}{1+2\log\lambda}\lambda^{-1}-\frac{1}{\sqrt{T}}\\
&\geq \frac{1}{\sqrt{2\pi\log\lambda}}\lambda^{-1}-\frac{1}{\sqrt{T}}\geq \frac{1}{2\sqrt{2\pi\log\lambda}}\lambda^{-1}.
\end{align*}
\begin{equation*}
\E\spar{\sum_{t=1}^Tc_tx_t\bigg |\abs{\sum_{t=1}^Tc_t}\geq k}\leq \frac{C\sqrt{T}}{\P\spar{\abs{\sum_{t=1}^Tc_t}\geq k}}\leq 2\sqrt{2\pi}\lambda\sqrt{\log\lambda}\cdot C\sqrt{T}.
\end{equation*}
Therefore, for any player policy $\bm{p}$ there exists an adversary policy $\bm{a}$ which induces $|\sum_{t=1}^Tc_t|\geq \sqrt{2T\log\lambda}$ and $\wel_T\leq 2\sqrt{2\pi}\lambda\sqrt{\log\lambda}\cdot C\sqrt{T}$.
\end{proof}

A similar result can be stated with respect to $\bar V_{-1/2}$, using a different ``barrier'' $k$ that depends on $T$. This introduces a specific technical issue: when we use Lemma~\ref{lemma:tail}, the normal approximation error ($1/\sqrt{T}$) is comparable in magnitude to the Gaussian tail bound (the first term in Lemma~\ref{lemma:tail}) which we care about. Therefore, the following theorem has a slightly weaker form than Theorem~\ref{thm:optimality}. 

\begin{theorem}\label{thm:optimality_alt}
For all $\lambda\geq \exp[(\sqrt{2}+1)/2]$, there exists $T_0\in\N_+$ (depending on $\lambda$) such that for all $T\geq T_0$ and any player policy $\bm{p}$ which guarantees $\wel_T\geq -C\sqrt{e}$ (e.g., Algorithm~\ref{algorithm:player} constructed from $\bar V_{-1/2}$), there exists an adversary policy $\bm{a}$ with the following property. In the coin-betting game induced by the policy pair $(\bm{p},\bm{a})$, 
\begin{enumerate}
\item $|\sum_{t=1}^Tc_t|\geq\sqrt{2T\log(\lambda \sqrt{T}/\log T)}$;
\item $\wel_T\leq 2\sqrt{2\pi e}\lambda(\log T)^{-1}\sqrt{\log(\lambda \sqrt{T}/\log T)}\cdot C\sqrt{T}$.
\end{enumerate}
\end{theorem}

\begin{proof}[Proof of Theorem~\ref{thm:optimality_alt}]
We follow a similar analysis as Theorem~\ref{thm:optimality} but use a different barrier. Let us only consider $T>1$ and let $k=\sqrt{2T\log(\lambda \sqrt{T}/\log T)}$. Using the Rademacher random adversary, 
\begin{equation*}
\E\spar{\sum_{t=1}^Tc_tx_t\bigg |\abs{\sum_{t=1}^Tc_t}\geq k}\P\spar{\abs{\sum_{t=1}^Tc_t}\geq k}\leq C.
\end{equation*}
Using $\lambda\geq \exp[(\sqrt{2}+1)/2]$, we have $1\leq 2(\sqrt{2}-1)\log(\lambda)\leq 2(\sqrt{2}-1)\log(\lambda \sqrt{T}/\log T)$. Therefore, 
\begin{align*}
\P\spar{\abs{\sum_{t=1}^Tc_t}\geq k}&\geq \sqrt{\frac{2}{\pi}}\frac{\sqrt{2\log(\lambda \sqrt{T}/\log T)}}{1+2\log(\lambda \sqrt{T}/\log T)}\frac{\log T}{\lambda\sqrt{T}}-\frac{1}{\sqrt{T}}\\
&\geq \frac{\log T}{\lambda\sqrt{2\pi\log(\lambda \sqrt{T}/\log T)}}T^{-1/2}-T^{-1/2}.
\end{align*}
Since the first term decays slower (with respect to $T$) than the second term $T^{-1/2}$, there exists $T_0$ depending on $\lambda$ such that for all $T\geq T_0$, 
\begin{equation*}
\P\spar{\abs{\sum_{t=1}^Tc_t}\geq k}\geq \frac{\log T}{2\lambda\sqrt{2\pi\log(\lambda \sqrt{T}/\log T)}}T^{-1/2},
\end{equation*}
\begin{equation*}
\E\spar{\sum_{t=1}^Tc_tx_t\bigg |\abs{\sum_{t=1}^Tc_t}\geq k}\leq \frac{C\sqrt{e}}{\P\spar{\abs{\sum_{t=1}^Tc_t}\geq k}}\leq 2\sqrt{2\pi e}\lambda(\log T)^{-1}\sqrt{\log(\lambda \sqrt{T}/\log T)}\cdot C\sqrt{T}.\qedhere
\end{equation*}
\end{proof}

\section{Detail on unconstrained OLO}\label{section:appendix_olo}

In this section we present detailed analysis on unconstrained OLO. First, using the conversion from coin-betting to OLO (Algorithm~\ref{algorithm:conversion}), our coin-betting policy (Algorithm~\ref{algorithm:player}) can be directly converted into a one-dimensional unconstrained OLO algorithm. For clarity, we restate its pseudo-code as Algorithm~\ref{algorithm:combined_1d}.

\begin{algorithm*}[ht]
\caption{PDE-based one-dimensional unconstrained OLO algorithm.\label{algorithm:combined_1d}}
\begin{algorithmic}[1]
\REQUIRE A one-dimensional limiting value function $\bar V$ which satisfies (\ref{eq:PDE_1d}). 
\FOR{$t=1,2,\ldots$}
\STATE Predict
\begin{equation*}
x_{t}=\frac{1}{2}\spar{\bar V\rpar{t,-\sum_{i=1}^{t-1}g_i+1}-\bar V\rpar{t,-\sum_{i=1}^{t-1}g_i-1}}.
\end{equation*}
\STATE Observe the loss gradient $g_t$ and store it. 
\ENDFOR
\end{algorithmic}
\end{algorithm*}

For general $d$-dimensional problems, we rely on a classical reduction \cite{cutkosky2018black} to the one-dimensional problem. Its pseudo-code is Algorithm~\ref{algorithm:extension}, and the associated performance guarantee is Lemma~\ref{lemma:extension} whose proof follows from \cite[Theorem 2]{cutkosky2018black} and the standard regret bound of OGD (e.g., \citep[Section~4.2.1]{orabona2019modern}). Our final product is Algorithm~\ref{algorithm:combined} presented in the main paper. 

\begin{algorithm*}[ht]
\caption{Reducing unconstrained OLO from \texorpdfstring{$\R^d$}{} to \texorpdfstring{$\R$}{}.\label{algorithm:extension}}
\begin{algorithmic}[1]
\REQUIRE A one-dimensional unconstrained OLO algorithm $\A_{1d}$.
\STATE Define $\A_B$ as the standard Online Gradient Descent (OGD) on $\ball^d$ with learning rate $\eta_t=1/\sqrt{t}$, initialized at the origin. 
\FOR{$t=1,2,\ldots$}
\STATE Obtain predictions $y_t\in\R$ from $\A_{1d}$ and $z_t\in\R^d$ from $\A_B$. 
\STATE Predict $x_t=y_tz_t\in\R^d$, observe $g_t\in\R^d$.
\STATE Return $\langle g_t,z_t\rangle$ and $g_t$ as the $t$-th loss gradient to $\A_{r}$ and $\A_B$, respectively.
\ENDFOR
\end{algorithmic}
\end{algorithm*}

\begin{lemma}[Theorem 2 of \cite{cutkosky2018black}, adapted]\label{lemma:extension}
For all $T\in\N_+$, if $\A_{1d}$ guarantees regret bound $\reg_T(u)\leq R_T(u)$ for all $u\in\R$, then Algorithm~\ref{algorithm:extension} guarantees $\reg_T(u)\leq R_T(\norms{u})+\norms{u}\sqrt{2T}$ for all $u\in\R^d$.
\end{lemma}

\subsection{OLO algorithm induced by \texorpdfstring{$\bar V_{1/2}$}{}}\label{subsection:OLO_detail}

Next, we consider Algorithm~\ref{algorithm:combined} and prove the regret upper bound induced by $\bar V_{1/2}$. 

\oloupper*

\begin{proof}[Proof of Theorem~\ref{thm:olo_upper}]
The proof follows from the combination of Lemma~\ref{lemma:duality}, Theorem~\ref{thm:lower_3} and Lemma~\ref{lemma:extension}. Specifically, let us first guarantee the performance of the $y_t$ sequence. For clarity, given any $T$, define a one-dimensional function $f_T$ as $f_T(S)=\bar V_{1/2}(T,S)$. Combining Lemma~\ref{lemma:duality} and Theorem~\ref{thm:lower_3}, for any $T\in\N_+$ and $w\in\R$ we have
\begin{equation*}
\sum_{t=1}^T\inner{g_t}{z_t}y_t-\sum_{t=1}^T\inner{g_t}{z_t}w\leq f^*_T(w).
\end{equation*}
Then, due to Lemma~\ref{lemma:extension}, for all $T\in\N_+$ and $u\in\R^d$ Algorithm~\ref{algorithm:combined} guarantees
\begin{equation*}
\reg_T(u)\leq f^*_T(\norm{u})+\norm{u}\sqrt{2T}.
\end{equation*}

The remaining task is to bound the Fenchel conjugate $f^*_T$. For all $w\in\R$, 
\begin{equation*}
f^*_T(w)=\sup_{S\in\R}Sw-f_T(S).
\end{equation*}
Let $S^*$ be the maximizing argument. Without loss of generality (due to symmetry), assume $w\geq 0$ and therefore $S^*\geq 0$. We have
\begin{equation*}
w=\nabla f_T(S^*)=\sqrt{2}C\int_0^{S^*/\sqrt{2T}}\exp(z^2)dz.
\end{equation*}

For any $x\geq 0$, consider the function $f(x)=\int_0^x\exp(z^2)dz$. It is lower bounded by $g(x)=\exp(x^2-x)-1$, as $f(0)=g(0)$, and
\begin{equation*}
f'(x)=\exp(x^2)\geq \exp(x^2-x)(2x-1)=g'(x),
\end{equation*}
due to the inequality $\exp(x)\geq 2x-1$. Therefore, 
\begin{equation*}
\frac{w}{\sqrt{2}C}=\int_0^{S^*/\sqrt{2T}}\exp(z^2)dz\geq \exp\spar{\rpar{\frac{S^*}{\sqrt{2T}}-\frac{1}{2}}^2-\frac{1}{4}}-1,
\end{equation*}
\begin{equation*}
S^*\leq \sqrt{2T}\spar{\sqrt{\frac{1}{4}+\log\rpar{1+\frac{w}{\sqrt{2}C}}}+\frac{1}{2}}.
\end{equation*}

Now consider $f^*_T(w)$. Since $f_T(S^*)\geq -C\sqrt{T}$ and $\sqrt{x+(1/4)}\leq \sqrt{x}+(1/2)$, 
\begin{equation*}
f^*_T(w)=S^*w-f_T(S^*)\leq S^*w+C\sqrt{T}\leq C\sqrt{T}+w\sqrt{2T}\spar{\sqrt{\log\rpar{1+\frac{w}{\sqrt{2}C}}}+1}.
\end{equation*}
Combining everything completes the proof. 
\end{proof}

Converting Theorem~\ref{thm:optimality} to unconstrained OLO, we also have a regret lower bound with respect to all algorithms (satisfying a condition). 

\begin{theorem}\label{thm:olo_lower}
For all $\eta\in(0,1)$, $U\geq 12\eta^{-1}C$, $T\geq 2\eta^2 U^2C^{-2}\log(\eta U C^{-1})$ and any unconstrained OLO algorithm $\A$ that guarantees $\reg_T(0)\leq C\sqrt{T}$ (e.g., Algorithm~\ref{algorithm:combined} constructed from $\bar V_{1/2}$), there exists an adversary and a comparator $u\in\R^d$ such that $\norms{u}=U$ and
\begin{equation*}
\reg_T(u)\geq(1-\eta)\norm{u}\sqrt{2T\log\frac{\eta\norm{u}}{2\sqrt{\pi}C}}.
\end{equation*}
\end{theorem}

\begin{proof}[Proof of Theorem~\ref{thm:olo_lower}]
We start by proving the regret lower bound for one-dimensional unconstrained OLO. Extension to the general $d$-dimensional problem will be considered later. 

For the one-dimensional problem, we first invoke a particular version of Theorem~\ref{thm:optimality} on unconstrained coin-betting. Specifically, for any constants $\eta\in(0,1)$ and $u\in\R/\{0\}$ we define $\lambda$ in Theorem~\ref{thm:optimality} as
\begin{equation*}
\lambda=\frac{\eta\abs{u}}{2\sqrt{\pi}C}.
\end{equation*}
For convenience of notation we also define
\begin{equation*}
T_0=\frac{2\eta^2\abs{u}^2}{C^2}\log\rpar{\frac{\eta\abs{u}}{2\sqrt{\pi}C}}.
\end{equation*}
Then, Theorem~\ref{thm:optimality} yields the following result: For all $\eta\in(0,1)$, $\abs{u}\geq 2\sqrt{\pi}\exp[(\sqrt{2}+1)/2]\eta^{-1}C$, $T\geq T_0$ and any coin-betting player policy $\bm{p}$ that guarantees $\wel_T\geq -C\sqrt{T}$, there exists a coin-betting adversary policy $\bm{a}$ such that in the game induced by $(\bm{p},\bm{a})$, 
\begin{enumerate}
\item $|\sum_{t=1}^Tc_t|\geq\sqrt{2T\log\lambda}$;
\item $\wel_T\leq \eta\abs{u}\sqrt{2T\log\lambda}$.
\end{enumerate}

Using Algorithm~\ref{algorithm:conversion}, we can equivalently convert OLO to coin-betting by letting $c_t=-g_t$. Then, the above result immediately translates to the following statement on one-dimensional unconstrained OLO: For all $\eta\in(0,1)$, $\abs{u}\geq 2\sqrt{\pi}\exp[(\sqrt{2}+1)/2]\eta^{-1}C$, $T\geq T_0$ and any unconstrained OLO algorithm $\A$ that guarantees the cumulative loss bound $\sum_{t=1}^Tg_tx_t\leq C\sqrt{T}$, there exists an OLO adversary such that in the induced game, 
\begin{enumerate}
\item $|\sum_{t=1}^Tg_t|\geq\sqrt{2T\log\lambda}$;
\item $-\sum_{t=1}^Tg_tx_t\leq \eta\abs{u}\sqrt{2T\log\lambda}$.
\end{enumerate}

Let us consider the regret of $\A$ in this setting with respect to comparators $u$ and $-u$. Using the above result, 
\begin{align*}
\max\left\{\reg_T(u),\reg_T(-u)\right\}&=\sum_{t=1}^Tg_tx_t+\max\left\{-\sum_{t=1}^Tg_tu,\sum_{t=1}^Tg_tu\right\}\\
&= \sum_{t=1}^Tg_tx_t+\abs{\sum_{t=1}^Tg_t}\abs{u}\\
&\geq (1-\eta)\abs{u}\sqrt{2T\log\lambda}\\
&=(1-\eta)\abs{u}\sqrt{2T\log\frac{\eta\abs{u}}{2\sqrt{\pi}C}}.
\end{align*}
Thus we have proved the desirable result when $d=1$. 

Extending this result to $d$-dimension follows from a standard technique: consider adversaries whose loss vectors $g_t$ are only nonzero in one coordinate. Let $g_t=[g_{t,1},\ldots,g_{t,d}]$, and assume $g_{t,2}=\ldots=g_{t,d}=0$. Then, for any player who plays against this adversary and competes against $u=[u_{1},0,\ldots,0]$, 
\begin{equation*}
\reg_T(u)=\sum_{t=1}^T\inner{g_t}{x_t}-\sum_{t=1}^T\inner{g_t}{u}=\sum_{t=1}^Tg_{t,1}x_{t,1}-\sum_{t=1}^Tg_{t,1}u_1,
\end{equation*}
$\norms{u}=|u_1|$, and the cumulative loss satisfies $\sum_{t=1}^T\inner{g_t}{x_t}=\sum_{t=1}^Tg_{t,1}x_{t,1}$. Therefore, any $d$-dimensional algorithm that guarantees $\reg_T(0)\leq C\sqrt{T}$ is translated into a one-dimensional algorithm with the same guarantee, and our one-dimensional regret lower bound can be applied. 
\end{proof}

Finally, for a clear comparison of the upper and lower bounds, we have the following theorem presented in the main paper. 

\interpret*

\begin{proof}[Proof of Theorem~\ref{thm:interpret}]
Let us first consider the upper bound. Plugging in Theorem~\ref{thm:olo_upper},
\begin{align*}
&\limsup_{U\rightarrow\infty}\limsup_{T\rightarrow\infty}\sup_{\norms{u}=U, adv}\frac{\reg_T^{\A_{1/2},adv}(u)}{\norms{u}\sqrt{T\log\norms{u}}}\\
\leq~& \limsup_{U\rightarrow\infty}\limsup_{T\rightarrow\infty}\sup_{\norms{u}=U, adv}\rpar{\frac{C+2\sqrt{2}\norms{u}}{\norms{u}\sqrt{\log\norms{u}}}+\sqrt{2\log\rpar{1+\frac{\norm{u}}{\sqrt{2}C}}\log^{-1}\norms{u}}}\\
\leq~&\lim_{U\rightarrow\infty}\frac{C+2\sqrt{2}U}{U\sqrt{\log U}}+\lim_{U\rightarrow\infty}\sqrt{2\log\rpar{1+\frac{U}{\sqrt{2}C}}\log^{-1}U}=\sqrt{2}
\end{align*}

As for the lower bound, we use Theorem~\ref{thm:olo_lower}. We first fix any $C$ and any $\A$ satisfying the condition in the theorem to be proved. For all $\eta\in(0,1)$, with $U\geq 12\eta^{-1}C$ and $T\geq 2\eta^2 U^2C^{-2}\log(\eta U C^{-1})$, 
\begin{align*}
\sup_{\norms{u}=U, adv}\frac{\reg_T^{\A,adv}(u)}{\norms{u}\sqrt{T\log\norms{u}}}&\geq (1-\eta)\sqrt{2\log\frac{\eta U}{2\sqrt{\pi}C}\log^{-1}U}\\
&=(1-\eta)\sqrt{2\rpar{1+\frac{\log\eta}{\log U}-\frac{\log(2\sqrt{\pi}C)}{\log U}}}.
\end{align*}
Taking $\liminf$ on both sides, for all $\eta\in(0,1)$, 
\begin{equation*}
\liminf_{U\rightarrow\infty}\liminf_{T\rightarrow\infty}\sup_{\norms{u}=U, adv}\frac{\reg_T^{\A,adv}(u)}{\norms{u}\sqrt{T\log\norms{u}}}\geq \sqrt{2}(1-\eta). 
\end{equation*}
Rewriting this statement, we have: for all $\eps\geq 0$ and $\eta\in(0,1)$, there exists $U_0$ depending on $\eps$ and $\eta$ such that for all $U\geq U_0$, 
\begin{equation*}
\liminf_{T\rightarrow\infty}\sup_{\norms{u}=U, adv}\frac{\reg_T^{\A,adv}(u)}{\norms{u}\sqrt{T\log\norms{u}}}\geq \sqrt{2}-\sqrt{2}\eta-\eps. 
\end{equation*}
Finally, using the definition of $\liminf$ completes the proof. 
\end{proof}

\subsection{OLO algorithm induced by \texorpdfstring{$\bar V_{-1/2}$}{}}\label{subsection:OLO_detail_alt}

Similar to the previous subsection, we can also convert our results on $\bar V_{-1/2}$ (Appendix~\ref{subsection:detail_neg}) to the OLO setting. Since $\bar V_{1/2}$ recovers the existing coin-betting potentials, the converted regret upper bound recovers the classical bound (\ref{eq:regret_existing}). See also \cite[Corollary~5]{orabona2016coin}.

\begin{theorem}\label{thm:olo_upper_alt}
For all $T\in\N_+$ and $u\in\R^d$, against any adversary, Algorithm~\ref{algorithm:combined} constructed from $\bar V_{-1/2}$ guarantees
\begin{equation*}
\reg_T(u)\leq C\sqrt{e}+\norm{u}\sqrt{2T}\spar{\sqrt{\log\rpar{1+\frac{\norm{u}T}{C}}}+1}.
\end{equation*}
\end{theorem}

\begin{proof}[Proof of Theorem~\ref{thm:olo_upper_alt}]
Following the proof of Theorem~\ref{thm:olo_upper}, the only difference here is to upper bound the Fenchel conjugate of $f_T(S)=\bar V_{-1/2}(T,S)$. We use the existing result \cite[Lemma~18]{orabona2016coin}: for any function $f(x)=\beta\exp(x^2/(2\alpha))$ with $\alpha,\beta>0$, 
\begin{equation*}
f^*(y)\leq \abs{y}\sqrt{\alpha\log\rpar{1+\frac{\alpha y^2}{\beta^2}}}-\beta.
\end{equation*}
Therefore, 
\begin{equation*}
f^*_T(\norm{u})\leq C\sqrt{e}+ \norm{u}\sqrt{T\log\rpar{1+\frac{\norms{u}^2T^2}{C^2}}}\leq C\sqrt{e}+ \norm{u}\sqrt{2T\log\rpar{1+\frac{\norms{u}T}{C}}}.
\end{equation*}
The rest of the proof is similar to the proof of Theorem~\ref{thm:olo_upper}.
\end{proof}

Next we present the regret lower bound induced by $\bar V_{-1/2}$, parallel to Theorem~\ref{thm:olo_lower}. 

\begin{theorem}\label{thm:olo_lower_alt}
For all $\eta\in(0,1)$ and $U\geq 12\eta^{-1}C$, there exists $T_0\in\N_+$ (depending on $\eta$, $U$ and $C$) such that the following statement holds. For all $T\geq T_0$ and any unconstrained OLO algorithm $\A$ that guarantees $\reg_T(0)\leq C\sqrt{e}$ (e.g., Algorithm~\ref{algorithm:combined} constructed from $\bar V_{-1/2}$), there exists an adversary and a comparator $u\in\R^d$ such that $\norms{u}=U$ and
\begin{equation*}
\reg_T(u)\geq\spar{1-\eta\rpar{\log T}^{-1}}\norm{u}\sqrt{2T\log\frac{\eta\norm{u}\sqrt{T}}{2\sqrt{\pi e}C\log T}}.
\end{equation*}
\end{theorem}

The proof is similar to Theorem~\ref{thm:olo_lower} therefore omitted. In particular, we plug a slightly different choice of $\lambda$ into Theorem~\ref{thm:optimality_alt}: $\lambda=\eta\abs{u}/(2\sqrt{\pi e}C)$. 

To our knowledge, existing lower bounds for unconstrained OLO (\cite[Theorem~7]{mcmahan2012no}, \cite[Theorem~2]{orabona2013dimension}, \cite[Theorem~5.12]{orabona2019modern}) all focused on the ``budget constraint'' $\reg_T(0)\leq \textrm{constant}$. Such a setting is different from Theorem~\ref{thm:olo_lower} presented in the main paper, but same as Theorem~\ref{thm:olo_lower_alt} above. Compared to those results, Theorem~\ref{thm:olo_lower_alt} improves the leading constant: previously the best known constant (on the leading term $\norms{u}\sqrt{T\log(\norms{u}\sqrt{T})}$) was $1/\sqrt{\log 2}\approx 1.201$ \cite{orabona2013dimension}, while we improve it to $\sqrt{2}\approx 1.414$. This is due to the use of a tighter tail lower bound for one-dimensional random walk (Lemma~\ref{lemma:tail}). 

Finally let us compare Theorem~\ref{thm:olo_lower_alt} to Theorem~\ref{thm:olo_upper_alt}. The leading constants in the upper and lower bounds are $2$ and $\sqrt{2}$ respectively (on the leading term $\norms{u}\sqrt{T\log(\norms{u}\sqrt{T})}$). Future works may consider closing this gap. 

\subsection{Algorithm-dependent regret lower bound}\label{subsection:OLO_dual}

In this subsection we convert our player-dependent wealth upper bound (Theorem~\ref{thm:upper_3}) into an algorithm-dependent regret lower bound for unconstrained OLO. The first step is to fix an unconstrained OLO algorithm for our analysis. The ideal choice would be our high-dimensional algorithm (Algorithm~\ref{algorithm:combined}) constructed from $\bar V_{1/2}$. However, the polar decomposition adopted in Algorithm~\ref{algorithm:combined} introduces some technicalities that are non-essential for understanding the nature of this problem. Therefore, we consider the one-dimensional algorithm (Algorithm~\ref{algorithm:combined_1d}), where the polar decomposition is not needed. 

For Algorithm~\ref{algorithm:combined_1d} constructed from $\bar V_{1/2}$, we can state the following regret upper bound using the proof of Theorem~\ref{thm:olo_upper}. Since we do not further bound $f^*_T(|u|)$, such a result is tighter than Theorem~\ref{thm:olo_upper}.

\begin{corollary}\label{thm:olo_upper_dep}
Denote $f_T(S)=\bar V_{1/2}(T,S)$. For all $T\in\N_+$ and $u\in\R$, against any adversary, Algorithm~\ref{algorithm:combined_1d} constructed from $\bar V_{1/2}$ guarantees
\begin{equation*}
\reg_T(u)\leq f^*_T(\abs{u}).
\end{equation*}
\end{corollary}

The Fenchel conjugate can be slightly simplified: if we define $z$ through $|u|=\sqrt{2}C\int_0^z\exp(x^2)dx$, then $f^*_T(|u|)=C\sqrt{T}\exp(z^2)$. Although the order of $|u|$ is not as clear as in Theorem~\ref{thm:olo_upper}, we can numerically evaluate this bound as in our experiments. 

Converting Theorem~\ref{thm:upper_3} to OLO, we have

\begin{theorem}\label{thm:olo_lower_dep}
Denote $f_T(S)=\bar V_{1/2}(T,S)$. For all $T\in\N_+$ and $|u|\leq (3/8)C(T+3)\exp(T/2)$, we can construct a finite sequence of loss gradients $g_1,\ldots,g_T\in[-1,1]$ such that Algorithm~\ref{algorithm:combined_1d} constructed from $\bar V_{1/2}$ has the regret lower bound
\begin{equation*}
\reg_T(u)\geq f^*_T(\abs{u})-O(\abs{u}\log\abs{u}),
\end{equation*}
against the aforementioned loss gradients. $O(\cdot)$ subsumes absolute constants. 
\end{theorem}

\begin{proof}[Proof of Theorem~\ref{thm:olo_lower_dep}]
For convenience, let us define the function
\begin{equation*}
h_T(S)=\bar V_{1/2}\rpar{T,S}+\frac{3C}{8}\exp\rpar{\frac{S^2}{2T}}\rpar{\frac{S^2}{T}+1}+2C.
\end{equation*}
Directly applying Theorem~\ref{thm:upper_3} yields the following result. For all $T\in\N_+$ and $S\in[-T,T]$, there exists $g_1,\ldots,g_T\in[-1,1]$ such that ($i$) $-\sum_{t=1}^Tg_t=S$; and ($ii$) Algorithm~\ref{algorithm:combined_1d} constructed from $\bar V_{1/2}$ satisfies $\sum_{t=1}^Tg_tx_t\geq -h_T(S)$ against loss gradients $g_{1:T}$. 

Define a variable $u^*$ as
\begin{equation*}
u^*= h'_T(S)=\sqrt{2}C\int_0^{S/\sqrt{2T}}\exp(x^2)dx+\frac{3CS}{8T}\exp\rpar{\frac{S^2}{2T}}\rpar{\frac{S^2}{T}+3}.
\end{equation*}
Since $S$ is arbitrary within the interval $[-T,T]$, $u^*$ can take any value within $[-U,U]$, where $U=(3/8)C(T+3)\exp(T/2)$. Due to a standard result from convex analysis \cite[Theorem~23.5]{rockafellar2015convex}, $h_T(S)+h^*_T(u^*)=Su^*$. Therefore, 
\begin{equation*}
\reg_T(u^*)=\sum_{t=1}^Tg_tx_t-\sum_{t=1}^Tg_t u^*\geq -h_T(S)+Su^*=h^*_T(u^*). 
\end{equation*}
The remaining task is to lower bound $h^*_T(\cdot)$. 

Without loss of generality, assume $u\geq 0$. Let us define a variable $\tilde S$ through the equation
\begin{equation*}
u=\sqrt{2}C\int_0^{\tilde S/\sqrt{2T}}\exp(z^2)dz.
\end{equation*}
Then, using the proof of Theorem~\ref{thm:olo_upper}, 
\begin{equation*}
h^*_T(u)=\sup_{S\in\R} Su-h_T(S)\geq \tilde Su-h_T(\tilde S)=f^*_T(u)-\frac{3C}{8}\exp\rpar{\frac{\tilde S^2}{2T}}\rpar{\frac{\tilde S^2}{T}+1}-2C,
\end{equation*}
and
\begin{equation*}
\tilde S\leq \sqrt{2T}\spar{\sqrt{\log\rpar{1+\frac{u}{\sqrt{2}C}}}+1}.
\end{equation*}
Combining the above completes the proof. 
\end{proof}

Comparing Corollary~\ref{thm:olo_upper_dep} to Theorem~\ref{thm:olo_lower_dep}, the leading terms in the player-dependent bounds are exactly the same. The gap between the upper and lower bounds does not depend on time. That is, we have a good estimate of the worst case performance of Algorithm~\ref{algorithm:combined_1d}.

\section{Detail on experiments}\label{section:appendix_experiments}

We now present details on our experiments. First, we introduce the KT algorithm \cite{orabona2016coin} as our baseline. It is perhaps the most well-known parameter-free algorithm for unconstrained OLO. Essentially, it is an optimistic version of Algorithm~\ref{algorithm:combined_1d} induced by the existing potential $\bar V_{-1/2}$. Next, we discuss the choice of hyperparameters in our experiments. In the last three subsections, we present empirical results omitted from the main paper. 

\subsection{Baseline: Krichevsky-Trofimov algorithm}

We first consider the one-dimensional version of the KT algorithm, whose pseudo-code is presented as Algorithm~\ref{algorithm:KT}. Theoretically it guarantees a similar bound as Theorem~\ref{thm:olo_upper_alt}, with only minor differences on the non-leading constants. 

\begin{algorithm*}[ht]
\caption{The Krichevsky-Trofimov algorithm.\label{algorithm:KT}}
\begin{algorithmic}[1]
\REQUIRE Initial wealth $\eps>0$. 
\FOR{$t=1,2,\ldots$}
\STATE Predict $x_t=(-\sum_{i=1}^{t-1}g_i/t)\cdot(\eps-\sum_{i=1}^{t-1}g_ix_i)$
\STATE Observe $g_t$ and store it. 
\ENDFOR
\end{algorithmic}
\end{algorithm*}

\begin{lemma}[Corollary 5 of \cite{orabona2016coin}]\label{lemma:KT}
For all $T\in\N_+$ and $u\in\R$, against any adversary, Algorithm~\ref{algorithm:KT} guarantees
\begin{equation*}
\reg_T(u)\leq \eps+\abs{u}\sqrt{T\log\rpar{1+\frac{24\abs{u}^2T^2}{\eps^2}}}.
\end{equation*}
\end{lemma}

The one-dimensional KT algorithm can be naturally extended to higher dimensions. Specifically, we wrap it using Algorithm~\ref{algorithm:extension} (the reduction from \cite{cutkosky2018black}), just like how Algorithm~\ref{algorithm:combined} extends Algorithm~\ref{algorithm:combined_1d} to higher dimensions. 

\subsection{Choice of hyperparameters}\label{subsection:hyper}

We now discuss the choice of hyperparameter $C$ in the two versions of Algorithm~\ref{algorithm:combined_1d}. Note that since both versions are parameter-free algorithms, the hyperparameter $C$ does not affect their performance as critically as the learning rate in OGD: for any $C$, the regret upper bound has the same asymptotic order (but with different minor constants). Specifically we choose $C=1$ in both versions. One reason is that this is the most natural choice when no information is available beforehand. More importantly, at the beginning of the optimization process, $C=1$ induces the same asymptotic exponential growth rate for the predictions of the two versions. (As we discussed in Section~\ref{section:empirical}, such an exponential growth is the key for the success of parameter-free algorithms.)

Concretely, the predictions of the both versions are roughly the gradients of the potentials, which are $\nabla_S\bar V_{-1/2}(t,S)=CS t^{-3/2}\exp[S^2/(2t)]$ for $\bar V_{-1/2}$ and $\nabla_S\bar V_{1/2}(t,S)=\sqrt{2}C\int_0^{S/\sqrt{2t}}\exp(x^2)dx$ for $\bar V_{1/2}$. At the beginning, all the gradient feedback are one-sided, therefore $|S|=t$. Applying $S=t$ and taking the derivative with respect to $t$, the growth rate of predictions based on $\bar V_{-1/2}$ is
\begin{equation*}
\nabla_t\spar{\nabla_S\bar V_{-1/2}(t,S)\big|_{S=t}}=\frac{C}{2\sqrt{t}}\rpar{1-\frac{1}{t}}\exp\rpar{\frac{t}{2}}. 
\end{equation*}
For $\bar V_{1/2}$ we have
\begin{equation*}
\nabla_t\spar{\nabla_S\bar V_{1/2}(t,S)\big|_{S=t}}=\frac{C}{2\sqrt{t}}\exp\rpar{\frac{t}{2}}. 
\end{equation*}
The leading terms would match if the hyperparameters of the two versions are the same. 

As for the initial wealth $\eps$ in KT, by comparing Theorem~\ref{thm:olo_upper_alt} and Lemma~\ref{lemma:KT} we can see that $\eps=\sqrt{e}C$ is the most reasonable choice. It matches the maximum allowable $\reg_T(0)$ in the KT algorithm and the version of Algorithm~\ref{algorithm:combined_1d} based on $\bar V_{-1/2}$. 

\subsection{Omitted results on 1d OCO}\label{subsection:exp_omitted_1d}

\paragraph{Testing more cases of $u^*$} We first present more cases of $u^*$ to support Figure~\ref{fig:1a}. Figure~\ref{fig:onedimensional_more} shows that for $u^*\geq 1$, our algorithm consistently beats the baselines. Note that the vertical scale in each subfigure is different. Using a unified scale, Figure~\ref{fig:1b} in the main paper plots the gap between the green line and the blue line at $T=500$. (The two baselines are similar, therefore the orange line is not considered in Figure~\ref{fig:1b}.)

\begin{figure}[ht]
     \centering
     \begin{subfigure}[b]{0.32\textwidth}
         \centering
         \includegraphics[width=\textwidth]{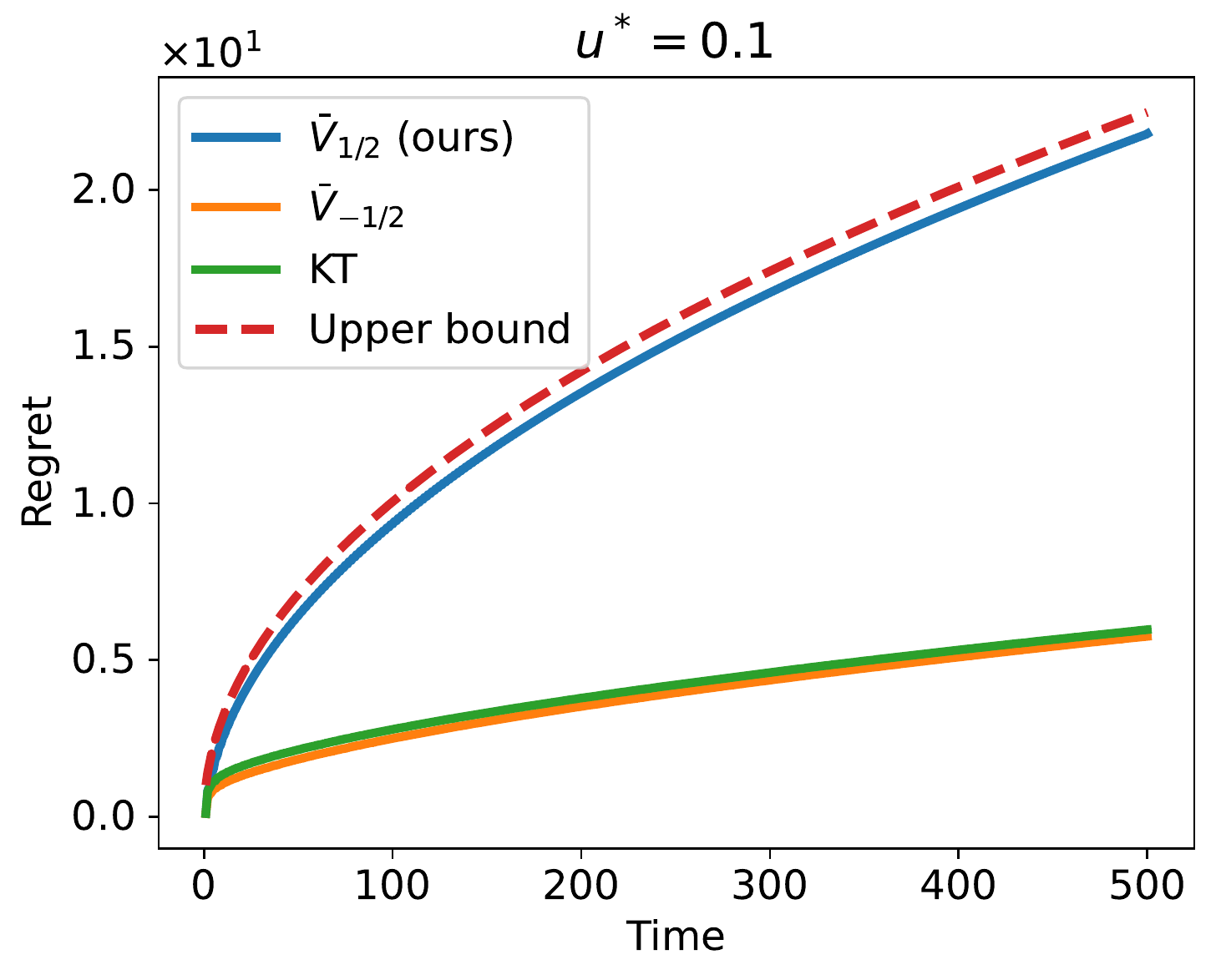}
     \end{subfigure}
     \hfill
     \begin{subfigure}[b]{0.31\textwidth}
         \centering
         \includegraphics[width=\textwidth]{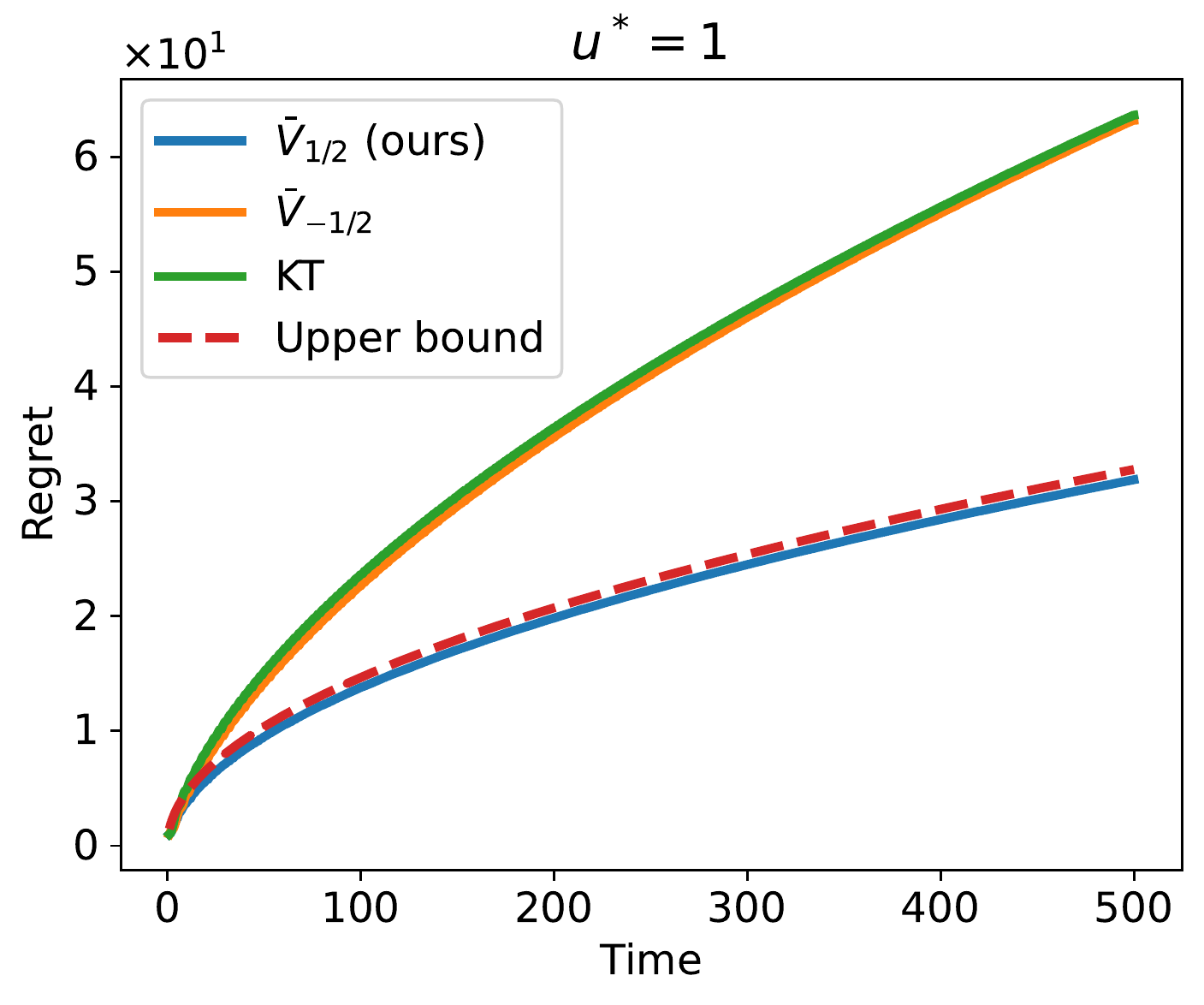}
     \end{subfigure}
     \hfill
     \begin{subfigure}[b]{0.31\textwidth}
         \centering
         \includegraphics[width=\textwidth]{OneD_Setting1_3.pdf}
     \end{subfigure}\\
          \begin{subfigure}[b]{0.31\textwidth}
         \centering
         \includegraphics[width=\textwidth]{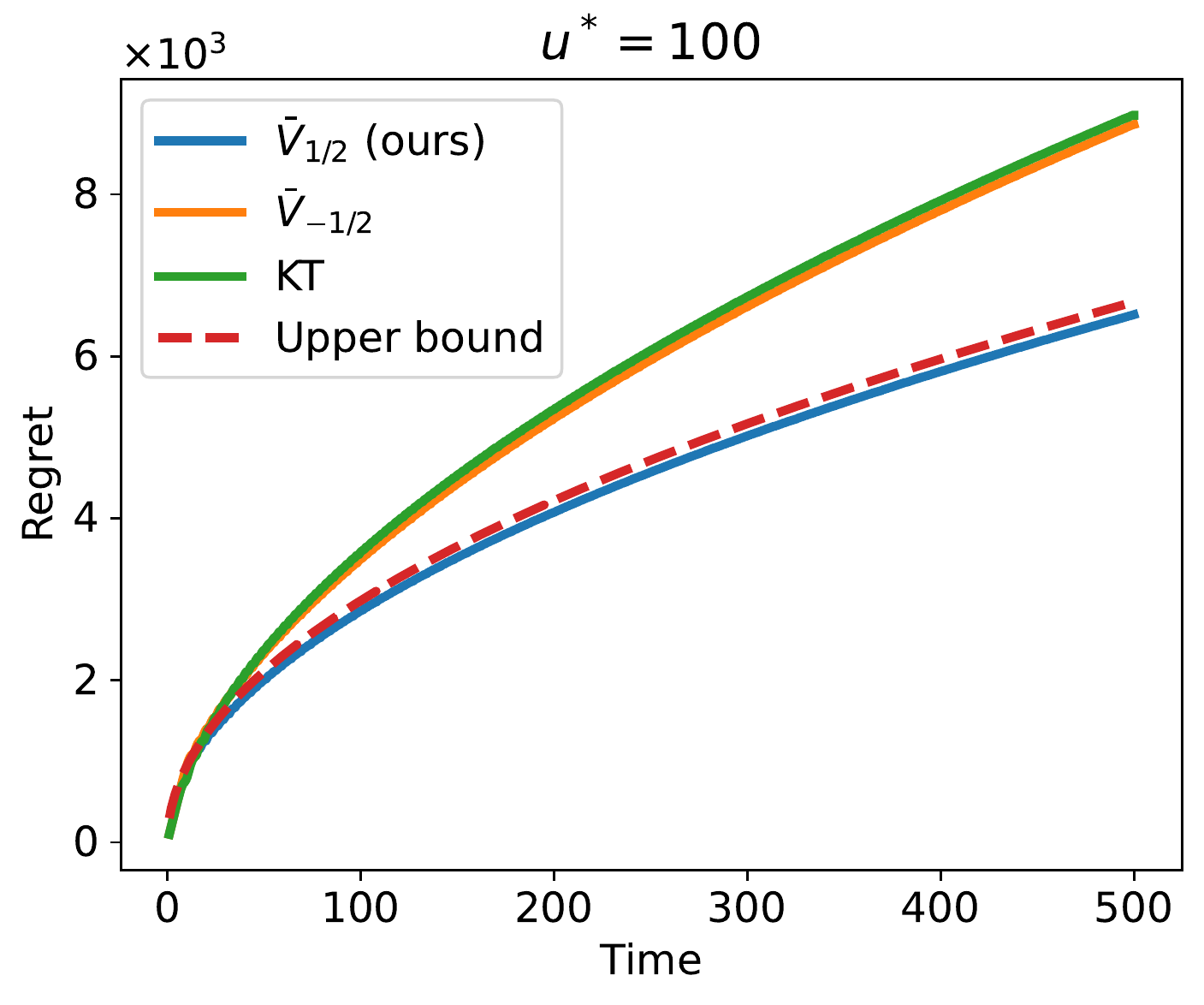}
     \end{subfigure}
     \hfill
     \begin{subfigure}[b]{0.32\textwidth}
         \centering
         \includegraphics[width=\textwidth]{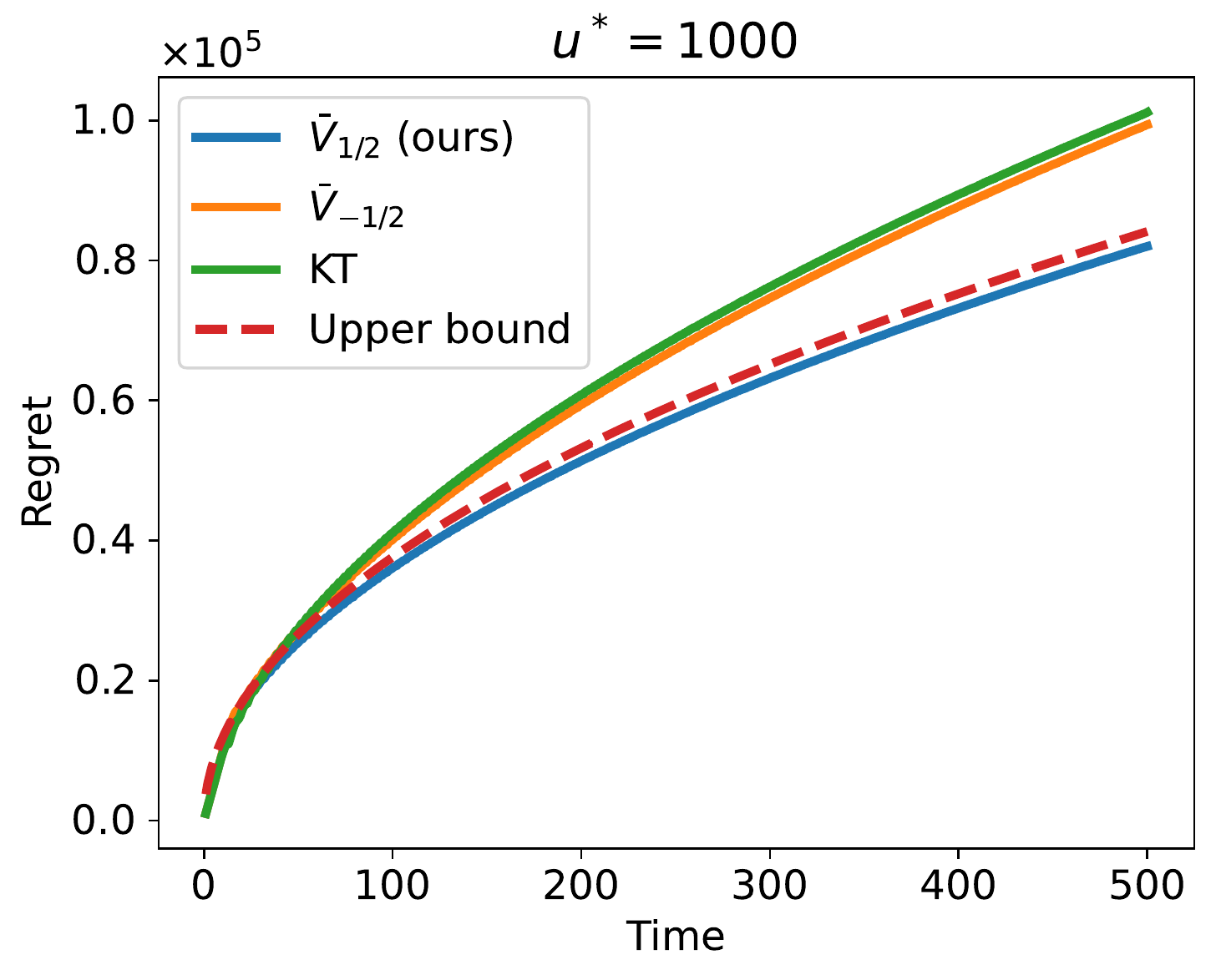}
     \end{subfigure}
     \hfill
     \begin{subfigure}[b]{0.32\textwidth}
         \centering
         \includegraphics[width=\textwidth]{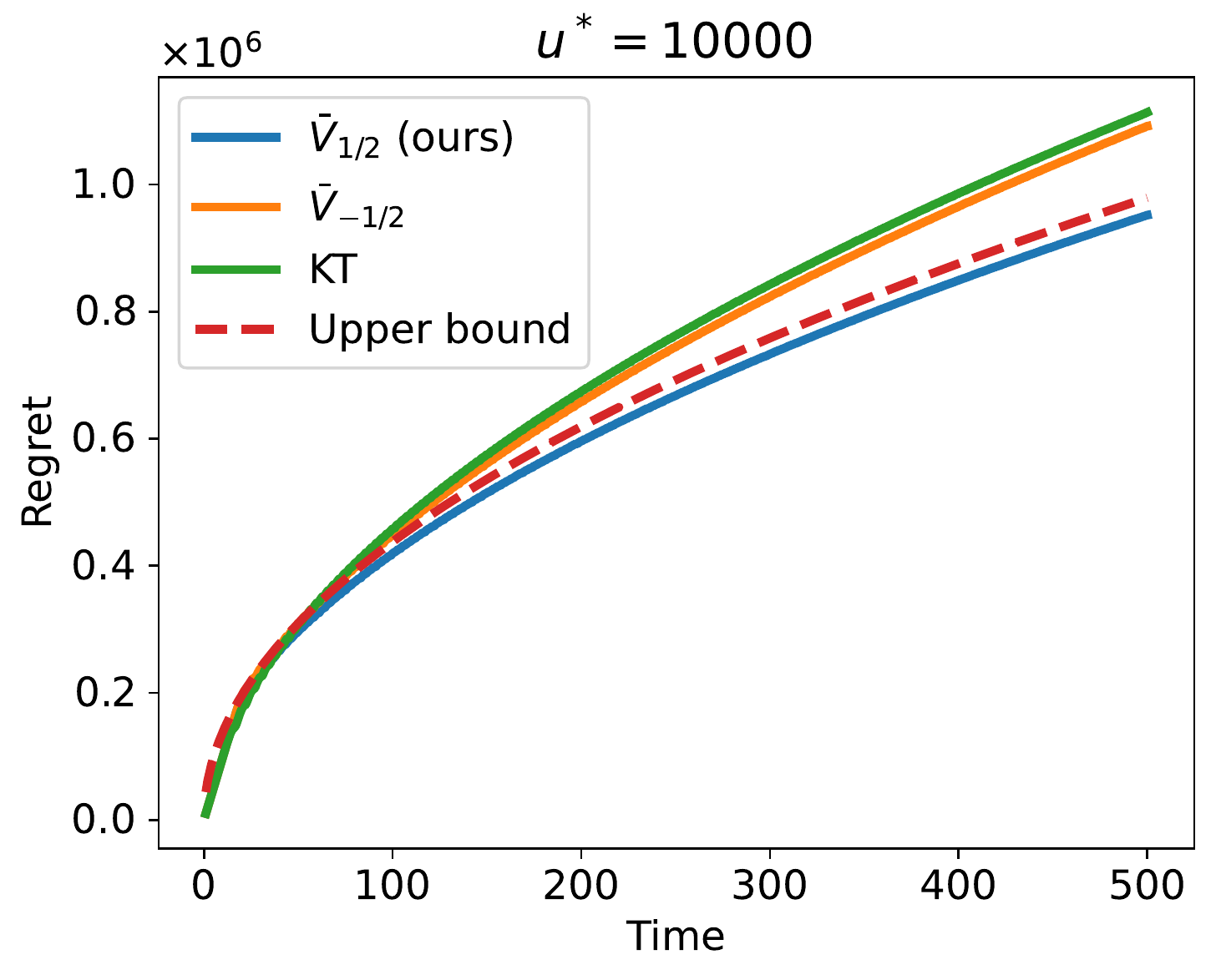}
     \end{subfigure}
        \caption{More cases of $u^*$ to support Figure~\ref{fig:1a}; $C=1$, $T=500$. }
        \label{fig:onedimensional_more}
\end{figure}

\paragraph{The effect of $T$} Next, we investigate the effect of the maximum time horizon $T$. When closely comparing the regret upper bounds of the two potential-based algorithms (Theorem~\ref{thm:olo_upper} and \ref{thm:olo_upper_alt}), one can see that for all fixed $C$ and nonzero $u^*$, the upper bound based on the new potential $\bar V_{1/2}$ is always better if $T$ is long enough ($O(\sqrt{T})$ as opposed to $O(\sqrt{T\log T})$). Then, a reasonable guess is that for some small $u^*$, the performance of our algorithm may be weaker than the baselines at $T=500$ (Figure~\ref{fig:onedimensional_more}), but better than the baselines at larger $T$. Such a guess is true in certain cases, as shown in Figure~\ref{fig:onedimensional_time}. Specifically, we pick $u^*=0.3$ and vary the maximum $T$. Initially our algorithm is worse, but as $T$ increases it can still outperform the baselines. 

\begin{figure}[ht]
     \centering
     \begin{subfigure}[b]{0.32\textwidth}
         \centering
         \includegraphics[width=\textwidth]{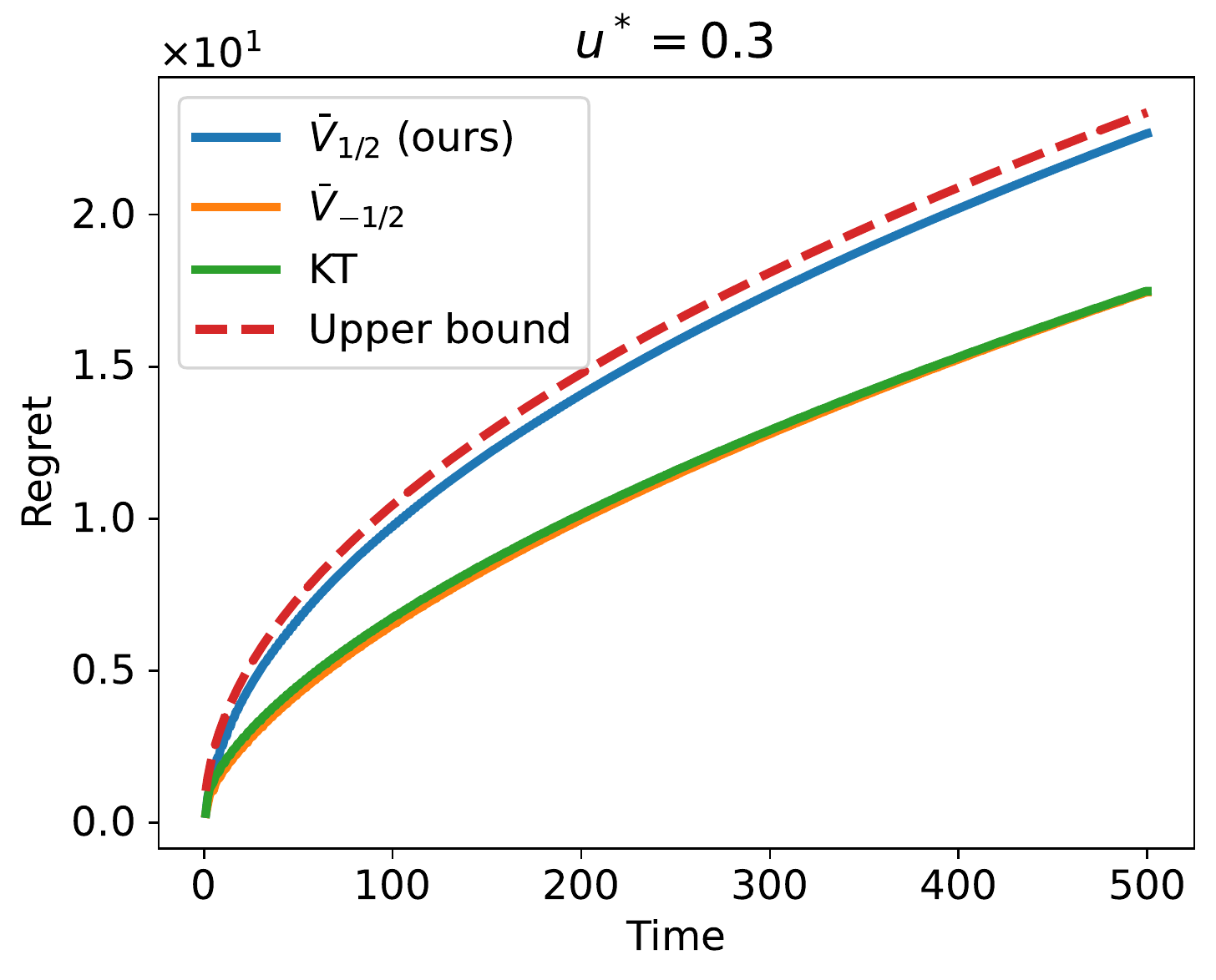}
     \end{subfigure}
     \hfill
     \begin{subfigure}[b]{0.32\textwidth}
         \centering
         \includegraphics[width=\textwidth]{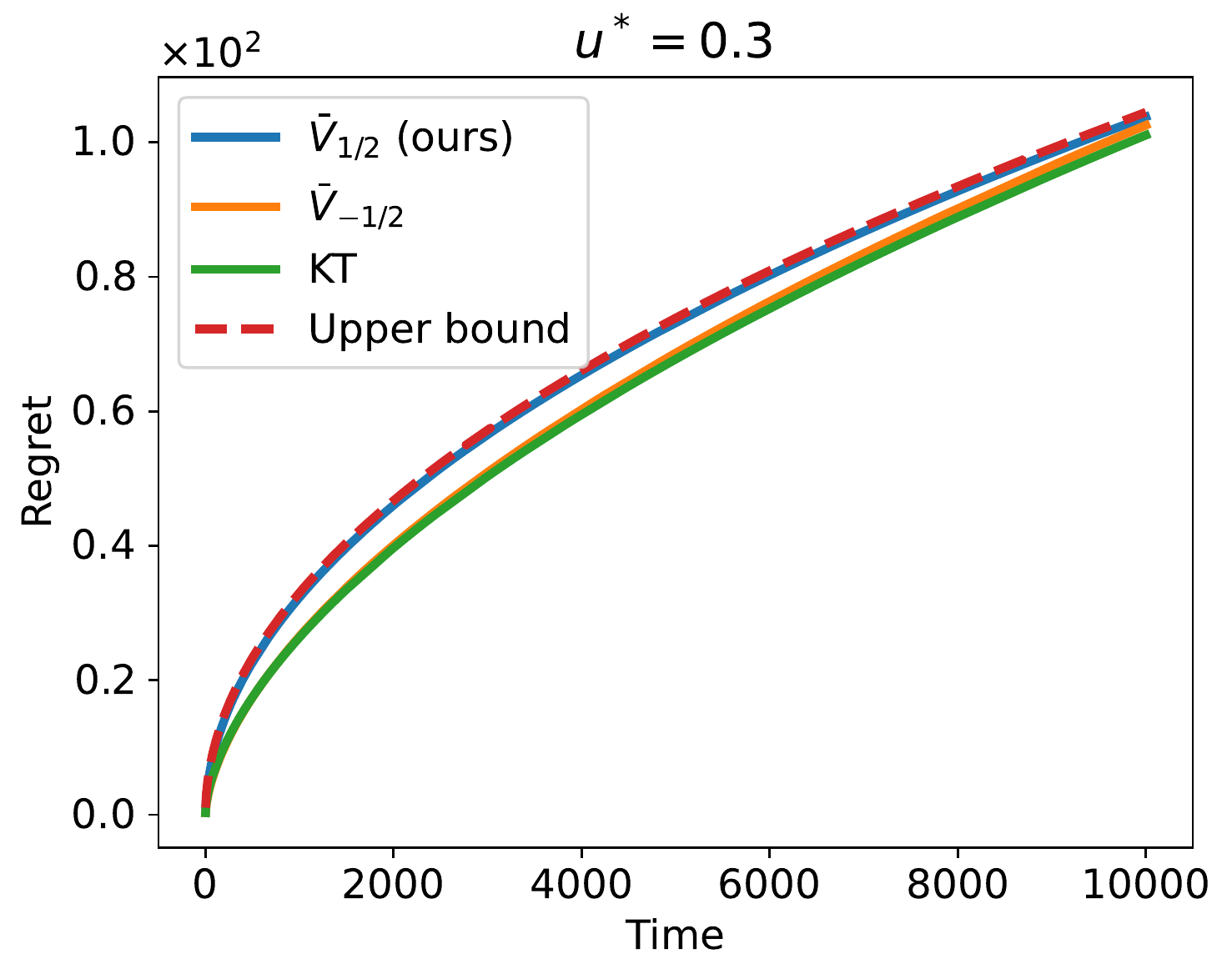}
     \end{subfigure}
     \hfill
     \begin{subfigure}[b]{0.32\textwidth}
         \centering
         \includegraphics[width=\textwidth]{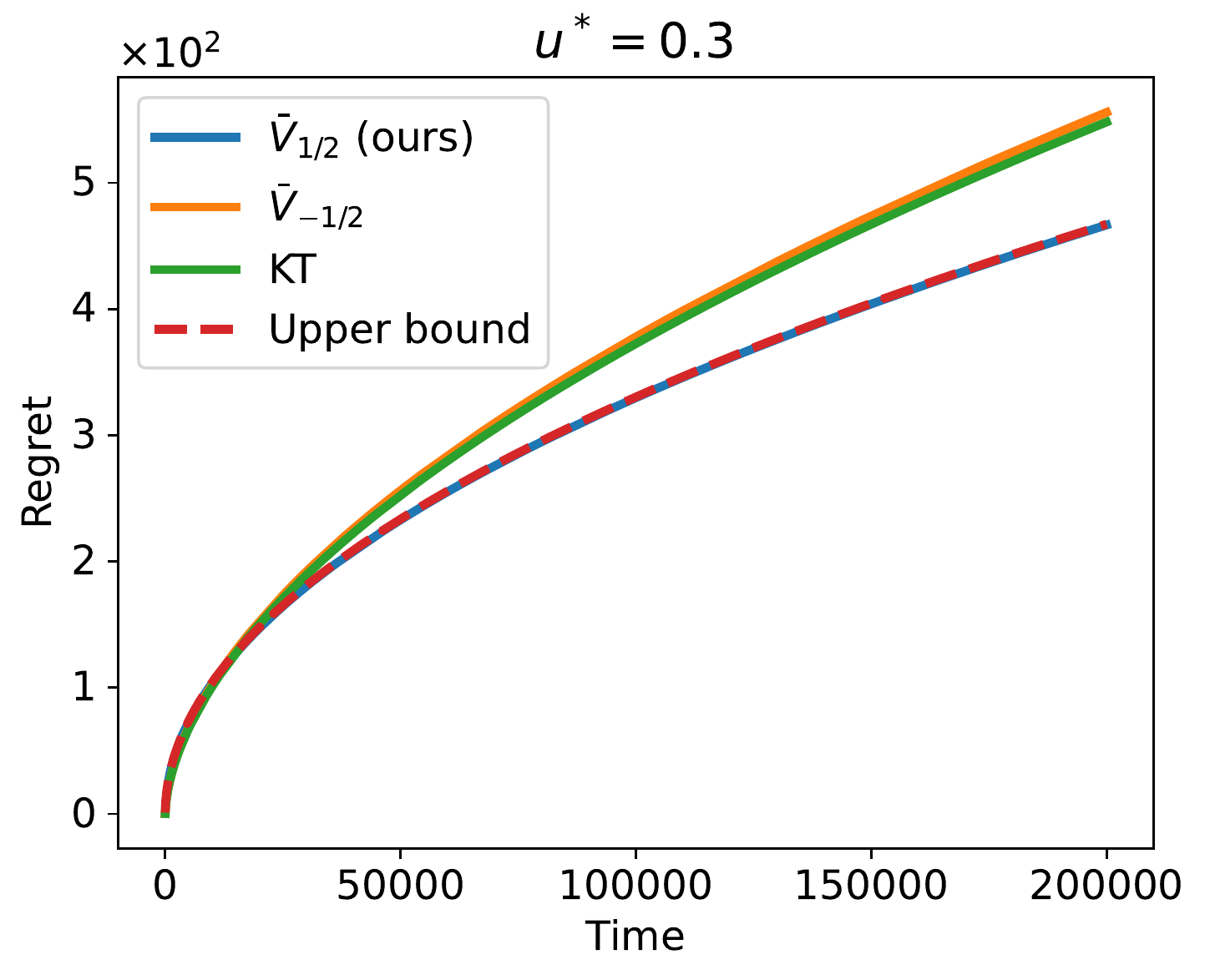}
     \end{subfigure}
        \caption{One-dimensional task with $C=1$ and $u^*=0.3$. From left to right: $T=500,10000,200000$.}
        \label{fig:onedimensional_time}
\end{figure}

\paragraph{A different hyperparameter $C$} Finally, we investigate the effect of the hyperparameter $C$ on the qualitative comparison of the three algorithms. We change $C$ to 10 and present results parallel to Figure~\ref{fig:onedimensional_more} in Figure~\ref{fig:onedimensional_C}. The initial wealth of KT is scaled accordingly to $10\sqrt{e}$. 

\begin{figure}[ht]
     \centering
     \begin{subfigure}[b]{0.32\textwidth}
         \centering
         \includegraphics[width=\textwidth]{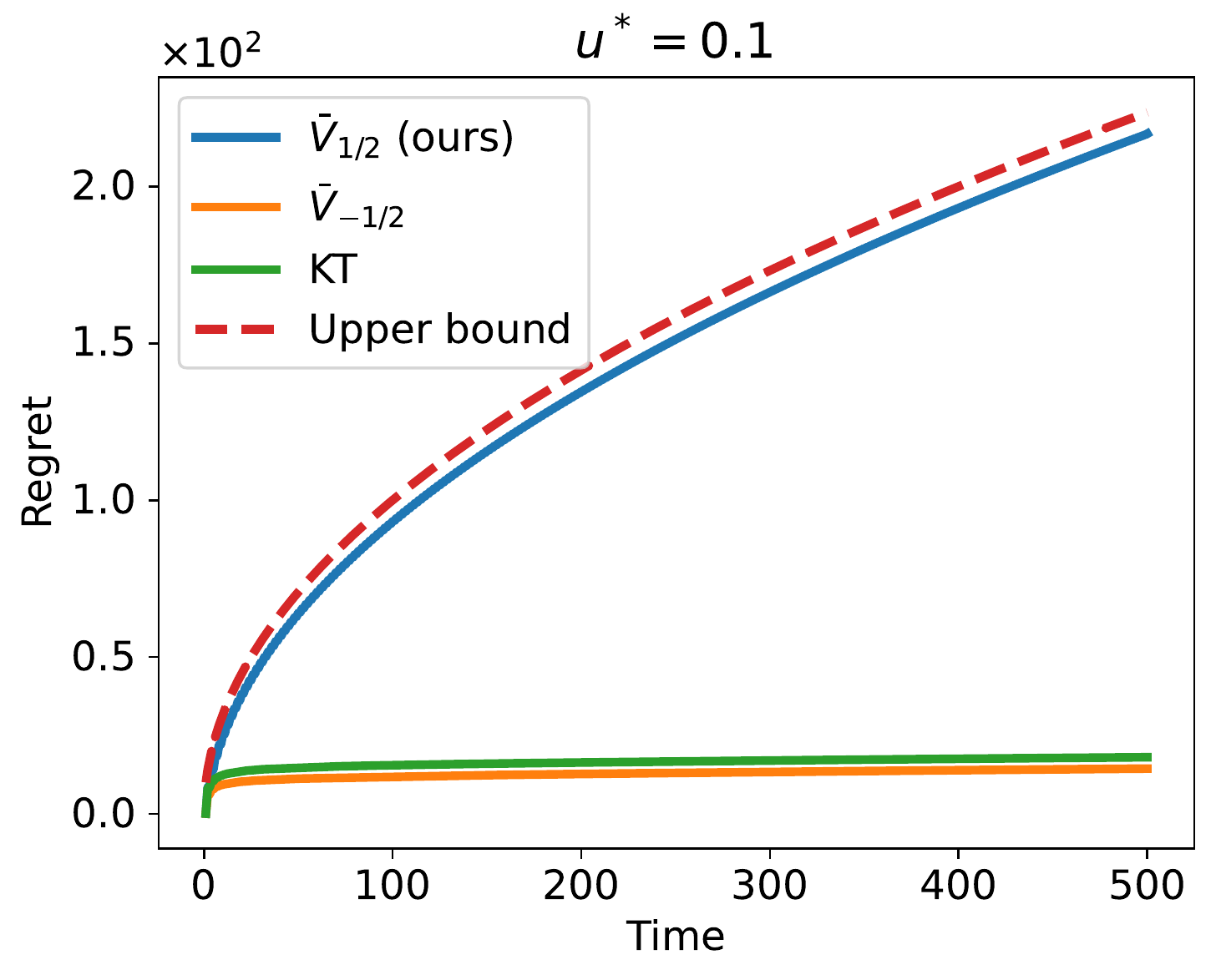}
     \end{subfigure}
     \hfill
     \begin{subfigure}[b]{0.32\textwidth}
         \centering
         \includegraphics[width=\textwidth]{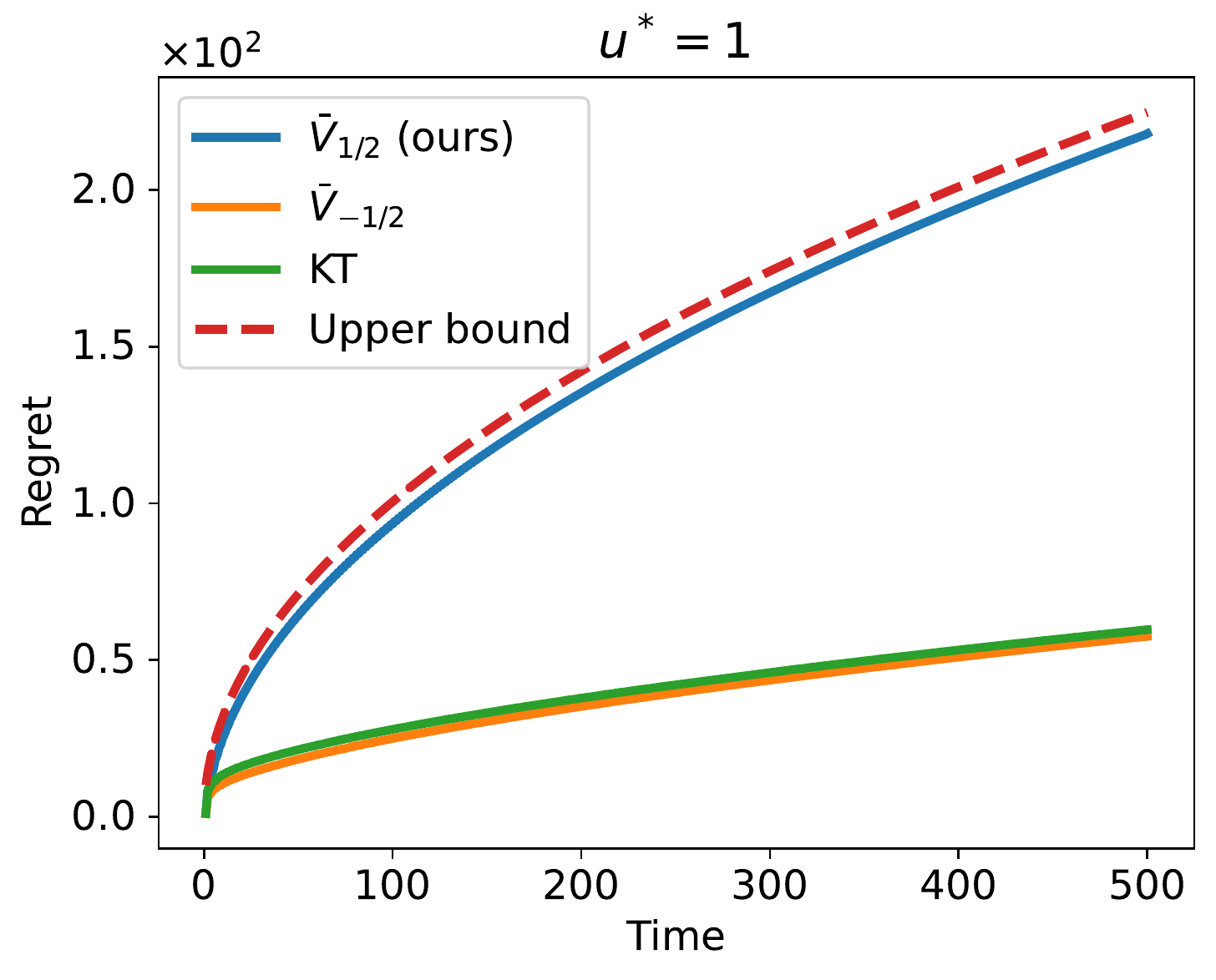}
     \end{subfigure}
     \hfill
     \begin{subfigure}[b]{0.31\textwidth}
         \centering
         \includegraphics[width=\textwidth]{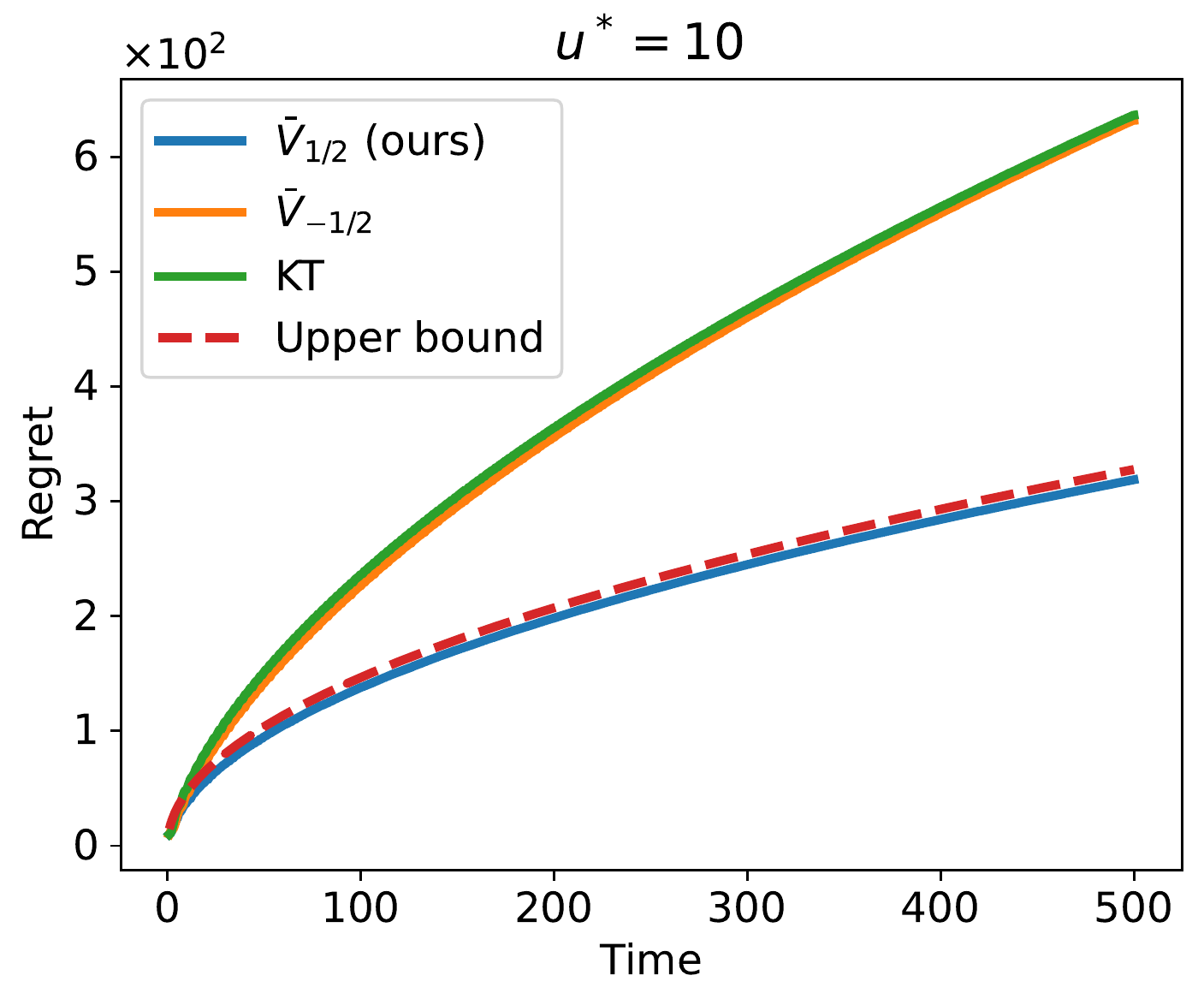}
     \end{subfigure}\\
     \begin{subfigure}[b]{0.31\textwidth}
         \centering
         \includegraphics[width=\textwidth]{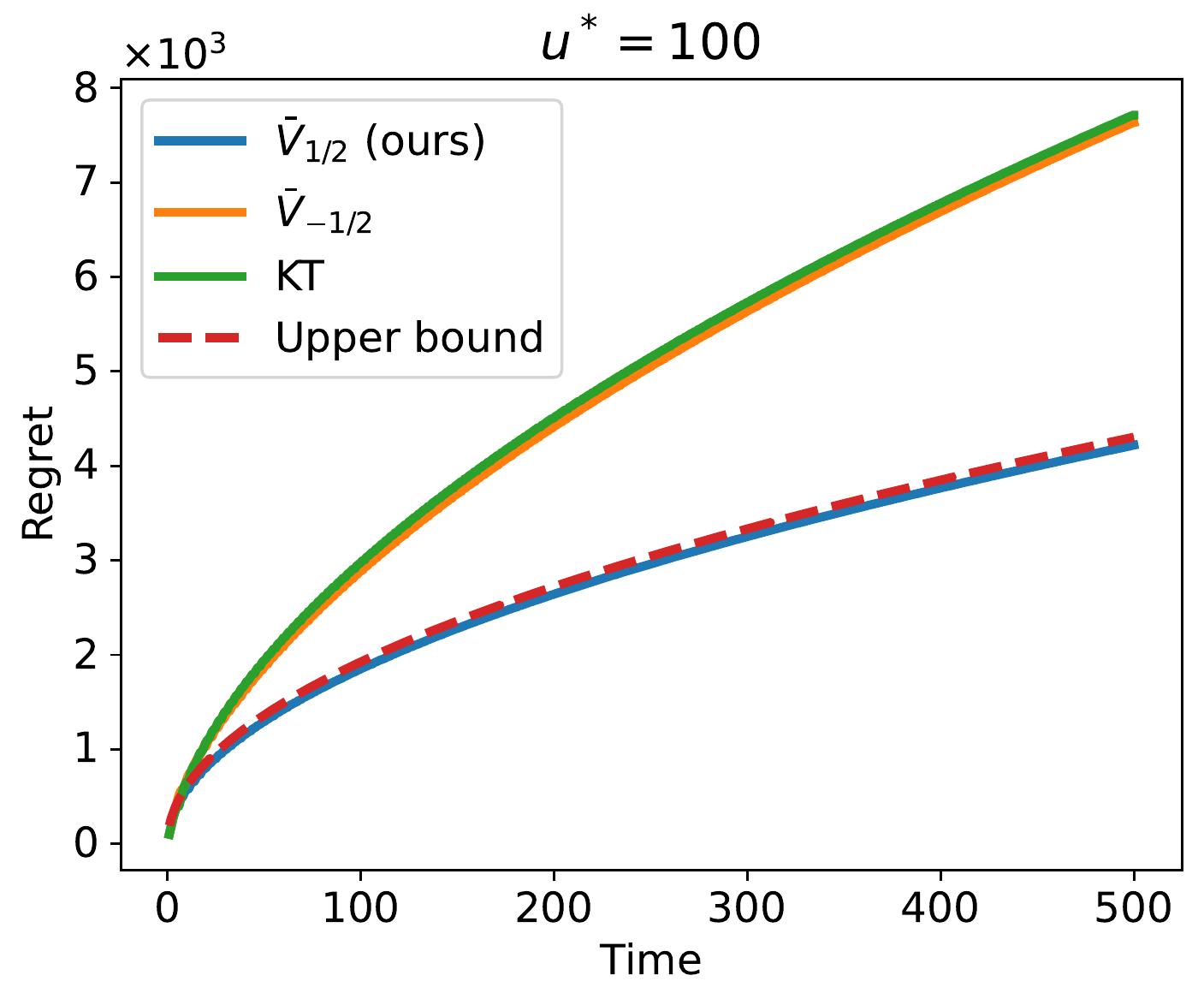}
     \end{subfigure}
     \hfill
     \begin{subfigure}[b]{0.31\textwidth}
         \centering
         \includegraphics[width=\textwidth]{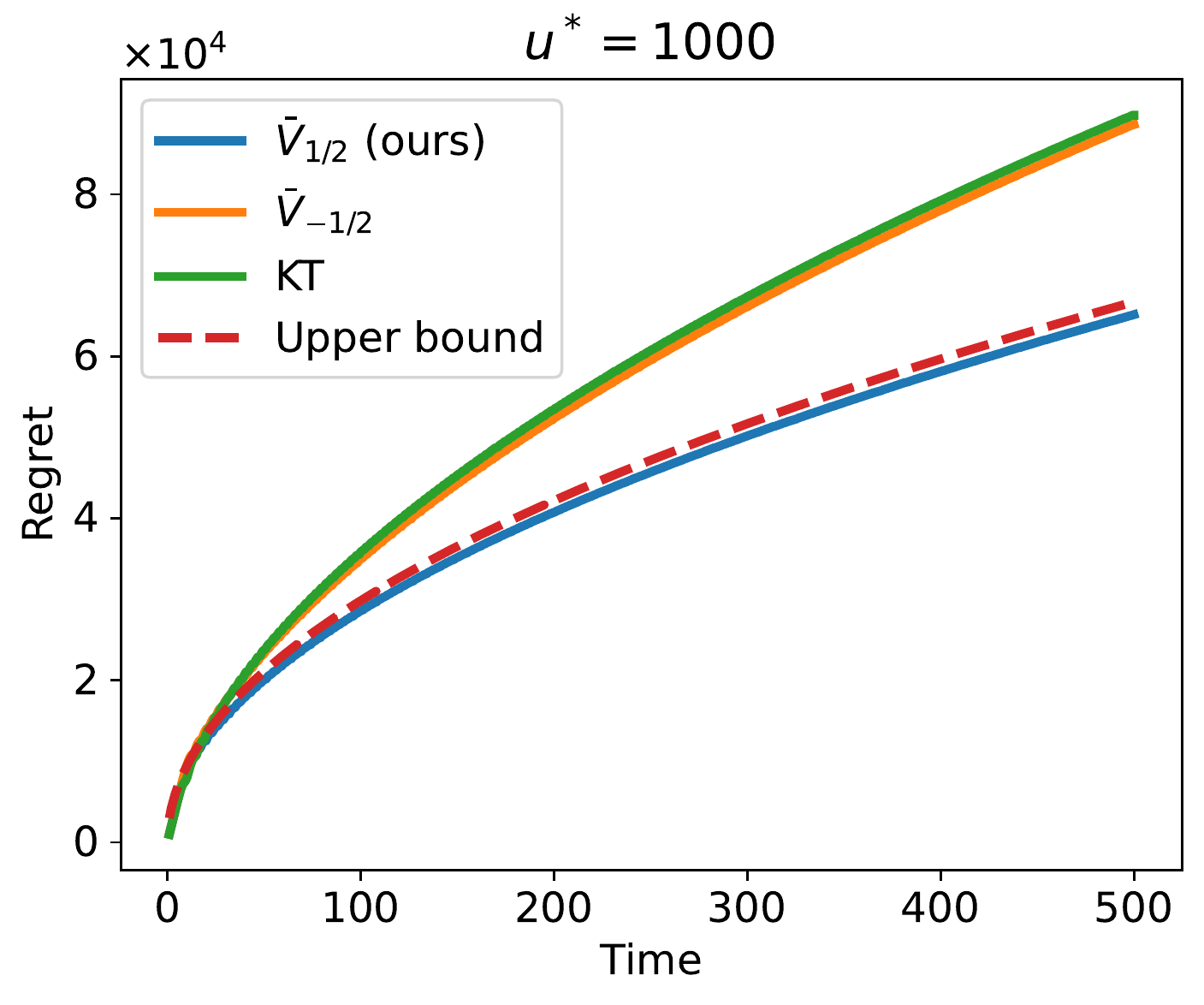}
     \end{subfigure}
     \hfill
     \begin{subfigure}[b]{0.32\textwidth}
         \centering
         \includegraphics[width=\textwidth]{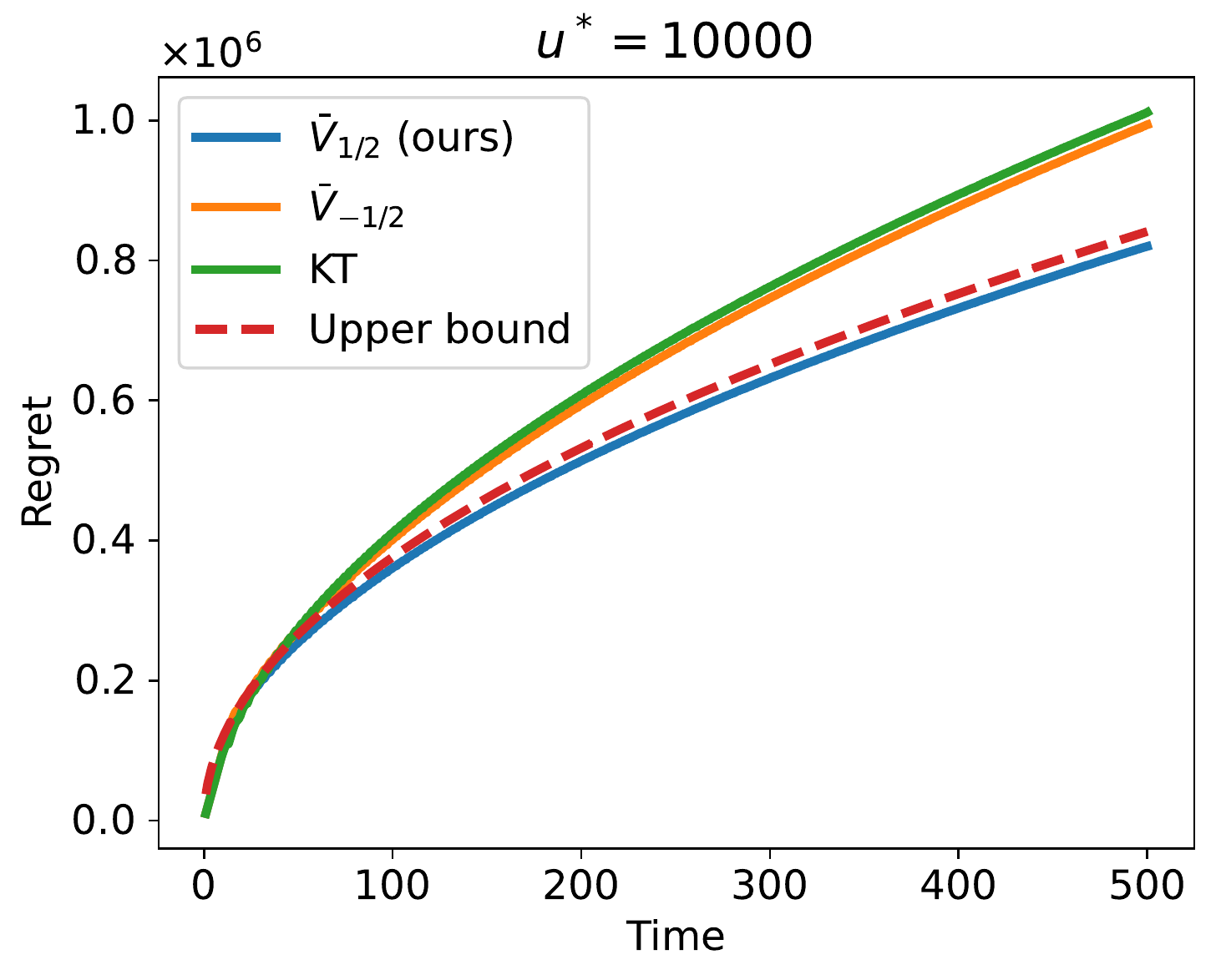}
     \end{subfigure}
        \caption{One-dimensional task with $C=10$.}
        \label{fig:onedimensional_C}
\end{figure}

Figure~\ref{fig:onedimensional_C} exhibits a similar behavior as Figure~\ref{fig:onedimensional_more}: while sacrificing the regret at small $|u^*|$, our algorithm is better when $u^*$ is far-away. However, we also see that our algorithm exhibits less qualitative improvement over the baselines: in order to beat our algorithm (at $T=500$), previously (with $C=1$) the baselines should initialize at $\tilde u$ with error $|\tilde u-u^*|\leq 1$, but now (with $C=10$) such an error is allowed to be less than 10. A possible concern is that the advantage of our algorithm becomes harder to justify in this setting. We address this concern from three different perspectives. 
\begin{enumerate}
\item Even with $C=10$, our algorithm still outperforms the baselines when $u^*$ is \emph{everywhere except on a compact set}. Therefore, our algorithm still works better in more situations (of $u^*$). 
\item Theoretically, for all fixed $C$ and nonzero $u^*$, our algorithm always guarantees better regret bound than the baselines when $T$ is large enough. Empirically this is validated in Figure~\ref{fig:onedimensional_time}. 
\item The key idea of parameter-free algorithms is to use a simple hyperparameter to replace the laborious tuning of learning rates. In practice (e.g., \cite{orabona2016coin,chen2022better}), such a hyperparameter is often simply set to 1, and the resulting algorithms already exhibit strong empirical performance. Actually, changing $C$ amounts to trading off loss with regret; without any prior knowledge, the most \emph{natural} choice is perhaps $C=1$. Therefore, when our algorithm and the two baselines are in the \emph{most natural configuration}, our algorithm has the best performance unless a very accurate guess of $u^*$ is known a priori (Figure~\ref{fig:onedimensional_more}). 
\end{enumerate}

\subsection{Additional experiment: 1D OLO with stochastic loss}

As suggested by an anonymous reviewer, we now consider another one-dimensional experimental setting where loss gradients are iid stochastic, rather than generated from linearization. We choose $C=1$ and $T=500$; the loss gradients $g_t$ are generated iid on the support $\{\pm 1\}$, with mean $0.2$. Each algorithm is run 50 times, and the mean of the negative cumulative loss $-\sum_{t=1}^Tg_tx_t$ as a function of $T$ is plotted in Figure~\ref{fig:oneD_stochastic}; higher is better. Our algorithm beats both baselines in this setting.

\begin{figure}[ht]
\centering
\includegraphics[width=0.35\textwidth]{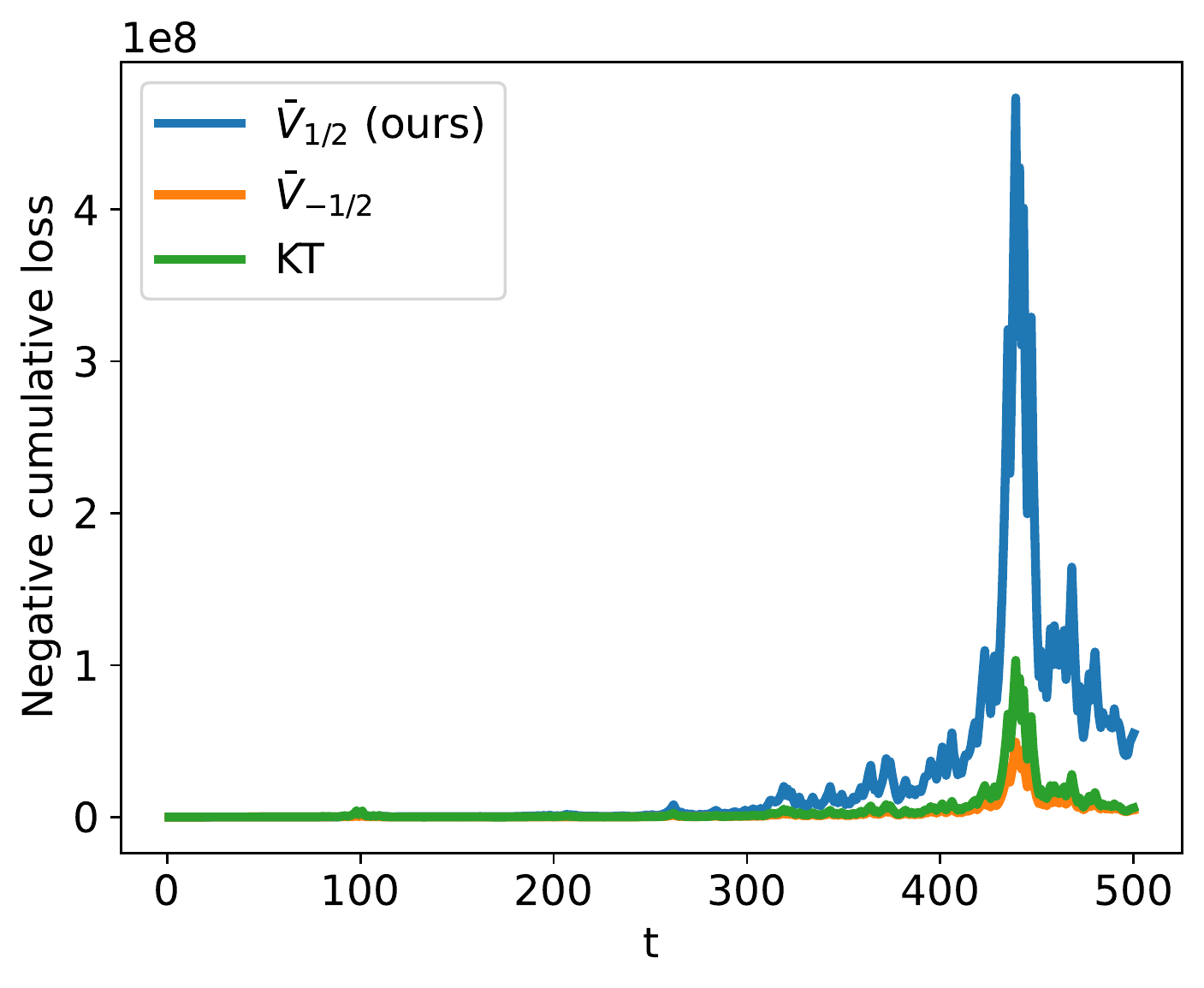}
\caption{1D OLO with stochastic losses. The plot shows the negative cumulative loss as a function of $T$; higher is better.}\label{fig:oneD_stochastic}
\end{figure}

\subsection{Additional experiment: High-dimensional regression with real data}\label{subsection:exp_omitted_regression}

In the last subsection, we report a high-dimensional regression experiment with real data. We use the YearPredictionMSD dataset \cite{Bertin-Mahieux2011} available from the UCI Machine Learning Repository \cite{Dua:2019}, and the context of this dataset is to predict the release year of a song from its audio features. The raw data is preprocessed in two steps. 
\begin{enumerate}
\item Feature normalization. For all the features (columns of the data matrix), we perform min-max scaling to transform their range to $[0,1]$. 
\item Row scaling. For all the feature vectors (rows of the data matrix, i.e., training samples), we scale them such that each feature vector has $L_2$-norm 1. This is due to the Lipschitz requirement in our setting, and the same procedure has been performed in prior works (e.g., \cite{orabona2016coin}). 
\end{enumerate}
After that, we use a linear model with absolute loss $l_t(x)=|\inner{z_t}{x}-y_t|$, where $z_t\in\R^{90}$ and $y_t\in\R_+$ are the $t$-th sampled feature vector and target. This can be converted into a 90-dimensional unconstrained OLO problem: the adversary picks $g_t=z_t$ if $\inner{z_t}{x_t}\geq y_t$, while $g_t=-z_t$ otherwise. Same as before, we consider three algorithms with $C=1$: ($i$) Algorithm~\ref{algorithm:combined} constructed from $\bar V_{1/2}$; ($ii$) Algorithm~\ref{algorithm:combined} constructed from $\bar V_{-1/2}$; and ($iii$) KT. 

To study how these algorithms adapt to the distance to the optimal comparator, we use a parameter $\gamma$ to scale the target $y_t$. That is, we assign $y_t\leftarrow \gamma y_t$, and $\gamma$ is varied across different settings. Due to the linearity of our regression model, the optimal comparator is effectively scaled by $\gamma$. With a small $\gamma$, the comparators are brought closer to our initialization $0$. Note that such a scaling does not work if we use a nonlinear regression model. In those general cases, one may only care about the unscaled setting ($\gamma=1$).

Our results are presented in two ways, ($i$) fix $\gamma$ and vary $T$; ($ii$) fix $T$ and vary $\gamma$. Specifically, since $(z_t,y_t)$ is sampled from the dataset, in each setting (of $\gamma$) we run each algorithm 5 times and use the average cumulative OCO loss $\sum_{t=1}^Tl_t(x_t)$ as the ``TotalLoss'' of this algorithm. Figure~\ref{fig:highdimensional_more} shows the type-($i$) results where we fix $\gamma$ at different values and plot TotalLoss as a function of $T$. 

\begin{figure}[ht]
     \centering
     \begin{subfigure}[b]{0.32\textwidth}
         \centering
         \includegraphics[width=\textwidth]{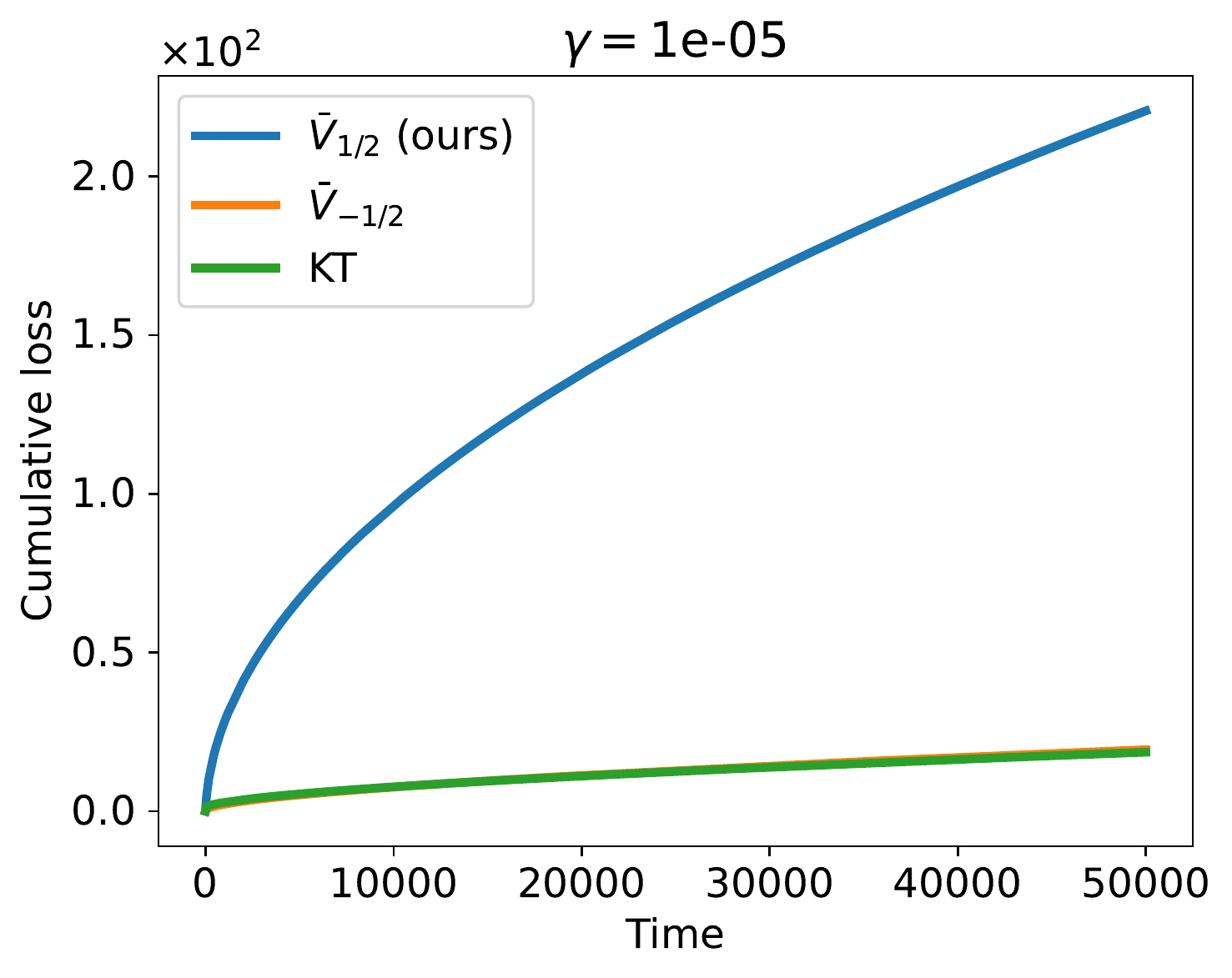}
     \end{subfigure}
     \hfill
     \begin{subfigure}[b]{0.32\textwidth}
         \centering
         \includegraphics[width=\textwidth]{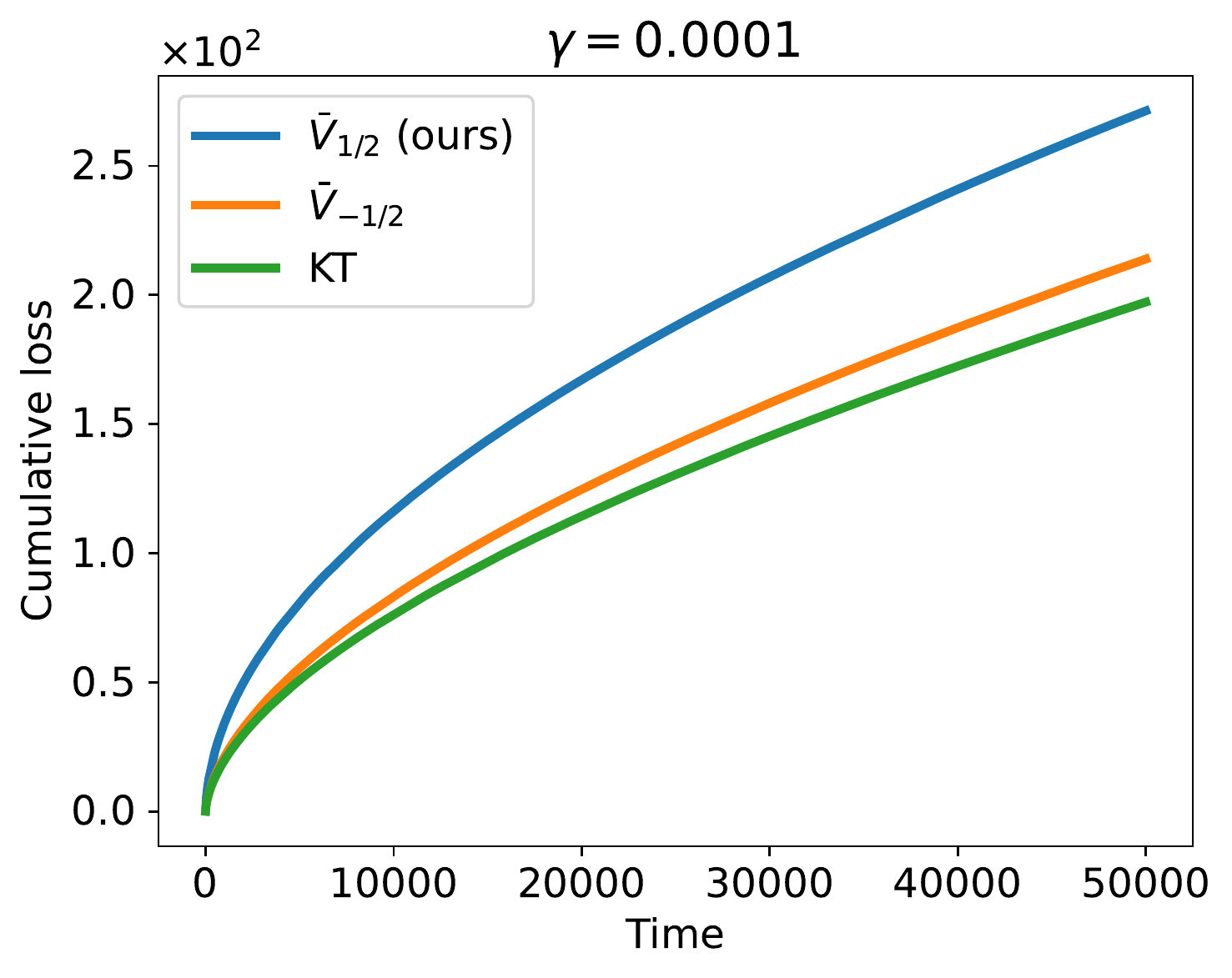}
     \end{subfigure}
     \hfill
     \begin{subfigure}[b]{0.32\textwidth}
         \centering
         \includegraphics[width=\textwidth]{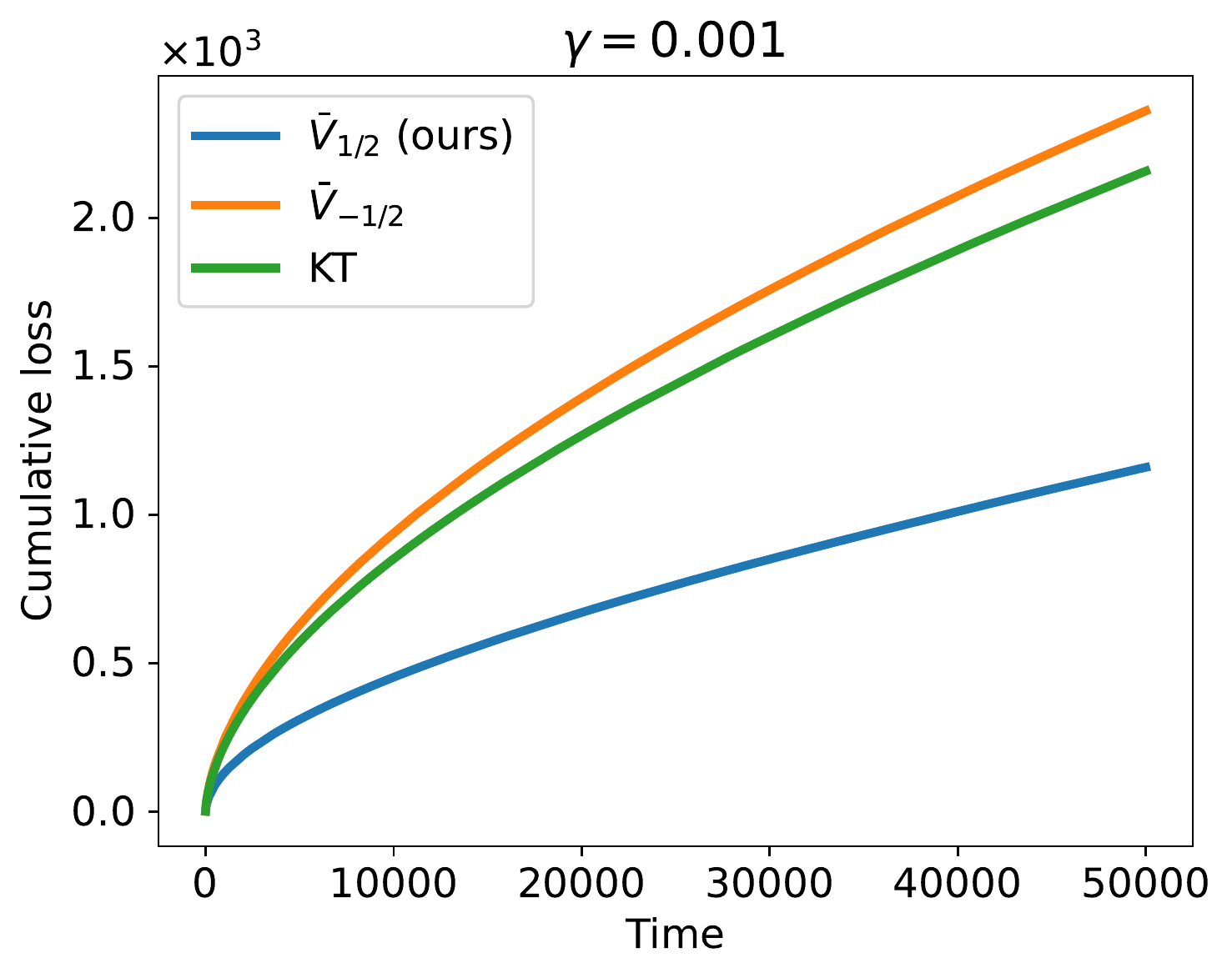}
     \end{subfigure}\\
          \begin{subfigure}[b]{0.32\textwidth}
         \centering
         \includegraphics[width=\textwidth]{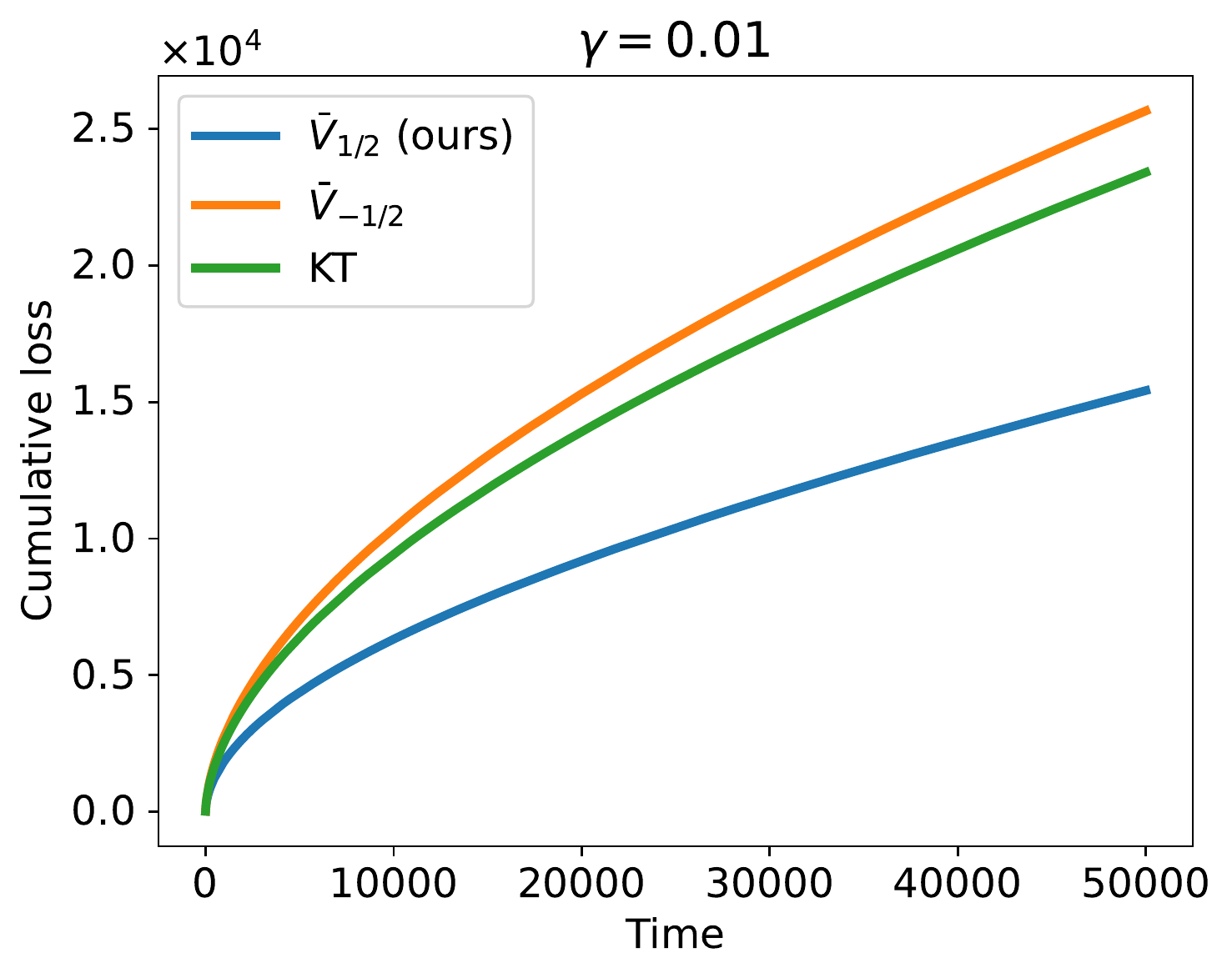}
     \end{subfigure}
     \hfill
     \begin{subfigure}[b]{0.32\textwidth}
         \centering
         \includegraphics[width=\textwidth]{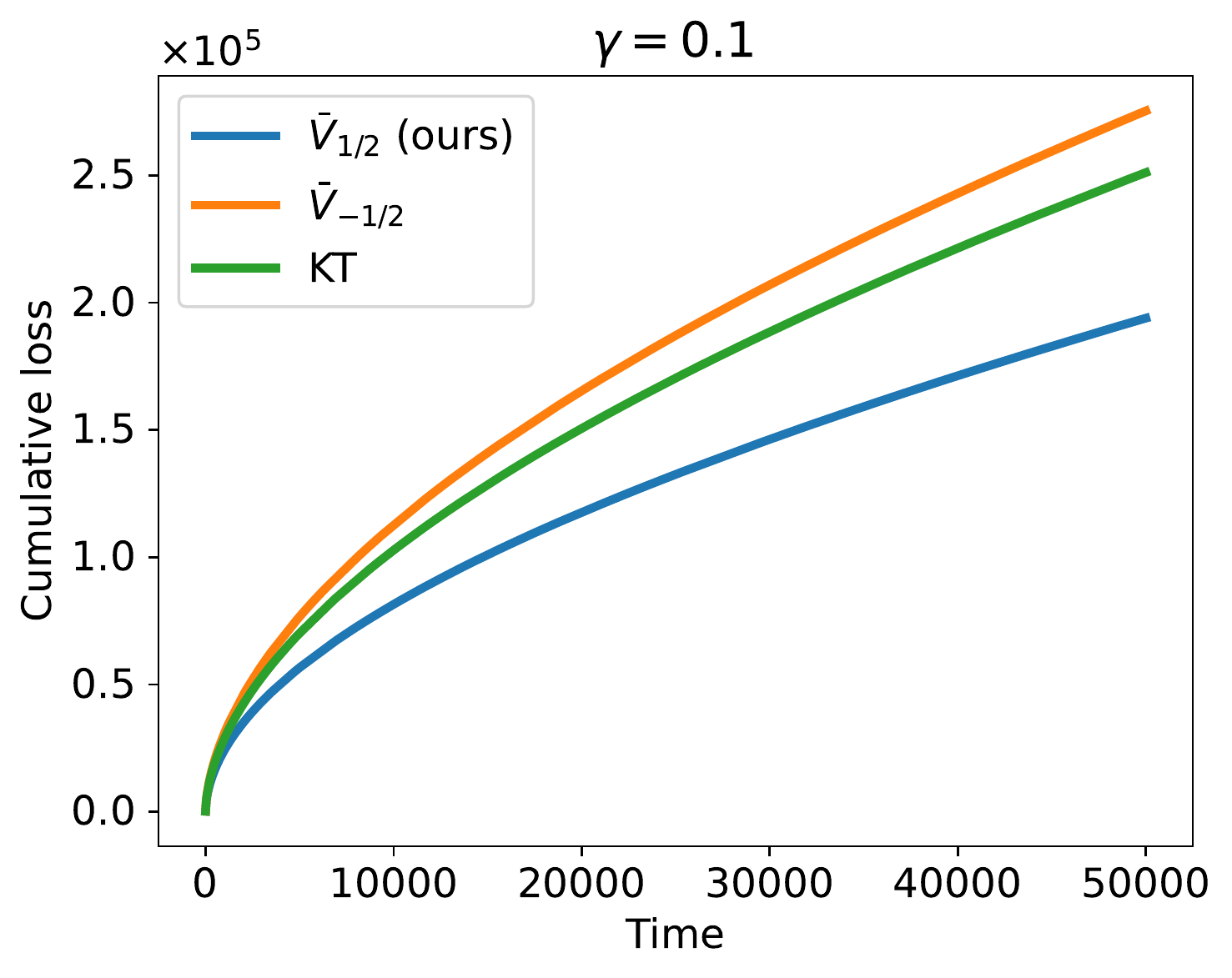}
     \end{subfigure}
     \hfill
     \begin{subfigure}[b]{0.32\textwidth}
         \centering
         \includegraphics[width=\textwidth]{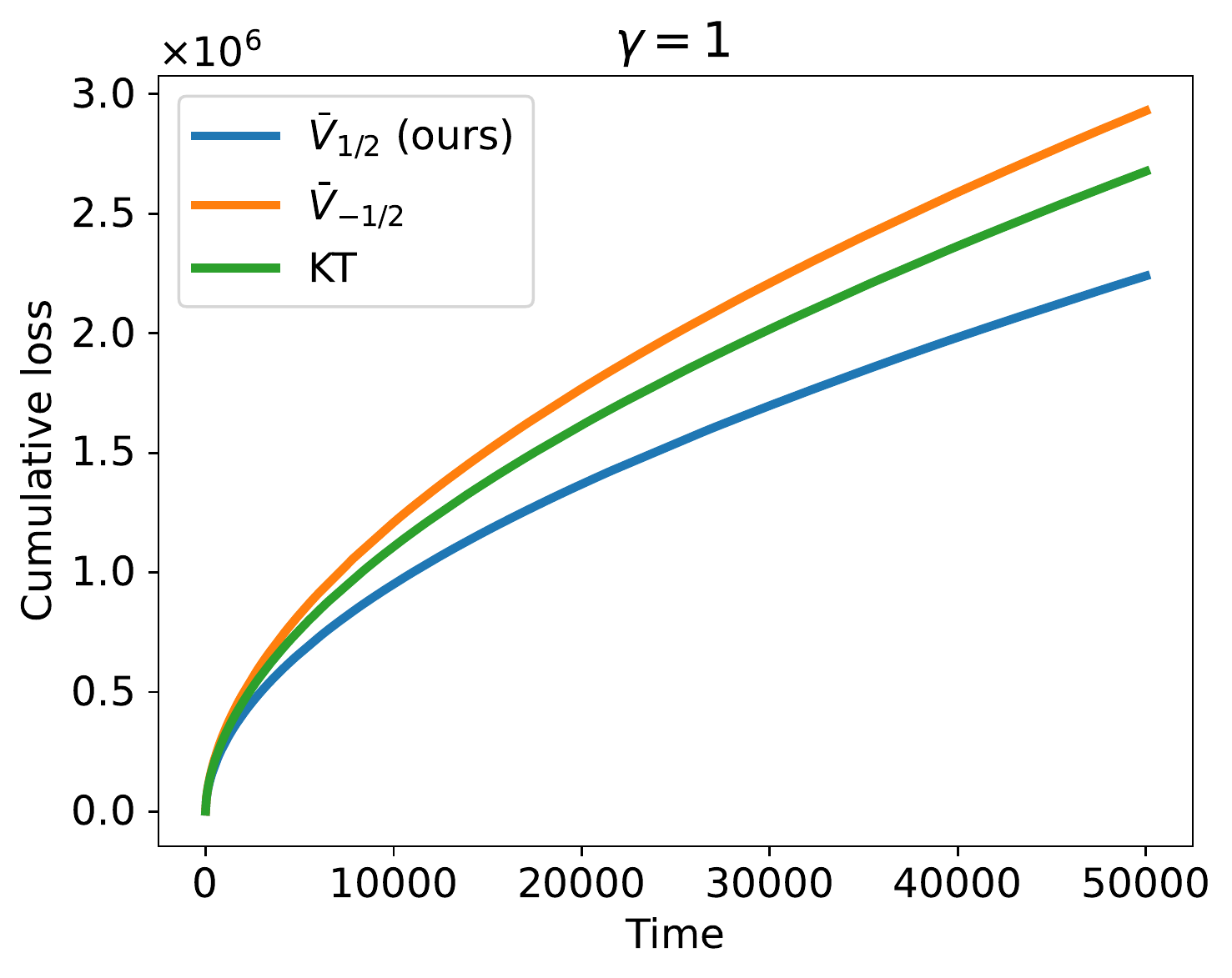}
     \end{subfigure}
        \caption{High-dimensional experiment with real data. $C=1$. }
        \label{fig:highdimensional_more}
\end{figure}

As for the type-($ii$) results, Figure~\ref{fig:highdimensional} shows the difference between KT and our algorithm ($\textrm{TotalLoss}|_{KT}-\textrm{TotalLoss}|_{ours}$) as a function of $\gamma$. In other words, it plots the difference between the green and blue lines in Figure~\ref{fig:highdimensional_more} at $T=50000$.

\begin{figure}[ht]
\centering
\includegraphics[width=0.35\textwidth]{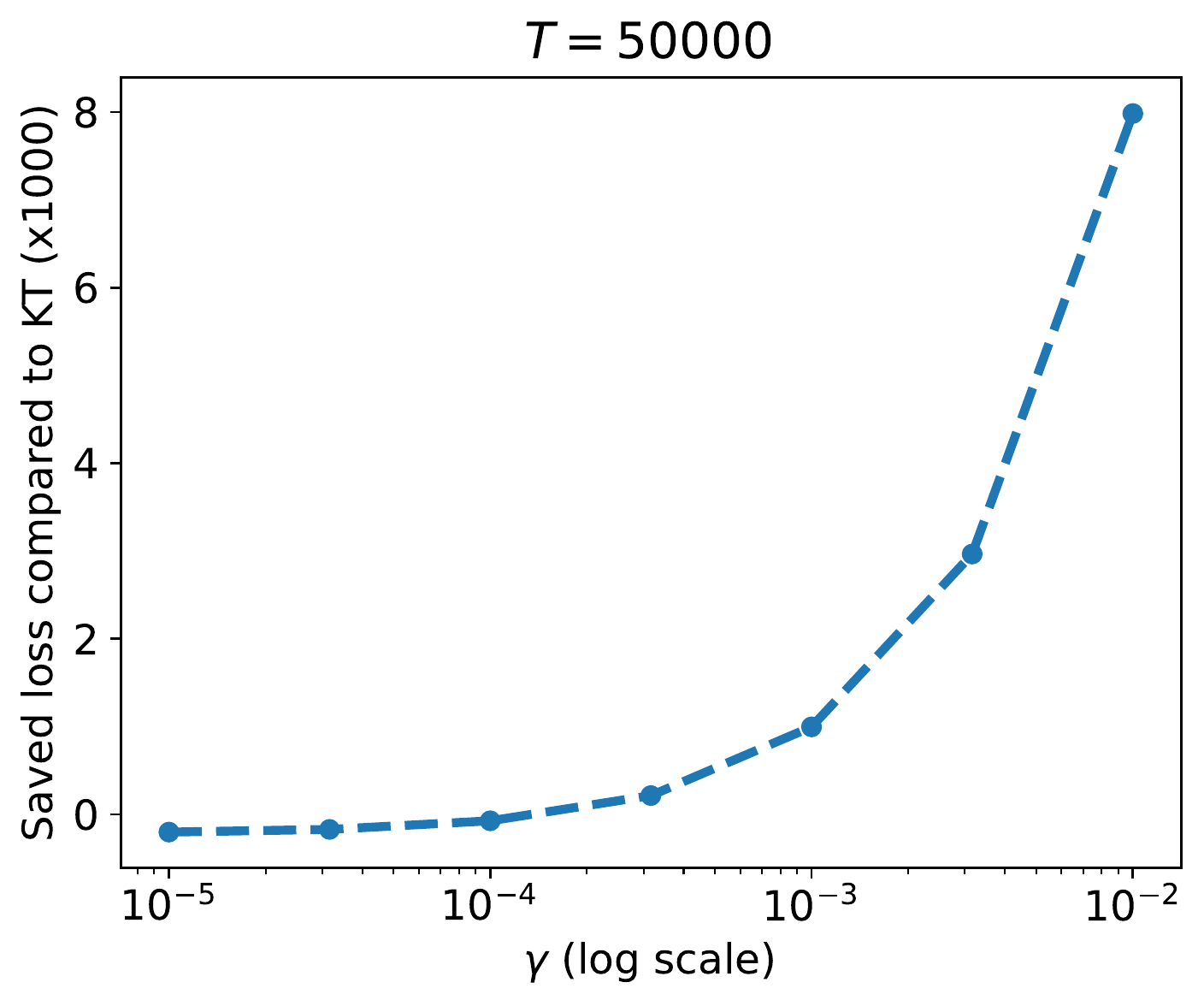}
\caption{High-dimensional regression. Plot shows the TotalLoss of KT minus the TotalLoss of our algorithm. }\label{fig:highdimensional}
\end{figure}

Combining two figures, we can draw a similar conclusion as the one-dimensional experiment: our algorithm outperforms the baseline when the optimal comparator is far-away from the initial prediction.

\end{document}